\documentclass[twoside,11pt]{article}

\usepackage{jmlr2e}

\usepackage{amsmath}
\usepackage{amssymb}

\usepackage{algorithm}
\usepackage{algorithmic}
\usepackage{caption}
\usepackage{graphicx}
\usepackage{epstopdf}
\usepackage{xcolor}
\usepackage{bbm}
\usepackage{pifont}

\usepackage{xr}

\usepackage{amssymb}
\usepackage{newunicodechar}
\usepackage{hyperref}       
\usepackage{url}            
\usepackage{booktabs}       
\usepackage{amsfonts}       
\usepackage{nicefrac}       
\usepackage{microtype}      
\usepackage{multirow}
\usepackage{lastpage}

\newcommand{\balpha}{\mathbf{\alpha}}

\DeclareMathOperator{\dist}{dist}

\newcommand{\afunc}[1]{\operatorname{\mathsf{#1}}}
\newcommand{\xmark}{\ding{55}}%

\usepackage{array}

\newcommand{\D}{\mathcal{D}}

\newcommand{\R}{\mathbb{R}}

\jmlrheading{24}{2023}{1-\pageref{LastPage}}{1/22; Revised
10/22}{1/23}{22-0099}{Boyu Wang, Jorge A.~Mendez, Changjian Shui, Fan Zhou, Di Wu, Gezheng Xu, Christian~Gagn\'e, and Eric~Eaton}

\ShortHeadings{Gap Minimization for Knowledge Sharing and Transfer}{Wang, Mendez, Shui, Zhou, Wu, Xu, Gagn\'e and Eaton}
\firstpageno{1}

\begin{document}

\title{Gap Minimization for Knowledge Sharing and Transfer}

\author{\name Boyu~Wang \email bwang@csd.uwo.ca \\
       \addr Department of Computer Science\\
       University of Western Ontario
       \AND 
       \name Jorge~A.~Mendez \email jmendez@csail.mit.edu \\
       \addr Computer Science \& Artificial Intelligence Laboratory \\
       Massachusetts Institute of Technology
       \AND
       \name Changjian~Shui \email changjian.shui.1@ulaval.ca \\
       \addr Institute Intelligence and Data \\
       Universit\'e Laval
       \AND
       \name Fan~Zhou \email zhoufan@buaa.edu.cn \\
       \addr  School of Transportation Science and Engineering \\
       Beihang University
       \AND
       \name Di~Wu \email di.wu5@mail.mcgill.ca \\
       \addr  Department of Electrical and Computer Engineering \\
       McGill University
       \AND
       \name Gezheng~Xu \email gxu86@uwo.ca \\
       \addr Department of Computer Science\\
       University of Western Ontario
       \AND 
       \name Christian~Gagn\'e \email christian.gagne@gel.ulaval.ca \\
       \addr Institute Intelligence and Data \\
       Universit\'e Laval
       \AND
       \name Eric~Eaton\email eeaton@seas.upenn.edu \\
       \addr Department of Computer and Information Science \\
       University of Pennsylvania}

\editor{John Shawe-Taylor}

\maketitle

\begin{abstract}
Learning from multiple related tasks by knowledge sharing and transfer has become increasingly relevant over the last two decades. In order to successfully transfer information from one task to another, it is critical to understand the similarities and differences between the domains. In this paper, we introduce the notion of \emph{performance gap}, an intuitive and novel measure of the distance between learning tasks. Unlike existing measures which are used as tools to bound the difference of expected risks between tasks (e.g., $\mathcal{H}$-divergence or discrepancy distance), we theoretically show that the performance gap can be viewed as a data- and algorithm-dependent regularizer, which controls the model complexity and leads to finer guarantees. More importantly, it also provides new insights and motivates a novel principle for designing strategies for  knowledge sharing and transfer: gap minimization. We instantiate this principle with two algorithms: 1. $\afunc{gapBoost}$, a novel and principled boosting algorithm that explicitly minimizes the performance gap between source and target domains for transfer learning; and 2. $\afunc{gapMTNN}$, a representation learning algorithm that reformulates gap minimization as semantic conditional matching for multitask learning. Our extensive evaluation on both transfer learning and multitask learning benchmark data sets shows that our methods outperform existing baselines.
\end{abstract}


\section{Introduction}

Transfer and multitask learning have been extensively studied over the last two decades~\citep{pan2010survey,weiss2016survey}. In both learning paradigms, the underlying assumption is that there is common knowledge that can be transferred or shared across different tasks. This sharing of knowledge can lead to a better generalization performance than learning each task individually. In particular, transfer learning improves the learning performance on a \emph{target} domain by leveraging knowledge from different (but related) \emph{source} domains, while multitask learning simultaneously learns multiple related tasks to improve the overall performance. However, not all tasks are equally related, and therefore sharing some information may be more beneficial than sharing other. In this work, we investigate how to quantify the similarity between tasks in a manner that is conducive to better knowledge-sharing algorithms. While the two paradigms of transfer and multitask learning are conceptually intertwined, they have typically been studied separately. From the theoretical aspect, most existing results on transfer learning primarily focus on \emph{domain adaptation}~\citep{ganin2016domain}, where no label information is available in the target domain. As a result, the benefits of transfer learning when target labels are available have remained elusive, and existing generalization bounds require additional assumptions for domain adaptation to succeed, which cannot be accurately estimated due to the lack of labeled examples in the target domain (e.g., the labeling function is $\lambda$-close to the hypothesis class~\citep{ben2007analysis,ben2010theory}). Similarly, for multitask learning, most existing methods and theoretical results only consider aligning marginal distributions to learn task-invariant feature representations~\citep{maurer2016benefit,luo2017exploiting} or task relations~\citep{zhang2012convex,bingel2017identifying,shui2019principled,zhou2020task}, without taking advantage of label information of the tasks. Consequently, the features can lack discriminative power for supervised learning even if their marginal distributions
have been properly matched.

In this paper, the  fundamental  question  we address is: \emph{when label information is available, how can we leverage it to measure the distance between tasks and to motivate concrete algorithms that share and transfer knowledge between the tasks}? To this end, we develop a unified theoretical framework for transfer learning and multitask learning by introducing the notion of \emph{performance gap}, an intuitive and novel measure of the divergence between different learning tasks. Intuitively, if the learning tasks are similar to each other, the model  trained on one task should also perform well on the other tasks, and vice versa. We formalize this intuition by showing that the transfer and multitask learning model complexities (and consequently their generalization errors) can be upper-bounded in terms of the performance gap. Our formulation eventually leads to \emph{gap minimization}, a general principle that is applicable to a large variety of forms of knowledge sharing and transfer, and provides a deeper understanding of this problem and new insight into how to leverage the labeled data instances. Gap minimization is, to the best of our knowledge, the first unified principle that accommodates both transfer and multitask learning  with the presence of labeled data. On the other hand, the specific form that gap minimization takes depends on the underlying transfer or multitask learning approach, and therefore the performance gap can be viewed as a data- and algorithm-dependent regularizer, which leads to finer and more informative generalization bounds. On the algorithmic side, the gap minimization principle also motivates new  strategies for knowledge sharing and transfer. 

Our analysis is based on \emph{algorithmic stability}~\citep{bousquet2002stability}. Most existing theoretical justifications for transfer and multitask learning in the literature based on algorithmic stability have examined the convergence rate of the stability coefficient under the assumption that the model complexity (and hence the loss function) of an algorithm is upper-bounded by a fixed constant. However, as we demonstrate in this work, the model complexity itself can be viewed as a function of the performance gap, which lends a new perspective on transfer and multitask learning algorithms. In particular, we show in a diversity of settings that a small performance gap equates to a small model complexity. This provides mathematical evidence for the intuition that if two tasks are similar to each other, transfer or multitask learning across these two tasks can successfully improve model performance. Usefully, in many cases, this performance gap is modulated by algorithm parameters (e.g., instance weights or representational parameters) and can therefore be minimized. In some settings, our analysis does not reveal immediate gap minimization strategies; in such cases, our analysis still provides theoretical insights that reveal when knowledge sharing and transfer can successfully improve performance.

We instantiate the principle of gap minimization with three algorithms: $\afunc{gapBoost}$ for transfer learning in classification problems, $\afunc{gapBoostR}$ for transfer learning in regression problems, and $\afunc{gapMTNN}$ for deep multitask learning. $\afunc{gapBoost}$ and $\afunc{gapBoostR}$ are based on four rules derived from our generalization bounds that explicitly minimize the performance gap between source and target domains. As boosting algorithms, they offer out-of-the-box usability and readily accommodate any base algorithm for transfer learning. To make gap minimization feasible in a deep learning setting, we create $\afunc{gapMTNN}$ by reformulating the gap minimization problem as semantic conditional matching. In addition, we also derive a theoretical analysis on representation learning based on the notion of \emph{conditional Wasserstein distance}~\citep{arjovsky2017wasserstein,shui2021aggregating}, which justifies our algorithm and motivates practical guidelines in the deep learning regime. 

The remainder of this paper is organized as follows. Section~\ref{work} provides a brief overview of related work. Section~\ref{setup} presents the preliminaries and notions used for our analysis. Our main theoretical results are presented in Section~\ref{theory}, followed by $\afunc{gapBoost}$ and $\afunc{gapMTNN}$ described in Section~\ref{algorithms}. The experimental results are reported in Section~\ref{experiments}. Section~\ref{conclusion} presents our conclusions and avenues for future work.

An earlier version of some of the results in this paper appeared in the \emph{International Conference on Artificial Intelligence and Statistics}~\citep{wang2019multitask} and in \emph{Advances in Neural Information Processing Systems}~\citep{wang2019transfer}. This submission substantially extends those results, containing: (1) extension of prior theoretical results to other transfer and multitask learning scenarios in Sections~\ref{sec:frtl},~\ref{sec:htl},~\ref{sec:twmtl}, and~\ref{sec:tcmtl}; (2) $\afunc{gapBoostR}$, a novel transfer boosting algorithm for regression problems in Section~\ref{sec:gapboost}; (3) $\afunc{gapMTNN}$, a novel multitask representation learning algorithm in Section~\ref{sec:gapMTNN}; and (4) additional experiments on transfer regression multitask learning problems in Sections~\ref{sec:performance} and~\ref{sec:multiexp}. More importantly, in addition to the theoretical results and algorithms developed for two specific learning scenarios in our previous works, this paper presents a comprehensive and unified framework for both transfer and multitask learning based on the notion of performance gap, which opens up avenues for future work.

\section{Related Work}\label{work}
A large number of theoretical analyses on transfer learning and multitask learning have been developed. Thus, we briefly review the most related work.

\subsection{Transfer learning} 
Most results in transfer learning have been obtained for the \emph{domain adaptation} setting, where label information in the target domain is generally not available~\citep{pan2010survey, DBLP:journals/tist/WilsonC20}. To deal with the lack of label information in the target domain, multiple works have used the $\mathcal{H}$-divergence to measure the distance between two domains in terms of the data distribution in the feature space for binary classification~\citep{ben2007analysis,blitzer2008learning,ben2010theory}. This  theory was later extended to the general loss function by introducing the \emph{discrepancy distance}~\citep{mansour2009domain}, \emph{Wasserstein distance}~\citep{redko2017theoretical}, and \emph{Jensen-Shannon divergence}~\citep{shui2022novel}. 

For transfer learning based on instance weighting, the concept of \emph{distributional stability} was introduced by \citep{cortes2008sample} to analyze the generalization gap between the perfect weighting scheme and the kernel mean matching (KMM)~\citep{huang2007correcting}. Other works in this setting studied a transfer learning algorithm that minimizes a convex combination of the empirical risks of the source and target domains~\citep{blitzer2008learning,ben2010theory}. In the context of representation learning, theoretical results focus on learning transferable representations~\citep{liang2020does}. For instance, \citep{zhang2019bridging} proposed a domain adaptation margin theory based on the hypothesis-induced discrepancy. The importance of task diversity was studied by \citep{tripuraneni2020theory}, who obtained statistical guarantees for transfer learning by studying the impact of representation learning for transfer learning. Another approach adopted a Bayesian perspective of causal theory induction and used these theories to transfer knowledge between different environments~\citep{edmonds2020theory}. In addition, it has been revealed that the function of the source domain data can be interpreted as a regularization matrix which benefits the learning process of the target domain task~\citep{liu2017understanding}. However, as the proof schema in this latter work does not exploit the label information in the source domain, the generalization bound does not verify the benefits of source data. On the other hand, the fundamental limits of unsupervised domain adaptation have been investigated, showing that there is no guarantee for successful domain adaptation if label information is not available~\citep{zhao2019learning}.

When target label information is available, several transfer learning algorithms with theoretical justifications have been proposed. A boosting approach was proposed to reweight the data from different domains~\citep{dai2007boosting,wang2015online}. While the effectiveness of the algorithm has been empirically verified, the theoretical analysis based on VC-dimension did not demonstrate the benefit of transferred knowledge. Other work proved the benefits of sparse coding by showing that by embedding the knowledge from the source domains into a low-dimensional dictionary, the resulting model had smaller Rademacher complexity~\citep{maurer2013sparse}. In the setting of hypothesis transfer, analysis of the algorithmic stability has shown that good prior knowledge learned from source domains guarantees not only good generalization, but also fast recovery of the performance of the best hypothesis in the class~\citep{kuzborskij2013stability,kuzborskij2017fast}. More recent work has focused on the setting where limited target labels are available. For example, \citep{mansour2021theory} developed a novel meta-learning theory when learning limited target labels in the presence of multiple sources. Another approach for semi-supervised domain adaptation used entropy minimization~\citep{saito2019semi}. A novel, theoretically justified method for semi-supervised domain adaptation has recently been proposed for using min-max to align the marginal and conditional distributions between source and target domains~\citep{li2021learning}. \citep{acuna2020f} obtained a generalization bound based on the $f$-divergences of different distributions. Other recent work studied minimax lower bounds to characterize the fundamental limits of transfer learning in the context of regression~\citep{kalan2020minimax}.

\subsection{Multitask learning}

There have also been extensive theoretical studies on multitask learning. For example, by utilizing the notations of VC-dimension and covering number, the generalization bound of multitask learning has been investigated under the assumption that the tasks share some common hypothesis space~\citep{baxter2000model}. Later, improved bounds were presented under the assumption that the tasks are related in a way such that they share a common linear operator which is chosen to preprocess data~\citep{maurer2006bounds}, or the linear model parameters lie in a low-dimensional subspace~\citep{maurer2013sparse,maurer2016benefit}. The latter work was also extended to eigenvalue decomposition~\citep{wang2016multitask} and nonlinear model by gradient boosting~\citep{wang2016generalized}.
More recently, algorithm-dependent bounds have also been analyzed, based on the notation of algorithmic stability~\citep{zhang2015multi,liu2017algorithm}. Moreover, it has also been proved that the algorithmic stability method has the potential to derive tighter generalization upper bounds than the approaches based on the VC-dimension and Rademacher complexity~\citep{liu2017algorithm}. Other work used random matrix theory to analyze the performance of the multitask least-square SVM algorithm~\citep{tiomoko2020deciphering}. \citep{zhang2020generalization} studied generalization bounds from the perspective of vector-valued function learning, specifically analyzing under what conditions multitask learning can perform better than single-task learning.

In the context of representation learning, \citep{chen2018gradnorm} proposed a method to balance the joint training of multiple tasks to ensure that all tasks are trained at approximately the same rate. \citep{shui2019principled} explored a generalization bound that combines task similarity together with representation learning under an adversarial training scheme~\citep{ben2010theory,shen2018wasserstein}. In the setting of deep multitask learning theory, \citep{wu2020understanding} showed that whether or not tasks’ data are well-aligned can significantly affect the performance of multitask learning. A task grouping method was proposed to determine which tasks should and should not be learned together in multitask learning~\citep{standley2020tasks}. This method is theoretically efficient when the number of tasks is large. Another theoretical analysis was presented on the conflicting gradients~\citep{yu2020gradient}. The authors also presented a method to mitigate this challenge in multitask learning. The regret bound was studied for online multitask learning in a non-stationary environment~\citep{herbster2020online}. 

One drawback of theoretical results in both transfer and multitask learning is that they either do not  explicitly exploit the distribution distance between tasks or only consider marginal distribution alignment, which does not take the label information into consideration. As a result, they usually require additional assumptions to ensure the algorithms succeed (e.g., the combined error across tasks is small~\citep{ben2010theory}), which may not hold in practice. In contrast to existing theoretical results that rely on the distribution distance defined over the input or feature space (e.g., $\mathcal{H}$-divergence or Wasserstein distance), our analysis takes advantage of label information by relating performance gap with model complexity and provides a unified framework that encompasses a variety of transfer and multitask learning algorithms.

\section{Preliminaries}\label{setup}

In this section, we first introduce the notation and problem settings for transfer learning and multitask learning. Then, we define several notions that are later used as tools for our analysis.

\subsection{Problem Settings}

Let $z=(x,y)\in \mathcal{X} \times \mathcal{Y}$ be a training example drawn from some unknown distribution $\mathcal{D}$, where $x$ is the data point, and $y$ is its label, with $\mathcal{Y} = \{-1,1\}$ for binary classification and $\mathcal{Y} \subseteq \mathbb{R}$ for regression.  A hypothesis is a function $h\in \mathcal{H}$ that maps $\mathcal{X}$ to the set $\mathcal{Y}^\prime$ (sometimes different from $\mathcal{Y}$), where $\mathcal{H}$ is a hypothesis class. For some convex, non-negative loss function $\ell: \mathcal{Y}^\prime \times \mathcal{Y} \mapsto \mathbb{R}_+$, we denote by $\ell(h(x),y)$ the loss of hypothesis $h$ at point $z=(x,y)$. Let ${S} = \{z_i=(x_i,y_i)\}_{i=1}^N$ be a set of $N$ training examples drawn independently from $\mathcal{D}$. The empirical loss of $h$ on ${S}$ and its generalization loss over $\mathcal{D}$ are defined, respectively, as $\mathcal{L}_S(h) = \frac{1}{N} \sum_{i=1}^N \ell(h(x_i),y_i)$, and $\mathcal{L}_\mathcal{D}(h) = \mathbb{E}_{z\sim \mathcal{D}} [\ell(h(x),y)]$. We consider the linear function class in a Euclidean space, where the hypothesis $h \in \mathcal{H}$ has the form of $h(x) = \langle h, x\rangle$, but our analysis is also applicable to a reproducing kernel Hilbert space. We also assume that $\|x\|_2 \le R, \forall x\in \mathcal{X}$ for some $R\in \mathbb{R}_+$, and the loss function is $\rho$-Lipschitz continuous for some $\rho\in \mathbb{R}_+$.

In the setting of transfer learning, we have a training sample $S = \{{S}_\mathcal{T}, {S}_\mathcal{S}\}$ of size $N = N_\mathcal{T}+N_\mathcal{S}$ composed of ${S}_\mathcal{T}=\{z^{\mathcal{T}}_i=(x^{\mathcal{T}}_i,y^{\mathcal{T}}_i)\}_{i=1}^{N_\mathcal{T}}$ drawn from a target distribution $\mathcal{D}_{\mathcal{T}}$ and ${S}_\mathcal{S}=\{z^{\mathcal{S}}_i=(x^{\mathcal{S}}_i,y^{\mathcal{S}}_i)\}_{i=1}^{N_\mathcal{S}}$ drawn from a source distribution $\mathcal{D}_{\mathcal{S}}$. 
The objective of transfer learning is to improve the generalization performance in the target domain by leveraging knowledge from the source domain. 
In the setting of multitask learning, we denote by $S=\{S_k\}_{k=1}^K$ the $K$ related tasks, where $S_k=\{z_i^k=(x_i^k, y_i^k)\}_{i=1}^{N_k}$  is the $k$-th task drown from a distribution $\mathcal{D}_k$. For simplicity, we assume that $N_k = N, \forall k = 1,\dots,K$. The objective of multitask learning is to improve the generalization performance of all $k$ tasks by sharing knowledge across them. Note that, in contrast to multitask learning in the computer vision~\citep{sermanet2013overfeat,misra2016cross,cordts2016cityscapes,doersch2017multi,yu2020bdd100k} or multi-label learning~\citep{zhang2013review,gibaja2015tutorial,sener2018multi,liu2021emerging} settings, where the learning paradigm of each task can be different (e.g., classification, regression, detection, and localization) or each instance has multiple labels, we study the setting where the tasks share the same label space $\mathcal{Y}$, but each task $k$ is sampled from a different (yet related) distribution $\mathcal{D}_k$~\citep{zhang2015multi,maurer2016benefit,liu2017algorithm,wang2019multitask,zhang2020generalization}.

\subsection{Definitions}

Our main analysis tool is the notion of \emph{uniform stability}~\citep{bousquet2002stability,mohri2018foundations}, as introduced below.

\begin{definition}[{\bf Uniform stability}]\label{definition1}
    Let $h_{{S}}\in \mathcal{H}$ be the hypothesis returned by a learning algorithm $\mathcal{A}$ when trained on sample $S$. An algorithm $\mathcal{A}$ has $\beta$-uniform stability, with $\beta \ge 0$, if the following holds:
   \begin{align*}
        \sup_{z \sim \mathcal{D}} \left| \ell(h_S(x),y)- \ell(h_{S^i}(x),y)\right| \le \beta \qquad \forall {S}, {S}^i\enspace,
    \end{align*}
    where ${S}^i$ is the training sample ${S}$ with the $i$-th example $z_i$ replaced by an i.i.d. example $z_i'$. The smallest such $\beta$ satisfying the inequality is called the stability coefficient of $\mathcal{A}$.
\end{definition}

Our analysis assumes that the loss function is $\rho$-Lipschitz continuous, as defined below.
\begin{definition}[{\bf $\rho$-Lipschitz continuity}]\label{definition2}
   A loss function $\ell(h(x), y)$ is $\rho$-Lipschitz continuous with respect to the hypothesis class $\mathcal{H}$ for some $\rho \in \mathbb{R}_+$, if, for any two hypotheses $h, h' \in \mathcal{H}$ and for any $(x,y)\in \mathcal{X}\times \mathcal{Y}$, we have:
   \begin{align*}
        \left| \ell(h(x),y) - \ell(h'(x),y ) \right| \le \rho \left|  h(x)-h'(x)\right|\enspace.
   \end{align*}
\end{definition}

We use the notion of \emph{$\mathcal{Y}$-discrepancy}~\citep{mohri2012new} to bound the difference of expected risks between two tasks.

\begin{definition}[\bf $\mathcal{Y}$-Discrepancy]
Let $\mathcal{H}$ be a hypothesis class mapping $\mathcal{X}$ to $\mathcal{Y}$ and let $\ell: \mathcal{Y} \times \mathcal{Y} \mapsto \mathbb{R}_+$ define a loss function over $\mathcal{Y}$. The $\mathcal{Y}$-discrepancy distance between two distributions $\mathcal{D}_1$ and $\mathcal{D}_2$ over $\mathcal{X}\times \mathcal{Y}$ is defined as:
\begin{align*}
        \dist_\mathcal{Y}(\mathcal{D}_1,\mathcal{D}_2) = \sup_{h \in \mathcal{H}} \left| \mathcal{L}_{\mathcal{D}_1} (h) - \mathcal{L}_{\mathcal{D}_2} (h) \right|\enspace.
    \end{align*}
\end{definition}

\begin{remark}
While $\dist_\mathcal{Y}(\hat{\mathcal{D}}_1,\hat{\mathcal{D}}_2)$  is defined taking into account the label information in the target domain, estimating its empirical counterpart $\dist_\mathcal{Y}(S_1,S_2)$ is still elusive~\citep{wang2019transfer}. Therefore, the notion of $\mathcal{Y}$-Discrepancy has limited practical use to motivate concrete transfer and multitask learning algorithms. 
\end{remark}

\section{Main Results}\label{theory}
In this section, we provide the generalization bounds for a batch of representative transfer and multitask learning algorithms. We define an algorithm-dependent notion of performance gap for each algorithm, based on a straightforward intuition: \emph{if two domains are similar, the model trained on one domain should also perform well on the other}. In contrast to existing theoretical tools which bound the difference of expected risks between two different distributions,  our analysis reveals that each performance gap can be viewed as a data- and algorithm-dependent regularizer,  which can lead to finer performance guarantees and motivate new strategies for transfer and multitask learning.  

Specifically, we study the instance weighting, feature representation, and hypothesis transfer approaches to transfer learning; and the tasking weighting, parameter sharing, and task covariance approaches to multitask learning. Note that these approaches are not isolated, but are intertwined with each other in various ways. As examples, parameter sharing is equivalent to the task covariance approach with an appropriately chosen task covariance matrix~\citep{zhang2010convex}, and task weighting can be integrated with representation learning for transfer learning~\citep{shui2021aggregating} and multitask learning~\citep{shui2019principled,zhou2021multi}.

\subsection{Transfer Learning}
\subsubsection{Instance Weighting}\label{secit}
We first investigate the \emph{instance weighting} approach for transfer learning, where  the objective is to correct the difference between the domains by weighting the instances, which has been widely adopted in the literature. Specifically, we analyze the following objective function:
\begin{align}\label{itobj1}
    \min_{h\in \mathcal{H}}  \mathcal{L}_{S}^{\Gamma}(h)+  \lambda \mathcal{R}(h)\enspace,
\end{align}
where {\small $\mathcal{L}_{S}^{\Gamma}(h) = \mathcal{L}_{S_\mathcal{T}}^{\Gamma^\mathcal{T}}(h) + \mathcal{L}_{S_\mathcal{S}}^{\Gamma^\mathcal{S}}(h)$} is the \emph{weighted} empirical loss over the source and target domains, $\mathcal{R}(h)$ is a regularization function to control the model complexity of $h$, and $\lambda$ is a regularization parameter. The domain-specific weighted losses are given by {\small $\mathcal{L}_{S_\mathcal{T}}^{\Gamma^\mathcal{T}}(h) = \sum_{i=1}^{N_\mathcal{T}} \gamma_i^{\mathcal{T}}\ell(h(x^{\mathcal{T}}_i), y_i^{\mathcal{T}})$} and {\small $\mathcal{L}_{S_\mathcal{S}}^{\Gamma^\mathcal{S}}(h) = \sum_{i=1}^{N_\mathcal{S}} \gamma_i^{\mathcal{S}}\ell(h(x^{\mathcal{S}}_i), y_i^{\mathcal{S}})$}. The instance weights {\small ${\Gamma=[\Gamma^\mathcal{T}; \Gamma^\mathcal{S}]}$, with ${\Gamma^\mathcal{T} = [\gamma_1^\mathcal{T},\dots,\gamma_{N_\mathcal{T}}^\mathcal{T}]^\top \in \mathbb{R}_+^{N_{\mathcal{T}}}}$ and ${\Gamma^\mathcal{S}=[\gamma_1^\mathcal{S},\dots,\gamma_{N_\mathcal{S}}^\mathcal{S}]^\top \in \mathbb{R}_+^{N_{\mathcal{S}}}}$, are such that the overall weight sums to one: $\sum_{i=1}^{N_\mathcal{T}} \gamma_i^{\mathcal{T}} + \sum_{i=1}^{N_\mathcal{S}} \gamma_i^{\mathcal{S}} =1$}. As we consider the linear function class, the hypothesis $h$ has the form of an inner product $h(x) = \langle h, x\rangle$, and we study the regularization function $\mathcal{R}(h) = \|h\|_2^2$ throughout this paper unless otherwise specified. 

A special case of (\ref{itobj1}) is to minimize a convex combination of the empirical losses of the source and target domains~\citep{blitzer2008learning,ben2010theory,liu2017understanding}:
\begin{align*}
    \min_{h\in\mathcal{H}}  \gamma \mathcal{L}_{S_\mathcal{T}}(h) + (1-\gamma)\mathcal{L}_{S_\mathcal{S}}(h) +  \lambda \mathcal{R}(h)\enspace,
\end{align*}
where $\gamma \in [0,1]$ is a weight parameter that controls the trade-off between target and source domains. Note that we only investigate instance transfer from a single source domain, but the extension to multi-source transfer is straightforward~\citep{yao2010boosting,eaton2011selective,sun2011two}.

There are a variety of weighting schemes developed in the literature. The instance weights can either be learned in a pre-processing step, such as kernel mean matching (KMM;~\citep{huang2007correcting,gretton2007kernel,cortes2008sample,pan2011domain}) or direct density ratio estimation~\citep{sugiyama2012machine}, or incorporated into learning algorithms, such as boosting~\citep{dai2007boosting,pardoe2010boosting,yao2010boosting,eaton2011selective} or probability alignment~\citep{sun2011two}.

For instance transfer, we define the notion of performance gap as follows.
\begin{definition}[{\bf Performance gap for instance transfer}]\label{it}
   Let ${\mathcal{V}_{\mathcal{S}}(h) = \mathcal{L}^{\Gamma^\mathcal{S}}_{S_\mathcal{S}}(h) +  \eta \lambda \mathcal{R}(h)}$ and ${\mathcal{V}_{\mathcal{T}}(h) =  \mathcal{L}^{\Gamma^\mathcal{T}}_{S_\mathcal{T}}(h) + \eta \lambda \mathcal{R}(h)}$ be the objective functions in the source and target domains, where $\eta \in [0, \frac{1}{2})$. Let their minimizers, respectively, be $h_{\mathcal{S}}$ and $h_{\mathcal{T}}$. The performance gap for instance transfer is defined as:
   \begin{align*}
        \nabla = \nabla_\mathcal{T} + \nabla_\mathcal{S}\enspace,
   \end{align*}
   where $\nabla_\mathcal{S} =  \mathcal{V}_{\mathcal{S}}(h_{\mathcal{T}}) -\mathcal{V}_{\mathcal{S}}(h_{\mathcal{S}})$ and $\nabla_\mathcal{T}=  \mathcal{V}_{\mathcal{T}}(h_{\mathcal{S}}) -\mathcal{V}_{\mathcal{T}}(h_{\mathcal{T}})$.
\end{definition}
The following theorem bounds the difference between $\mathcal{L}_{\mathcal{D}_\mathcal{T}}$ and $\mathcal{L}_S^\Gamma$, which provides general principles to follow when designing an instance weighting scheme  for transfer learning.

\begin{theorem}\label{maintheorem1}
Let $h^*$ be the optimal solution of the instance weighting transfer learning problem (\ref{itobj1}). Then, for any $\delta \in (0,1)$, with probability at least $1-\delta$, we have:
     \begin{align*}
         \mathcal{L}_{\mathcal{D}_\mathcal{T}}(h^*) & \le \mathcal{L}_S^\Gamma (h^*) + \|\Gamma^\mathcal{S}\|_1 \dist_\mathcal{Y}(\mathcal{D}_\mathcal{T}, \mathcal{D}_\mathcal{S}) + \beta  + (\Delta + \beta + \|\Gamma\|_\infty B(\Gamma))\sqrt{\frac{N\log \frac{1}{\delta}}{2}}\enspace,
     \end{align*}
 where  $B(\Gamma)$ is the upper bound of the loss function $\ell$, such that $\ell(h(x),y) \le B(\Gamma)$.  $\beta = \max\{\beta_1^\mathcal{T},\dots,\beta_{N_\mathcal{T}}^\mathcal{T},\beta_1^\mathcal{S},\dots,\beta_{N_\mathcal{S}}^\mathcal{S}\}$ and $\Delta = \sum_{i=1}^{N_\mathcal{T}} \gamma_i^\mathcal{T} \beta_i^\mathcal{T}+\sum_{i=1}^{N_\mathcal{S}} \gamma_i^\mathcal{S} \beta_i^\mathcal{S}$, with the \emph{weight-dependent} stability coefficients upper bounded by:
 \begin{align*}
         \beta_i^\mathcal{T} \le    \frac{\gamma_i^\mathcal{T} \rho^2 R^2}{\lambda}, \text{ and } \beta_i^\mathcal{S} \le    \frac{\gamma_i^\mathcal{S} \rho^2 R^2}{\lambda}\enspace.
 \end{align*}
 \end{theorem}

\begin{remark}
Substituting $\beta_i^\mathcal{T}$ and $\beta_i^\mathcal{S}$ into $\beta$ and $\Delta$, we can reformulate the upper bound as:
\begin{align}\label{theoremeq1}
       \mathcal{L}_{\mathcal{D}_\mathcal{T}}(h^*)  & \le \mathcal{L}_S^\Gamma (h^*)  + \|\Gamma^\mathcal{S}\|_1 \dist_\mathcal{Y}(\mathcal{D}_\mathcal{T}, \mathcal{D}_\mathcal{S})  + \frac{||\Gamma||_\infty \rho^2 R^2}{\lambda} \nonumber \\
      & + \left(\frac{(||\Gamma||_\infty+\|\Gamma\|_2^2) \rho^2 R^2}{\lambda} + \|\Gamma\|_\infty B(\Gamma) \right)\sqrt{\frac{N\log \frac{1}{\delta}}{2}}\enspace.
\end{align}
Then, it can be observed that if $\gamma_i = \frac{1}{N}, \forall i \in \{1,\ldots, N\}$, we recover the standard stability bound from (\ref{theoremeq1}), which suggests assigning equal weights to all instances to achieve a fast convergence rate, due to $\|\Gamma\|_\infty$ and $\|\Gamma\|_2$. In particular, if $\|\Gamma\|_\infty$ (and hence $\|\Gamma\|_2^2$) is $\mathcal{O}(\frac{1}{N})$, (\ref{theoremeq1}) leads to a convergence rate of $\mathcal{O}(\frac{1}{\sqrt{N}})$. In the setting of transfer learning, it is usually the case that $N_\mathcal{T} \ll N_\mathcal{S}$. Consequently, we may have $\|\Gamma\|_\infty \ll \frac{1}{N_\mathcal{T}}$, which implies that transfer learning has a faster convergence rate than single-task learning.   On the other hand, as we will show in Lemma~\ref{lemmait} and Corollary~\ref{corollaryit}, the loss bound $B$ is also a function of $\Gamma$, which suggests a new criterion for instance weighting.
\end{remark}

\begin{lemma}\label{lemmait}
    Let $h^*$ be the optimal solution of the instance weighting transfer learning problem (\ref{itobj1}). Then, we have:
    \begin{align}\label{lemmaiteq}
           \|h^*\|_2 \le   \sqrt{\frac{\nabla}{2\lambda(1-2\eta)}  + \frac{\|h_{\mathcal{S}}\|_2^2 + \|h_{\mathcal{T}}\|_2^2}{2}}\enspace.
    \end{align}
\end{lemma}
By bounding the model complexity, we obtain various upper bounds for different loss functions.
\begin{corollary}\label{corollaryit}
The hinge loss function of the learning algorithm (\ref{itobj1}) can be upper bounded by:
    \begin{align*}
        B(\Gamma) \le 1+R\sqrt{\frac{ \nabla}{2\lambda(1-2\eta)}  + \frac{\|h_{\mathcal{S}}\|_2^2 + \|h_{\mathcal{T}}\|_2^2}{2}}\enspace.
     \end{align*}

For regression, if the response variable is bounded by ${|y| \le Y}$, the $\ell_q$ loss of (\ref{itobj1}) can be bounded by:
  \begin{align*}
         B(\Gamma) \le \left(Y+R\sqrt{\frac{\nabla}{2\lambda(1-2\eta)}  + \frac{\|h_{\mathcal{S}}\|_2^2 + \|h_{\mathcal{T}}\|_2^2}{2}} \right)^q\enspace.
      \end{align*}
\end{corollary}

\begin{remark}
Lemma~\ref{lemmait} shows that, given fixed weights, the model complexity (and hence the upper bound of a loss function) is related to  the intrinsic complexity of the learning problems in source and target domains (measured by $||h_\mathcal{S}||_2^2+||h_\mathcal{T}||_2^2$) and the performance gap between them (measured by $\nabla$). Intuitively, if two domains are similar, $\nabla$ should be small, and vice versa. In other words, Lemma~\ref{lemmait} reveals that transfer learning~(\ref{itobj1}) can succeed when the hypotheses trained on their own domains also work well on the other domains, which leads to a lower training loss and a faster convergence to the best hypothesis in the class in terms of sample complexity.
If we assume $B$ to be constant as previous work has done (and not dependent on $\Gamma$ as we show in Corollary~\ref{corollaryit}), then we would wrongly conclude that the optimal weighting scheme should only focus on the trade-off between assigning balanced weights to data points (i.e., treat source and target domains equally), as suggested by $\|\Gamma\|_2$ and $\|\Gamma\|_\infty$, and assigning more weight to the target domain sample than to the source domain sample, as suggested by $\|\Gamma^\mathcal{S}\|_1$. However, Lemma~\ref{lemmait} reveals that, even though this leads to a smaller stability coefficient, its convergence rate may still be slow if the performance gap is high, due to the higher model complexity. In other words, a good weighting scheme should also take minimizing the performance gap into account. 
\end{remark}

In contrast to previous analyses, which usually require additional assumptions to characterize the difference between the source domains and the target domain -- e.g., $\lambda$-closeness between a labeling function and hypothesis class~\citep{ben2007analysis,ben2010theory}, or boundedness of the difference between the conditional distributions of the source domain and the target domain~\citep{wang2015generalization} -- our analysis takes  advantage of the label information in the target domain and does not make any assumption on the relatedness between the domains, as such property has already been characterized by the performance gap, which can be estimated from the training instances.

\subsubsection{Feature Representation}\label{sec:frtl}
\emph{Feature representation} is another important approach to transfer learning, which aims to minimize the domain divergence by learning a domain-invariant feature representation. Specifically, we consider a predictor $(h \circ \Phi)(x) = h(\Phi(x))$, where $\Phi: \mathcal{X} \rightarrow \mathcal{Z}$ is a feature representation function, and $h$ is a linear hypothesis defined over $\mathcal{Z}$ such that $h(z)=\langle h, z\rangle, \forall z \in \mathcal{Z}$. For any fixed representation function $\Phi$, we analyze how it can affect the complexity of $h$ by studying the following objective function:
\begin{align}\label{frobj1}
    \min_{h\in \mathcal{H}}  \mathcal{L}_{S}^{\Phi}(h)+  \lambda \mathcal{R}(h)\enspace,
\end{align}
where $\mathcal{L}_{S}^{\Phi}(h) = \mathcal{L}_{S_\mathcal{T}}^\Phi(h) + \mathcal{L}_{S_\mathcal{S}}^\Phi(h)$ is the combined  empirical loss, and the empirical target and source losses are respectively given by  $\mathcal{L}_{S_\mathcal{T}}^\Phi(h) = \frac{1}{N}\sum_{i=1}^{N_\mathcal{T}} \ell(h(\Phi(x_i^\mathcal{T})),y_i^\mathcal{T})$ and $\mathcal{L}_{S_\mathcal{S}}^\Phi(h) = \frac{1}{N} \sum_{i=1}^{N_\mathcal{S}} \ell(h(\Phi(x_i^\mathcal{S})),y_i^\mathcal{S})$.

For feature representation, the performance gap can be defined in a similar way as for instance weighting. 
\begin{definition}[{\bf Performance gap for feature  transfer}]\label{deffr}
   Let ${\mathcal{V}_{\mathcal{S}}(h) =  \mathcal{L}^{\Phi}_{S_\mathcal{S}}(h) + \eta \lambda \mathcal{R}(h)}$ and ${\mathcal{V}_{\mathcal{T}}(h) =  \mathcal{L}^{\Phi}_{S_\mathcal{T}}(h) + \eta \lambda \mathcal{R}(h)}$, respectively, be the objective functions in the source and target domains, where $\eta \in [0, \frac{1}{2})$, and let their minimizers, respectively, be $h_{\mathcal{S}}$ and $h_{\mathcal{T}}$. The performance gap for feature representation transfer is defined as:
   \begin{align*}
        \nabla = \nabla_\mathcal{T} + \nabla_\mathcal{S}\enspace,
   \end{align*}
   where $\nabla_\mathcal{S} =  \mathcal{V}_{\mathcal{S}}(h_{\mathcal{T}}) -\mathcal{V}_{\mathcal{S}}(h_{\mathcal{S}})$ and $\nabla_\mathcal{T}=  \mathcal{V}_{\mathcal{T}}(h_{\mathcal{S}}) -\mathcal{V}_{\mathcal{T}}(h_{\mathcal{T}})$.
\end{definition}

\begin{theorem}\label{maintheorem2}
 Let $h^*$ be the optimal solution of the transfer learning problem (\ref{frobj1}). Then, for any $\delta \in (0,1)$, with probability at least $1-\delta$, we have:
\begin{align*}
      \mathcal{L}_{\mathcal{D}_\mathcal{T}}(h^*) & \le \mathcal{L}_S^\Phi (h^*)  + \frac{N_\mathcal{S}}{N} \dist_\mathcal{Y}(\mathcal{D}_{\mathcal{T}}^\Phi, \mathcal{D}_{\mathcal{S}}^\Phi) + \frac{ \rho^2 R^2}{\lambda N}  + \left(\frac{2 \rho^2 R^2}{\lambda}  + {B(\Phi)} \right)\sqrt{\frac{\log \frac{1}{\delta}}{2N}}\enspace, 
\end{align*}
where $\dist_\mathcal{Y}(\mathcal{D}_{\mathcal{T}}^\Phi, \mathcal{D}_{\mathcal{S}}^\Phi)$ is the $\mathcal{Y}$-discrepancy between two domains defined over $\mathcal{Z} \times \mathcal{Y}$ induced by $\Phi$.
\end{theorem}
Similarly, $B(\Phi)$ and hence $||h^*||_2$ can be upper bounded in the same way as in Lemma~\ref{lemmait}.
\begin{lemma}\label{lemmafr}
    Let $h^*$ be the optimal solution of the feature representation transfer learning problem (\ref{frobj1}), $h_\mathcal{T}$, $h_\mathcal{S}$, and $\nabla$ be defined as in Definition~\ref{deffr}. Then, $\|h^*\|_2$ can be upper bounded by (\ref{lemmaiteq}) as in Lemma~\ref{lemmait}.
\end{lemma}

Note that, unlike in Theorem~\ref{maintheorem1}, where the $\mathcal{Y}$-discrepancy is a constant independent of the algorithm, here it is defined over the learned feature representation $\Phi$, and can therefore be optimized. However, estimating  $\dist_\mathcal{Y}(\mathcal{D}_{\mathcal{T}}^\Phi, \mathcal{D}_{\mathcal{S}}^\Phi)$ from finite samples is challenging in practice. In~\citep{wang2019transfer}, it has been shown that in a binary classification problem, the $\mathcal{Y}$-discrepancy can be upper bounded from a finite sample by constructing a classification problem, where the positive target examples and negative source examples are positively labeled, and the negative target examples and positive source examples are negatively labeled. However, how to extend this notion to the more general multiclass classification problem remains elusive. Moreover, it has been shown recently that aligning two distributions by minimizing the $\mathcal{Y}$-discrepancy can be problematic when dealing with multiple tasks~\citep{mansour2021theory}. Therefore, in most existing works, the  feature representation is learned by minimizing $\dist(\mathcal{D}_{\mathcal{T}}^\Phi, \mathcal{D}_{\mathcal{S}}^\Phi)$ and its variants~\citep{ganin2016domain,redko2017theoretical}, defined as follows.

\begin{definition}[\bf Discrepancy]
Let $\mathcal{H}$ be a hypothesis class mapping $\mathcal{X}$ to $\mathcal{Y}$ and let \mbox{$\ell: \mathcal{Y} \times \mathcal{Y} \rightarrow \mathbb{R}_+$} define a loss function over $\mathcal{Y}$. The discrepancy distance between two distributions $\mathcal{D}_1$ and $\mathcal{D}_2$ over $\mathcal{X}$ is defined by:
\begin{align*}
        \dist(\mathcal{D}_1,\mathcal{D}_2) = \sup_{h,h' \in \mathcal{H}} \left| \mathcal{L}_{\mathcal{D}_1} (h,h') - \mathcal{L}_{\mathcal{D}_2} (h,h') \right|\enspace,
    \end{align*}
where we have slightly abused our notation without creating confusion by making $\mathcal{L}_{\mathcal{D}} (h,h^\prime) = \mathbb{E}_{x\sim \mathcal{D}} \left[\ell(h(x),h^\prime(x)) \right]$.
\end{definition}

If there exists an underlying labeling function $f$ (i.e., a deterministic scenario) in the hypothesis class $\mathcal{H}$ ($f\in\mathcal{H}$), then $\dist_\mathcal{Y}(\mathcal{D}_1,\mathcal{D}_2)$ can be bounded by $\dist(\mathcal{D}_1,\mathcal{D}_2)$~\citep{mohri2012new}.

\begin{remark}
    While $\dist(\mathcal{D}_1,\mathcal{D}_2)$ can be accurately and  efficiently estimated from finite samples by adversarial training~\citep{ben2010theory,ganin2016domain}, it is only defined over the marginal distribution of input features and therefore does not leverage label information. In other words, neither the discrepancy nor the $\mathcal{Y}$-discrepancy can motivate concrete and efficient algorithms that leverage the label information for knowledge transfer in general. Gap minimization indicates a new principle to leverage the label information in the target domain, which is complementary to existing transfer learning strategies. Specifically, in addition to learning an invariant feature space across source and target domains (e.g., adversarially), a good representation should also minimize the performance gap between the domains.
\end{remark}

\subsubsection{Hypothesis Transfer}\label{sec:htl}

In hypothesis transfer, we consider  the setting where we are given a training sample $S_\mathcal{T}$ from the target domain and multiple hypotheses $\{h_k\}_{k=1}^K$ learned from $K$ source domains as in~\citep{kuzborskij2013stability,kuzborskij2017fast}. Let  $\Xi = [\xi_1,\dots,\xi_K]^\top$ be fixed weights of the source hypotheses.  Then, the objective function of hypothesis transfer is given by:
\begin{align}\label{htfobj1}
    \min_{h\in \mathcal{H}}  \mathcal{L}_{S_\mathcal{T}}(h)+  \lambda \mathcal{R}^\Xi(h)\enspace,
\end{align}
where $R^\Xi(h) = ||h-\langle H, \Xi \rangle||_2^2$, $H=[h_1,\dots, h_K]^\top$ is the ensemble of $K$ hypotheses. 

For hypothesis transfer, the performance gap can be defined as follows.
\begin{definition}[{\bf Performance gap for hypothesis transfer}]\label{it}
  Let $h_{\mathcal{T}}$ be the minimizer of $\mathcal{L}_{S_\mathcal{T}}(h)$, the empirical loss in the target domain. The performance gap of hypothesis transfer is defined as:
   \begin{align*}
       \nabla = \mathcal{L}_{S_\mathcal{T}}(\langle H, \Xi \rangle)  - \mathcal{L}_{S_\mathcal{T}}(h_\mathcal{T})\enspace.
   \end{align*}
\end{definition}
\begin{theorem}\label{maintheorem3}
Let $h^*$ be the optimal solution of the transfer learning problem (\ref{htfobj1}). Then, for any $\delta \in (0,1)$, with probability at least $1-\delta$, we have:
\begin{align*}
      \mathcal{L}_{\mathcal{D}_\mathcal{T}}(h^*) & \le \mathcal{L}_{S_\mathcal{T}} (h^*)   + \frac{ \rho^2 R^2}{\lambda N_\mathcal{T}} + \left(\frac{ 2\rho^2 R^2}{\lambda}  + {B(\Xi)}\right)\sqrt{\frac{\log \frac{1}{\delta}}{2N_\mathcal{T}}}\enspace.
\end{align*}
\end{theorem}
\begin{remark}
As no source sample is available, the  convergence rate of $\beta$ has the order of $\mathcal{O}(\frac{1}{N_\mathcal{T}})$, which has been shown in~\citep{perrot2015theoretical}. It has been proved that a fast convergence rate can be achieved under certain conditions~\citep{kuzborskij2013stability,kuzborskij2017fast}, but the analysis therein is restricted to linear regression. More importantly, their theoretical results do not motivate any approach to leveraging the knowledge of multiple source hypotheses. In contrast, the following Lemma indicates a finer bound on the model complexity of $h^*$ and suggests a principled scheme combining source hypotheses.
\end{remark}

\begin{lemma}\label{lemmahtf}
    Let $h^*$ be the optimal solution of the hypothesis transfer learning problem (\ref{htfobj1}). Then, we have:
    \begin{align}\label{lemmahtfeq}
           ||h^*||_2 \le \sqrt{\frac{\nabla}{\lambda}} + ||\langle H, \Xi \rangle||_2\enspace.
    \end{align}
\end{lemma}
\begin{remark}
Lemma~\ref{lemmahtf} indicates that in order to regularize the target hypothesis, one should assign the weights of source hypotheses in a way such that it can approach towards the target optimal solution as much as possible. When the number of tasks is smaller than the input feature dimension, it can be viewed as a low-dimensional embedding of the target solution into the space spanned by source hypotheses~\citep{maurer2013sparse,maurer2014aninequality}.
\end{remark}

\subsection{Multitask Learning}\label{multitask}
\subsubsection{Task Weighting}\label{sec:twmtl}
The first multitask learning approach we investigate is \emph{task weighting}~\citep{shui2019principled},  where the objective function of the $j$-th task is:
  \begin{align}\label{twmobj1}
      \min_{h\in \mathcal{H}} \mathcal{L}_S^{\Gamma^j}(h) + \lambda \mathcal{R}(h)\enspace,
  \end{align}
where $\mathcal{L}_{S_k}(h) =\frac{1}{N}\sum_{i=1}^{N} \ell (h(x^k_i),y^k_i)$, and $\mathcal{L}_S^{\Gamma^j}(h)=\sum_{k=1}^K \gamma_k^j\mathcal{L}_{S_k}(h)$. $\Gamma^j = [\gamma^j_1,\dots, \gamma^j_K]^\top\in \Upsilon^K =\{ \gamma^j_k\ge 0, \sum_{k=1}^K \gamma^j_k = 1\}$ are the task relation coefficients for the $j$-th task. Note that $\Gamma^j$ is task-dependent, and we do not force $\gamma^j_i = \gamma^i_j$. In other words, (\ref{twmobj1}) can capture asymmetric task relationships.
 
For the task weighting approach to multitask learning, the performance gap is defined as follows.
\begin{definition}[{\bf Performance gap for task weighting}] 
Let $\mathcal{V}_k(h) = \mathcal{L}_{S_k}(h)+\eta \lambda \mathcal{R}(h)$ be the loss function of the $k$-th task, where $\eta \in [0,1)$, and let $\bar{h}_k$ be its minimizer. For any given simplex $\Gamma^j \in \Upsilon^K$. The performance gap of task weighting with respect to the $j$-th task is defined as:
  \begin{align*}
      \nabla_j =  \sum_{k\neq j}  \gamma_k^j \left[ \mathcal{V}_k(\bar{h}_j) -\mathcal{V}_k(\bar{h}_k) \right]\enspace.
  \end{align*}
  Note that the performance gap of the $j$-th task  is defined in terms of the losses over the other tasks $k\neq j$.
\end{definition}
\begin{theorem}\label{twmtheorem}
    Let $h_j^*$ be the optimal solution of the task weighting multitask learning problem (\ref{twmobj1}). 
    Then, for any $\delta \in (0,1)$, with probability at least $1-\delta$, we have:
    \begin{align*}
        & \mathcal{L}_{\mathcal{D}_j} (h_j^*) \le \mathcal{L}_S^{\Gamma^j}(h_j^*) + \sum_{k\neq j}\gamma_k^j\text{dist}_\mathcal{Y}(\mathcal{D}_j, \mathcal{D}_k) + \frac{ \rho^2 R^2 ||{\Gamma}^j||_\infty}{\lambda N} \\
        & + \left(  \frac{ (||{\Gamma}^j||_\infty + ||\Gamma^j||_2^2)\rho^2 R^2 }{\lambda N } + \frac{||{\Gamma}^j||_\infty}{N} B(\Gamma^j)\right)\sqrt{\frac{KN\log\frac{1}{\delta}}{2}}\enspace.
    \end{align*}
\end{theorem}
\begin{remark}
By leveraging the instances from other tasks, the weighted multitask learning problem (\ref{twmobj1}) may achieve a fast convergence rate of $\mathcal{O}(\frac{1}{\sqrt{KN}})$ when $\gamma^j_k = \frac{1}{K}$. On the other hand,  Theorem~\ref{twmtheorem} suggests that $\mathcal{L}_S^{\Gamma^j}(h_j^*)$ and  $\text{dist}_\mathcal{Y}(\mathcal{D}_j, \mathcal{D}_k)$ should also be taken into consideration in order to minimize the upper bound in Theorem~\ref{twmtheorem}. Specifically,  $\Gamma^j$ is chosen by balancing the trade-off between $\mathcal{L}_S^{\Gamma^j}(h)$, $\sum_{k\neq j}\gamma_k^j \text{dist}_\mathcal{Y}(\mathcal{D}_j, \mathcal{D}_k)$,  $||{\Gamma}^j||_\infty$, and $||\Gamma^j||_2^2$. Intuitively, the task coefficients $\{\gamma_k^j\}_{k=1}^K$ should capture the relatedness between tasks. Therefore, one should assign a small value to $\gamma_k^j$ if $\mathcal{L}_{S_k}(h_j^*)$ and $\text{dist}_\mathcal{Y}(\mathcal{D}_j,\mathcal{D}_k)$ are large. \citep{shui2019principled} replace $\text{dist}_\mathcal{Y}(\mathcal{D}_j,\mathcal{D}_k)$ by $\text{dist}(\mathcal{D}_j,\mathcal{D}_k)$ for computing $\gamma_k^j$, and therefore the label information is ignored when measuring the distance between tasks. As complementary to this approach, the following Lemma suggests an additional criterion for learning $\Gamma^j$. 
\end{remark}

\begin{lemma}\label{lemmatwm}
    Let $h^*_j$ be the optimal solution of the task weighting multitask learning problem (\ref{twmobj1}). Then, we have:
    \begin{align}\label{lemmatwmeq}
           ||h^*_j||_2 \le \sqrt{\frac{\nabla_j}{\lambda(1-\eta)}+||\bar{h}_j||_2^2}\enspace.
    \end{align}
\end{lemma}
\begin{remark}
Similar to Lemma~\ref{lemmait}, Lemma~\ref{lemmatwm} indicates that merely balancing the trade-off between assigning balanced weights to tasks and assigning more weight to the $j$-th task can still be insufficient to achieve a high generalization performance due to the performance gap. Intuitively, Lemma~\ref{lemmatwm} implies another principle that one should also assign more weight to the tasks over which $\bar{h}_j$ performs well (i.e., $\mathcal{V}_k(\bar{h}_j)$ is small). 
\end{remark}

{\bf \noindent Integrating with Representation Learning }
The task weighting approach can also be integrated with representation learning for multitask learning~\citep{shui2019principled}. Specifically, the objective function becomes:
  \begin{align}\label{twrlmobj1}
      \min_{h\in \mathcal{H}} \mathcal{L}_S^{\Gamma^j,\Phi}(h) + \lambda \mathcal{R}(h)\enspace,
  \end{align}
  where $\mathcal{L}^\Phi_{S_k}(h) = \frac{1}{N}\sum_{i=1}^{N}\ell(h(\Phi(x_i^k)),y_i^k)$, $\mathcal{L}_S^{\Gamma^j,\Phi}(h) = \sum_{k=1}^K \gamma_k^j \mathcal{L}^\Phi_{S_k}(h)$. Also, let $\mathcal{V}_k^\Phi(h) = \mathcal{L}^\Phi_{S_k}(h) + \eta \lambda \mathcal{R}(h)$ and $\nabla_j(\Phi,\Gamma^j) = \sum_{k \neq j} \gamma_k^j\left[\mathcal{V}_k^\Phi (\bar{h}_j) - \mathcal{V}_k^\Phi (\bar{h}_k) \right]$, where $\bar{h}_k$ is the minimizer of  $\mathcal{V}_k^\Phi(h)$. Note that given a fixed set of weights $\Gamma^j$, $\nabla_j$ is only a function of $\Phi$, not $h$. Then, we can obtain the generalization bound for task weighting multitask representation learning.

\begin{corollary}\label{twrlmcor1}
Let $h_j^*$ be the optimal solution of the task weighting multitask learning problem (\ref{twrlmobj1}). 
    Then, for any $\delta \in (0,1)$, with probability at least $1-\delta$, we have
    \begin{align*}
         \mathcal{L}_{\mathcal{D}_j} (h_j^*) & \le \mathcal{L}_S^{\Gamma^j,\Phi}(h_j^*) + \sum_{k\neq j}\gamma_k^j\text{dist}_\mathcal{Y}(\mathcal{D}_j^{\Phi}, \mathcal{D}_k^{\Phi}) + \frac{ \rho^2 R^2 ||{\Gamma}^j||_\infty}{\lambda N} \\
        & + \left(  \frac{ (||{\Gamma}^j||_\infty +||\Gamma^j||_2^2)\rho^2 R^2 }{\lambda N} + \frac{||{\Gamma}^j||_\infty}{N} B(\Phi,\Gamma^j)\right)\sqrt{\frac{KN\log\frac{1}{\delta}}{2}} \enspace,
    \end{align*}
    where  $\dist_\mathcal{Y}(\mathcal{D}^{\Phi}_j, \mathcal{D}^{\Phi}_k)$ is the $\mathcal{Y}$-discrepancy between tasks $j$ and $k$ over $\mathcal{Z} \times \mathcal{Y}$ induced by $\Phi$.  Besides, the model complexity can be bounded in the same way as in Eq.~(\ref{lemmatwmeq}).
\end{corollary}

\subsubsection{Parameter Sharing}

Another widely used multitask learning approach is \emph{parameter sharing}, where the objective function is: 
\begin{align*}
    \min_{\{h_k\}_{k=1}^K}  \frac{1}{K} \sum_{k=1}^K \left[ \mathcal{L}_{{S}_k} (h_k) + \mathcal{R}_k(h_k) \right]\enspace.
\end{align*}
To exploit the commonalities  and  differences  across  the tasks, the parameter sharing approach assumes that for each task, the model parameter $h_k$ can be decomposed as $h_k=w_0 +w_k$, where $w_0$ is  a global parameter that is shared across tasks, and $w_k$ is a task-specific model parameter~\citep{evgeniou2004regularized,parameswaran2010large,liu2017algorithm}. To this end, we analyze the following multitask formulation:

{
\begin{align}\label{psmobj1}
     \hspace{-8pt}\min_{\{w_k\}_{k=0}^K}  \frac{1}{K} \sum_{k=1}^K \left[ \mathcal{L}_{{S}_k} (w_0+w_k) + \lambda_0 ||w_0||_2^2 + \lambda ||w_k||_2^2 \right]\enspace,
\end{align}
}where $\lambda_0$ and $\lambda$ are the regularization parameters to control the trade-off between the empirical loss and the model complexity. Furthermore, they also balance model diversity between the tasks. If $\lambda_0 \rightarrow \infty$, (\ref{psmobj1}) reduces to \emph{single-task} approaches, which solve $T$ tasks individually, and if $\lambda \rightarrow \infty$, (\ref{psmobj1}) reduces to \emph{pooling-task} approaches, which simply treat all the tasks as a single one.

For the parameter sharing approach to multitask learning, the performance gap is defined as follows.

\begin{definition}[{\bf Performance gap for parameter sharing}]
For any task $k$, let $\mathcal{V}_{k}(h) = \mathcal{L}_{{S}_k} (h) + \bar{\lambda} ||h||_2^2$ be the single-task loss, where $\bar{\lambda}=\frac{\lambda_0\lambda}{\lambda_0+\lambda}$,  and let the pooling-task loss be $\mathcal{V}_{0}(h) =  \frac{1}{K} \sum_{k=1}^K \left( \mathcal{L}_{{S}_k} (h) + \lambda_0 ||h||_2^2 \right)$. Let $\bar{h}_k$ and $\bar{h}_0$, respectively, be the minimizers of $\mathcal{V}_k$ and $\mathcal{V}_0$. The performance gap for parameter sharing is defined as:
\begin{align*}
    \nabla = K\mathcal{V}_0(\bar{h}_0)-\sum_{k=1}^K \mathcal{V}_k(\bar{h}_k)\enspace.
\end{align*}
\end{definition}

\begin{theorem}\label{psmtheorem}
    Let $\{w_k^*\}_{k=0}^K$ be the optimal solution of the parameter sharing multitask learning problem~(\ref{psmobj1}) and let $h_j^* = w_0^* + w_j^*, \forall j = 1,\dots, K$. Then, for any task $j$ and any $\delta \in (0,1)$, with probability at least $1-\delta$, we have:
    \begin{align*}
        \mathcal{L}_{\mathcal{D}_j}(h_j^*) &\le \mathcal{L}_{S_j}(h_j^*)  + \beta  + (2N\beta + B(\lambda_0, \lambda) )\sqrt{\frac{\log \frac{1}{\delta}}{2N}}\enspace,
    \end{align*}
    and the stability coefficient $\beta$ can be upper bounded by:
    \begin{align*}
        \beta \le \frac{\rho^2 R^2}{\lambda_0 KN} + \frac{\rho^2 R^2}{\lambda N}\enspace.
    \end{align*}
\end{theorem}
\begin{remark}
Theorem~\ref{psmtheorem} shows how the regularization parameters $\lambda_0$ and $\lambda$ control the stability coefficient $\beta_j$. In particular, $\beta_j$ achieves a fast convergence rate in $\mathcal{O}(\frac{1}{KN})$ as $\lambda \rightarrow \infty$, with the risk that the empirical loss can be high. $\lambda_0 \rightarrow \infty$ leads to the single-task solution with the convergence rate in  $\mathcal{O}(\frac{1}{N})$, and therefore there is no benefit of multitask learning. When $K=1$, we obtain the stability coefficient $\beta_j$ of single-task learning with the same regularization strength. If we further set $\lambda_0  = \lambda$, we can show that the ratio between overall convergence rates of multitask and single-task learning is $\frac{K+1}{2K} \approx \frac{1}{2}$. 
In other words, Theorem~\ref{psmtheorem} formalizes the intuition that the parameter sharing approach can be viewed as a trade-off between single-task and pooling-task learning. To obtain good generalization performances, one should the control balance between the model diversity and training loss across the tasks by tuning $\lambda_0$ and $\lambda$. Note that in order to make a fair comparison between the multitask and single-task learning methods, $\mathcal{L}_{\mathcal{D}_j}(h_j^*)$ is upper bounded in terms of $\mathcal{L}_{S_j}(h_j^*)$ rather than $\frac{1}{K} \sum_{k=1}^K \mathcal{L}_{S_k}(h_k^*)$.
\end{remark}
\begin{lemma}\label{psmlemma}
Let $\{w_k^*\}_{k=0}^K$ be the optimal solution of the parameter sharing multitask learning problem~(\ref{psmobj1}) and let $h_j^* = w_0^* + w_j^*, \forall j = 1,\dots, K$. Then, we have:
\begin{align*}
    ||h^*_j||_2 \le \sqrt{\frac{\nabla}{\bar{\lambda}}+||\bar{h}_j||_2^2}\enspace.
 \end{align*}
\end{lemma}
\begin{remark}
Lemma~\ref{psmlemma} connects the model complexity of parameter sharing multitask learning (\ref{psmobj1}) to the diversity between tasks, which is measured by the performance gap between  pooling-task learning and single-task learning. It formalizes the intuition that the parameter sharing  approach can work when tasks are similar to each other in a way such that pooling-task learning and single-task learning have similar learning performances. 
\end{remark}

\subsubsection{Task Covariance}\label{sec:tcmtl}
In addition to task weighting, another approach to capturing the task relationships is learning with a task covariance matrix. Specifically, the objective function is formulated as:
\begin{align}\label{tcmobj1}
    \min_{H}\frac{1}{K} \sum_{k=1}^K \mathcal{L}_{S_k}(h_k) + \mathcal{R}(H)\enspace,
\end{align}
where $H=[h_1,\dots, h_K]^\top$ is the ensemble of $K$ hypotheses,  $\mathcal{R}(H) = \text{tr}(H^\top \Omega^{-1} H )$. $\text{tr}(\cdot)$  denotes the trace of a square matrix, and $\Omega \in \mathbb{S}^K_+$ is a positive definite matrix that captures the pairwise relationships between tasks. From a probabilistic perspective, the regularizer $\mathcal{R}(H)$ corresponds to a matrix normal prior over $H$, with $\Omega$ capturing the covariance between individual hypotheses~\citep{zhang2010convex,zhang2015multi}.
\begin{remark}
Let $M = I-\frac{J}{K}$, where $J$ is the all-one matrix.
Then, it can be shown that (\ref{psmobj1}) is a special case of (\ref{tcmobj1}) with $\Omega^{-1} = \pi_1 I + \pi_2 M$, where $\pi_1 = \frac{\lambda_0 \lambda}{K(\lambda_0 + \lambda)}$ and  $\pi_2 = \frac{ \lambda^2}{K(\lambda_0 + \lambda)}$~\citep{evgeniou2004regularized}. 
\end{remark}

For the task covariance approach to multitask learning, the performance gap is defined as follows.

\begin{definition}[{\bf Performance gap for task covariance}]
For any task $k$, let $\mathcal{V}_k(h) = \mathcal{L}_{S_k}(h)+ {\frac{K}{\sigma_\text{max}}}||h||_2^2$ be the single-task loss, where $\sigma_\text{max}$ is the largest eigenvalue of $\Omega$, and $\mathcal{V}_0(h)= \frac{1}{K} \sum_{k=1}^K \mathcal{L}_{S_k}(h) + \omega ||h||_2^2$ is the pooling-task loss, where $\omega$ be the sum of the elements of $\Omega^{-1}$. Let $\bar{h}_k$ and $\bar{h}_0$, respectively, be the minimizers of $\mathcal{V}_k$ and $\mathcal{V}_0$. The performance gap for task covariance is defined as:
\begin{align*}
    \nabla = K\mathcal{V}_0(\bar{h}_0)-\sum_{k=1}^K \mathcal{V}_k(\bar{h}_k)\enspace.
\end{align*}
\end{definition}
\begin{theorem}\label{tcmtheorem}
    Let $H^*=[h^*_1,\dots, h^*_K]^\top$ be the optimal solution of the task covariance multitask learning problem~(\ref{tcmobj1}). Then, for any task $j$ and any $\delta \in (0,1)$, with probability at least $1-\delta$, we have:
    \begin{align*}
        \mathcal{L}_{\mathcal{D}_j}(h_j^*) \le \mathcal{L}_{S_j}(h_j^*) + \beta + (2N\beta + B(\Omega) )\sqrt{\frac{\log \frac{1}{\delta}}{2N}}\enspace,
    \end{align*}
    and the stability coefficient $\beta$ can be upper bounded by:
    \begin{align*}
        \beta \le \frac{\sigma_{\text{max}} \rho^2 R^2}{KN}\enspace.
    \end{align*}
\end{theorem}
\begin{remark}
When $\Omega^{-1} = \frac{\lambda}{K} I$, where $I$ is the identity matrix, we have $\sigma_{\text{max}} = \frac{K}{\lambda}$, and therefore the stability coefficient $\beta$ is upper bounded by $\frac{\rho^2 R^2}{\lambda N}$, which is consistent with the fact that (\ref{tcmobj1}) reduces to $K$ individual single-task problems: $\mathcal{L}_{S_k}(h_k) + \lambda ||h_k||_2^2$ with $\Omega^{-1} = \frac{\lambda}{K} I$. On the other hand,  
choosing $\Omega$ such that $\sigma_{\text{max}} < \frac{K}{\lambda}$ leads to a faster convergence rate of $\beta_j$, which indicates benefits of task covariance approach (\ref{tcmobj1}).   
\end{remark}

\begin{lemma}\label{tcmlemma}
 Let $H^*=[h_1^*,\dots, h_K^*]$ be the optimal solution of the task covariance multitask learning problem (\ref{tcmobj1}). Then, for any task  $j$, we have:
    \begin{align*}
           ||h^*_j||_2 \le \sqrt{\frac{\sigma_\text{max} \nabla}{K}+||\bar{h}_j||_2^2}\enspace.
    \end{align*}
\end{lemma}

\begin{remark}
Lemma~\ref{tcmlemma} reveals that the model complexity of the task covariance approach (\ref{tcmobj1}) can be bounded in a similar fashion as in the parameter sharing approach (\ref{psmobj1}). In addition,  Theorem~\ref{tcmtheorem} and Lemma~\ref{tcmlemma} also indicate that one should also regularize $\sigma_\text{max}$, which can be achieved by penalizing $\text{tr}(\Omega)$, or equivalently penalizing the trace norm of $H$~\citep{argyriou2007multi,pong2010trace,zhang2010convex,zhang2015multi}.
\end{remark}

\section{Algorithmic Instantiations}\label{algorithms}

We have now analyzed a variety of transfer and multitask learning approaches, and the theoretical results have consistently shown that in order to achieve good generalization performance, transfer and multitask learning algorithms should trade off empirical risk minimization and \emph{performance gap minimization}. To this end, we frame the general framework of performance gap minimization for transfer and multitask learning as:
\begin{align}\label{gapobj1}
    \min_{h, \Psi} \mathcal{L}_S(h, \Psi) + \lambda_1 \nabla(\Psi) + \lambda_2 r(h, \Psi)\enspace,
\end{align}
where $\Psi$ are the model parameters involved in gap minimization (e.g., instance weights, representation layers of a deep net, or a task covariance matrix between tasks). In Eq.~(\ref{gapobj1}), we have treated $\nabla$ as a data- and algorithm-dependent regularizer to guide the learning of $\Psi$, $r(h, \Psi)$ accommodates other regularization functions over $h$ and $\Psi$ -- e.g., $\ell_2$-norm regularization or marginal discrepancy between tasks~\citep{ben2010theory,ganin2016domain,zhou2021domain,zhou2021discriminative,shui2022benefits} -- and $\lambda_1, \lambda_2$ are the parameters that balance the trade-off between the loss function and the regularization functions. In practice, minimizing~(\ref{gapobj1}) leads to \emph{bilevel optimization} problems, which are challenging to optimize and require high computational cost for large-scale problems~\citep{bard2013practical,franceschi2018bilevel}. To see this, consider the instance weighting problem (\ref{itobj1}), which essentially boils down to two nested optimization problems: the inner objective is to find $h_\mathcal{S}$ and $h_\mathcal{T}$ by minimizing the source and target domain losses, and the outer objective is to minimize~(\ref{gapobj1}). To realize the principle of gap minimization sidestepping these optimization challenges, we present two algorithms, one for transfer learning and one for multitask learning. 

\subsection{$\pmb{\afunc{gapBoost}}$ and $\pmb{\afunc{gapBoostR}}$ for Transfer Learning}\label{sec:gapboost}

Our primary algorithmic goal in this paper is to derive a computationally efficient method that is flexible enough to accommodate arbitrary learning algorithms for knowledge sharing and transfer. To this end, we consider the instance weighting problem (\ref{itobj1}) as an instantiation of our framework (\ref{gapobj1}), which aims to assign appropriate values to $\Gamma$ so that the solution leads to effective transfer.  Instead of solving (\ref{itobj1}) by bilevel optimization, we follow four intuitively reasonable and principled rules motivated by Theorem~\ref{maintheorem1}:
\begin{itemize}
\item[1.] Minimize the weighted empirical loss over the source and target domains, as suggested by $\mathcal{L}^\Gamma_S$.
\item[2.] Assign balanced weights to data points (i.e., treat the source and target domains equally), as suggested by $\|\Gamma\|_2$ and $\|\Gamma\|_\infty$. 
\item[3.] Assign more weight to the target domain sample than to the source domain sample, as suggested by $\|\Gamma^\mathcal{S}\|_1$.
\item[4.] Assign the weights to the examples such that the performance gap $\nabla = \nabla_\mathcal{S}+\nabla_\mathcal{T}$ are small.  
\end{itemize}

\begin{algorithm}[t]\small
\caption{{ $\pmb{\afunc{gapBoost}}$}}
\label{alg:gapboost}
  \textbf{Input:} $S_\mathcal{S}, S_{\mathcal{T}}, K,   \rho_\mathcal{S} \le \rho_\mathcal{T}  \le 0, \gamma_{\text{max}}$, a learning algorithm $\mathcal{A}$
  \begin{algorithmic}[1]
  \STATE Initialize $D^\mathcal{S}_1(i) = D^\mathcal{T}_1(i)=\frac {1}{N_\mathcal{S}+N_\mathcal{T}}$
    \FOR{\(k=1,\dots,K\)}
        \STATE Call $\mathcal{A}$ to train a base learner $h_k$ using $S_\mathcal{S}\cup S_\mathcal{T}$ with distribution $D^\mathcal{S}_k \cup D^\mathcal{T}_k$
        \STATE Call $\mathcal{A}$ to train an auxiliary learner $h_k^\mathcal{S}$ over source domain using $S_\mathcal{S}$ with distribution $D^\mathcal{S}_k $
        \STATE Call $\mathcal{A}$ to train an auxiliary learner $h_k^\mathcal{T}$ over target domain using $S_\mathcal{T}$ with distribution $D^\mathcal{T}_k $
        \STATE ${\epsilon_{k} = \sum_{i=1}^{N_\mathcal{S}} D_k^\mathcal{S}(i) \mathbbm{1}_{h_k(x_i^\mathcal{S})\neq y_i^\mathcal{S}}+\sum_{i=1}^{N_\mathcal{T}} D_k^\mathcal{T}(i)\mathbbm{1}_{h_k(x_i^\mathcal{T}) \neq y_i^\mathcal{T} }}$
        \STATE $\alpha_k=\log\frac{1-\epsilon_{k}}{\epsilon_{k}}$
        \FOR{$i = 1,\dots, N_\mathcal{S}$}
            \STATE{$\beta_i^\mathcal{S} = \rho_\mathcal{S} \mathbbm{1}_{h_k^\mathcal{S}(x_i^\mathcal{S})  \neq h_k^\mathcal{T}(x_i^\mathcal{S})} + \alpha_k \mathbbm{1}_{ h_k(x_i^\mathcal{S}) \neq y_i^\mathcal{S}}$}
            \STATE {$D^\mathcal{S}_{k+1}(i)=D^\mathcal{S}_{k}(i)\exp\left( \beta^\mathcal{S}_i \right)$}
        \ENDFOR
        \FOR{$i = 1,\dots, N_\mathcal{T}$}
            \STATE{$\beta_i^\mathcal{T} = \rho_\mathcal{T} \mathbbm{1}_{ h_k^\mathcal{S}(x_i^\mathcal{T}) \neq h_k^\mathcal{T}(x_i^\mathcal{T})} + \alpha_k \mathbbm{1}_{ h_k(x_i^\mathcal{T}) \neq y_i^\mathcal{T}}$}
            \STATE {$D^\mathcal{T}_{k+1}(i)=D^\mathcal{T}_{k}(i)\exp\left( \beta^\mathcal{T}_i \right)$}
        \ENDFOR
        \STATE $Z_{k+1} = \sum_{i=1}^{N_\mathcal{S}} D_{k+1}^\mathcal{S}(i)+\sum_{i=1}^{N_\mathcal{T}} D_{k+1}^\mathcal{T}(i)$
        \IF{any $D^\mathcal{S}_{k+1}(i), D^\mathcal{T}_{k+1}(i) > \gamma_{\text{max}}Z_{k+1}$}
            \STATE $D^\mathcal{S}_{k+1}(i), D^\mathcal{T}_{k+1}(i) = \gamma_{\text{max}}Z_{k+1}$
        \ENDIF
        \STATE Normalize $D^\mathcal{S}_{k+1}$ and $D^\mathcal{T}_{k+1}$ such that $\sum_{i=1}^{N_\mathcal{S}} D_{k+1}^\mathcal{S}(i)+\sum_{i=1}^{N_\mathcal{T}} D_{k+1}^\mathcal{T}(i) =1$
    \ENDFOR
  \end{algorithmic}
  \textbf{Output:}  $f(x) = \text{sign} \left(\sum_{k=1}^K \alpha_k h_k(x) \right)$
\end{algorithm}

Note that these rules are somewhat contradictory. Thus, when designing an instance weighting algorithm for transfer learning, one should properly balance the rules to obtain a good generalization performance in the target domain.

To this end, we propose $\afunc{gapBoost}$ in Algorithm~\ref{alg:gapboost}, which exploits the rules explicitly. The algorithm trains a joint learner for source and target domains, as well as auxiliary source and target learners (lines 3--5). Then, it up-weights incorrectly labeled instances as per traditional boosting methods and down-weights instances for which the source and target learners disagree; the trade-off for the two schemes is controlled separately for source and target instances via hyper-parameters $\rho_{\mathcal{S}}$ and $\rho_{\mathcal{T}}$ (lines 6--15). Finally, the weights are clipped to a maximum value of $\gamma_\text{max}$ and normalized (lines 16--20). {\bf 1.} $\afunc{gapBoost}$ follows Rule 1 by training the base learner $h_k$ at each iteration, which aims to minimize the weighted empirical loss over the source and target domains. {\bf 2.} By tuning $\gamma_\text{max}$, it explicitly controls $\|\Gamma\|_\infty$ and implicitly controls $\|\Gamma\|_2$, as required by Rule 2. Additionally, as each base learner $h_k$ is trained with a different set of weights, the final classifier $f$ returned by $\afunc{gapBoost}$ is potentially trained over a balanced distribution. {\bf 3.} Moreover, by setting $\rho_\mathcal{T} \ge \rho_\mathcal{S}$, $\afunc{gapBoost}$ penalizes instances from the source domain more than from the target domain, implicitly assigning more weight to the target domain sample than to the source domain sample, as suggested by Rule~3. {\bf 4.} Finally, as $\rho_\mathcal{S}, \rho_\mathcal{T} \le 0$, the weight of any instance $x$ will decrease if the learners disagree (i.e., $h_k^{\mathcal{S}}(x) \neq h_k^{\mathcal{T}}(x)$). By doing so, $\afunc{gapBoost}$ follows Rule~4 by minimizing the gap $\nabla$. {\bf 5.} The trade-off between the rules is balanced by the choice of the hyper-parameters $\rho_\mathcal{T}$,  $\rho_\mathcal{S}$ and $\gamma_\text{max}$.

Table~\ref{tab:boost} compares $\afunc{gapBoost}$ various traditional boosting algorithms for transfer learning in terms of the instance weighting rules. Conventional AdaBoost~\citep{freund1997decision} treats source and target samples equally, and therefore does not reduce $\|\Gamma^\mathcal{S}\|_1$ or minimize the performance gap. On the other hand, TrAdaBoost~\citep{dai2007boosting} and  TransferBoost~\citep{eaton2011selective}, as described in Appendix~\ref{boosts}, explicitly exploit Rule 3 by assigning less weight to the source domain sample at each iteration. However, they do not control $\|\Gamma\|_\infty$ or $\|\Gamma\|_2$, so the weight of the target domain sample can be large after a few iterations. Most critically, none of the previous algorithms  minimize the performance gap explicitly as we do, which can be crucial for transfer learning to succeed.

\begin{table}[t]
\caption{Boosting algorithms for transfer learning.}
\vspace{-6pt}
\label{tab:boost}
\begin{center}
\begin{small}
\begin{tabular}{lcccc}
\toprule
                      & Rule 1    & Rule 2    & Rule 3      &  Rule 4 \\
\midrule
AdaBoost              & \checkmark   & \checkmark       &  \xmark       & \xmark \\
TrAdaBoost            & \checkmark   & \xmark       &  \checkmark   & \xmark \\
TransferBoost         & \checkmark   & \xmark       &  \checkmark   & \xmark \\
$\afunc{gapBoost}$    & \checkmark   & \checkmark   &  \checkmark   & \checkmark \\
\bottomrule
\end{tabular}
\end{small}
\end{center}
\vskip -0.1in
\end{table}

The generalization performance of $\afunc{gapBoost}$ is upper-bounded by the following proposition.

\begin{proposition}\label{boostpropositionA}
Let $f(x) = \sum_{k=1}^K \alpha_k h_k(x)$ be the ensemble of classifiers returned by $\afunc{gapBoost}$, with each base learner trained by solving (\ref{itobj1}).
For simplicity, we assume that $\sum_{k=1}^K\alpha_k =1$. Then, for any $\delta \in (0,1)$, with probability at least $1-\delta$, we have:
\begin{align*}
         {\textstyle \mathcal{L}_{\mathcal{D}_\mathcal{T}}(f)  \le  \mathcal{L}_{S_\mathcal{T}} (f)  + \frac{2\rho^2 R^2\gamma^\mathcal{T}_\infty}{\lambda}  \sqrt{2  \log \frac{4}{\delta}}   + B(\Gamma) \sqrt{\frac{\log \frac{2}{\delta}}{2N_\mathcal{T}}}\enspace. }
\end{align*}
where $\gamma_\infty^\mathcal{T}$ is the largest weight of the target sample over all boosting iterations.
\end{proposition}

\begin{remark}\label{remark7}
We observe that if $\gamma_\infty^\mathcal{T}\gg \sqrt{\frac{1}{N_\mathcal{T}}}$, the bound will be dominated by the second term. Then, Proposition~\ref{boostpropositionA} suggests to set ${\gamma_\text{max} = \mathcal{O}(\frac{1}{\sqrt{N_\mathcal{T}}})}$ to achieve a fast convergence rate. On the other hand, as the loss function is convex, $B(\Gamma)$ can be upper bounded by $B(\Gamma) \le \sum_{k=1}^K \alpha_k B(\Gamma_k)$, where $\Gamma_k$ is the set of weights at the $k$-th boosting iteration. In other words, one should aim to minimize the performance gap for every boosting iteration to achieve a tighter bound.
\end{remark}

\noindent {\bf Extending to Regression Problems }  Our principles can also be applied to boosting algorithms for regression~\citep{pardoe2010boosting,wu2019multiple}. In this section, we explore AdaBoost.R2~\citep{drucker1997improving} that has been shown generally effective and propose $\afunc{gapBoostR}$ for regression problems in Algorithm~\ref{alg:gapboostR}. The high-level idea of $\afunc{gapBoostR}$ is the same as that of $\afunc{gapBoost}$. The main difference is that the error in regression problems can be arbitrarily large since the output is real-valued. Therefore, we adjust it in the range $[0,1]$ by dividing it by the largest error at each iteration, and the weight of an instance is updated according to the value of the corresponding error.\footnote{We simply use the absolute error as adopted in~\citep{pardoe2010boosting}, but other error (e.g., squared error) is possible~\citep{drucker1997improving}.}

\begin{algorithm}[t]\small
\caption{{ $\pmb{\afunc{gapBoostR}}$}}
\label{alg:gapboostR}
  \textbf{Input:} $S_\mathcal{S}, S_{\mathcal{T}}, K,   \rho_\mathcal{S} \le \rho_\mathcal{T}  \le 0, \gamma_{\text{max}}$, a learning algorithm $\mathcal{A}$
  \begin{algorithmic}[1]
  \STATE Initialize $D^\mathcal{S}_1(i) = D^\mathcal{T}_1(i)=\frac {1}{N_\mathcal{S}+N_\mathcal{T}}$ for all $i$
    \FOR{\(k=1,\dots,K\)}
        \STATE Call $\mathcal{A}$ to train a base learner $h_k$ using $S_\mathcal{S}\cup S_\mathcal{T}$ with distribution $D^\mathcal{S}_k \cup D^\mathcal{T}_k$
        \STATE Call $\mathcal{A}$ to train an auxiliary learner $h_k^\mathcal{S}$ over source domain using $S_\mathcal{S}$ with distribution $D^\mathcal{S}_k $
        \STATE Call $\mathcal{A}$ to train an auxiliary learner $h_k^\mathcal{T}$ over target domain using $S_\mathcal{T}$ with distribution $D^\mathcal{T}_k $
        \STATE $E_k=\max\left\{\left\{\left| h_k(x_i^\mathcal{S}) - y_i^\mathcal{S} \right|\right\}_{i=1}^{N_\mathcal{S}}, \left\{\left| h_k(x_i^\mathcal{T}) - y_i^\mathcal{T} \right|\right\}_{i=1}^{N_\mathcal{T}} \right\}$, \\
         $\epsilon_{k,i}^\mathcal{S} = \frac{ \left| h_k(x_i^\mathcal{S}) - y_i^\mathcal{S} \right|}{E_k}$ and $\epsilon_{k,i}^\mathcal{T} = \frac{\left| h_k(x_i^\mathcal{T}) - y_i^\mathcal{T} \right|}{E_k}$
        \STATE ${\epsilon_{k} = \sum_{i=1}^{N_\mathcal{S}} D_k^\mathcal{S}(i) \epsilon_{k,i}^\mathcal{S} +\sum_{i=1}^{N_\mathcal{T}} D_k^\mathcal{T}(i) \epsilon_{k,i}^\mathcal{T}}$, $\alpha_k=\log\frac{1-\epsilon_{k}}{\epsilon_{k}}$
        \STATE $E_k^\mathcal{S}=\max\left\{\left| h_k^\mathcal{S}(x_i^\mathcal{S})  - h_k^\mathcal{T}(x_i^\mathcal{S}) \right|\right\}_{i=1}^{N_\mathcal{S}}$, $\kappa_{k,i}^\mathcal{S} = \frac{\left| h_k^\mathcal{S}(x_i^\mathcal{S})  - h_k^\mathcal{T}(x_i^\mathcal{S}) \right|}{E_k^\mathcal{S}}$ 
        \STATE \mbox{$E_k^\mathcal{T}=\max\left\{\left| h_k^\mathcal{S}(x_i^\mathcal{T})  - h_k^\mathcal{T}(x_i^\mathcal{T}) \right|\right\}_{i=1}^{N_\mathcal{T}}$, $\kappa_{k,i}^\mathcal{T} = \frac{\left| h_k^\mathcal{S}(x_i^\mathcal{T})  - h_k^\mathcal{T}(x_i^\mathcal{T}) \right|}{E_k^\mathcal{T}}$} 
        \FOR{$i = 1,\dots, N_\mathcal{S}$}
            \STATE{$\beta_i^\mathcal{S} = \rho_\mathcal{S} \kappa_{k,i}^\mathcal{S} + \alpha_k \epsilon_{k,i}^\mathcal{S}$}, {$D^\mathcal{S}_{k+1}(i)=D^\mathcal{S}_{k}(i)\exp\left( \beta^\mathcal{S}_i \right)$}
        \ENDFOR
        \FOR{$i = 1,\dots, N_\mathcal{T}$}
            \STATE{$\beta_i^\mathcal{T} = \rho_\mathcal{T} \kappa_{k,i}^\mathcal{T} + \alpha_k \epsilon_{k,i}^\mathcal{T}$}, {$D^\mathcal{T}_{k+1}(i)=D^\mathcal{T}_{k}(i)\exp\left( \beta^\mathcal{T}_i \right)$}
        \ENDFOR
        \STATE $Z_{k+1} = \sum_{i=1}^{N_\mathcal{S}} D_{k+1}^\mathcal{S}(i)+\sum_{i=1}^{N_\mathcal{T}} D_{k+1}^\mathcal{T}(i)$
        \IF {$D^\mathcal{S}_{k+1}(i), D^\mathcal{T}_{k+1}(i) > \gamma_{\text{max}}Z_{k+1}$}
            \STATE $D^\mathcal{S}_{k+1}(i), D^\mathcal{T}_{k+1}(i) = \gamma_{\text{max}}Z_{k+1}$
        \ENDIF
        \STATE Normalize $D^\mathcal{S}_{k+1}$ and $D^\mathcal{T}_{k+1}$ such that $\sum_{i=1}^{N_\mathcal{S}} D_{k+1}^\mathcal{S}(i)+\sum_{i=1}^{N_\mathcal{T}} D_{k+1}^\mathcal{T}(i) =1$
    \ENDFOR
    \STATE Normalize $\alpha_k$ such that $\sum_{k=1}^K \alpha_k = 1$.
  \end{algorithmic}
  \textbf{Output:}  $f(x) =\sum_{k=1}^K \alpha_k h_k(x)$
\end{algorithm}

\subsection{$\pmb{\afunc{gapMTNN}}$ for Multitask Learning}\label{sec:gapMTNN}

 In this section, we instantiate the principle of performance gap minimization with \emph{gap multitask neural network} ($\afunc{gapMTNN}$), which jointly learns task relation coefficients and a feature representation for multitask learning. Specifically, we consider a widely used deep multitask 
 learning architecture~\citep{kumar2012learning,maurer2014aninequality,maurer2016benefit,shui2019principled,zhou2021multi} that consists of a feature extractor $\Phi$ shared across tasks and task-specific classifiers $\{h_1, \dots, h_K\}$. Then, the weighted empirical loss for the $j$-th task can be defined as:
   \begin{align*}
      \mathcal{L}_S^j(\Phi,h_j,\Gamma^j) = \sum_{k=1}^K\gamma_k^j  \frac{1}{N}\sum_{i=1}^{N}\ell(h_j(\Phi(x_i^k)),y_i^k)\enspace.
  \end{align*}
 In the view of Corollary~\ref{twrlmcor1} and Eq.~(\ref{gapobj1}), $\afunc{gapMTNN}$ aims to minimize the following regularized weighted empirical loss function:
    \begin{align}\label{gapmobj1}
      \min_{\{h_j\}_{j=1}^K,\Phi,\{\Gamma^j\}_{j=1}^K}   \sum_{j=1}^K  \mathcal{E}_j \enspace, \qquad {\text{s.t. }} \Gamma^j \in \Upsilon^K, \forall j \enspace,
  \end{align}
  where:
   \begin{align*} 
         \mathcal{E}_j = \mathcal{L}_S^j (\Phi,h_j,\Gamma^j) + \lambda_1  \nabla_j(\Phi,\Gamma^j) + \lambda_2  r_j(\Phi,\Gamma^j) + \lambda_3 ||\Gamma^j||_2^2\enspace,
  \end{align*}
and $r_j(\Phi,\Gamma^j) = \sum_{k\neq j} \gamma_k^j \dist (\mathcal{D}_j^{\Phi}, \mathcal{D}_k^{\Phi})$ measures the marginal discrepancy between task $j$ and the other tasks, which can be minimized by centroid alignment~\citep{xie2018learning,dou2019domain,zhou2021multi}. 

In order to minimize the gap term $\nabla_j$, we note that for the task weighting approach, if we ignore the regularization term in $\mathcal{V}_j^\Phi(h)$, performance gap minimization is equivalent to conditional distribution alignment. To see why, note that by definition, $\nabla_j = 0$ (regardless of the loss function) when $\bar{h}_j = \bar{h}_k$ over $S_k, \forall k \neq j$. On the other hand, since we can treat $\bar{h}_j(\Phi(x))$ as an empirical model of the conditional probability $\mathcal{D}_j^\Phi(y|x)$ over the representation space induced by $\Phi$ (e.g., using a softmax function), then $\nabla_j$ vanishes when $\mathcal{D}^\Phi_j(y|x) = \mathcal{D}^\Phi_k(y|x), \forall j \neq k$. In other words, $\nabla_j$ can be viewed as the overall disagreement between conditional probabilities of the $j$-th task and the other tasks. In consequence, we adopt centroid alignment to approach gap minimization due to its  simplicity and effectiveness~\citep{luo2017label,snell2017prototypical}. Specifically, let $S_k^c$ be the set of instances of the $k$-th task labeled with class $c$. Its centroid in the feature space is defined as $C_k^c = \frac{1}{|S_k^c|}\sum_{(x_i^k, y_i^k)\in S_k} \Phi(x_i^k)$.
Then, if all the classes are equally likely, the conditional probability can be formulated as~\citep{snell2017prototypical}:
\begin{align*}
    p(y = c|\Phi(x), (x,y) \sim \mathcal{D}_k) =\frac{\exp (-d(\Phi(x),C_k^c))}{\sum_{c'} \exp (-d(\Phi(x),C_k^{c'}))}\enspace,
\end{align*}
where $d(\cdot,\cdot)$ is a distance measure function (e.g., the squared Euclidean distance). It can be observed that the conditional probability can be aligned if $C_k^c = C_j^c, \forall j \neq k$. Therefore, gap minimization can be achieved by minimizing the distance of the centroids across the domains: 
\begin{align*}
    \nabla_j \approx \sum_{k \neq j} \gamma_k^j\sum_c \left|\left|C_j^c -C_k^c\right|\right|_2^2\enspace,
\end{align*}
which can be solved by aligning \emph{moving average centroids}~\citep{xie2018learning}. 

On the other hand, the centroid can also be viewed as an approximation of semantic conditional distribution if the instances on the latent space follow a mixture of  Gaussian distributions with an identical covariance matrix: $p(\Phi(x)|y=c, (x,y) \sim \mathcal{D}_k) = \mathcal{N}(\mu_k^c,\Sigma)$~\citep{shui2021aggregating}, where $\mathcal{N}(\mu,\Sigma)$ is the Gaussian distribution with mean $\mu$ and covariance matrix $\Sigma$. Then, we develop the following proposition, which provides theoretical insights into representation learning for multitask learning. 

\begin{proposition}\label{gapMTNNproposition}
We assume the predictive loss is positive and $\rho$-Lipschitz. We also define a stochastic representation function $\Phi$ and we denote its conditional distribution as $\mathcal{D}(z|y) = \int_{x} \Phi(z|x)\mathcal{D}(x|y)$. We additionally assume the predictor $h$ is $\nu$-Lipschitz. Then, for any task $j$ the expected error in the multitask setting is upper-bounded by:
\begin{align*}
    \mathcal{L}_{\mathcal{D}_j}(h_j^*) \le \mathcal{L}^{\Gamma^j}_{\mathcal{D}}(h_j^*)+  \frac{\nu\rho}{|\mathcal{Y}|} \sum_{k \neq j} \gamma^j_k \sum_{y} W_1(\mathcal{D}_t(z|y)\|\mathcal{D}_k(z|y))\enspace,
\end{align*}
where $|\mathcal{Y}|$ is the number of classes, and $W_1(\cdot||\cdot)$ is the Wasserstein-1 distance with $\ell_2$ distance as the cost function.
\end{proposition}

\begin{remark}
 It is worth mentioning that under proper assumptions, similar theoretical results could be derived through other divergences such as the Jensen-Shannon (JS) divergence or the Total Variation (TV) distance \citep{tachet2020domain}.  We present this result in terms of the Wasserstein distance because it captures distribution shift better than the JS or TV distances in cases where the distribution supports are not identical \citep{arjovsky2017wasserstein}.
\end{remark}

Proposition~\ref{gapMTNNproposition} introduces a (probabilistic) representation function $\Phi(z|x)$. The resulting upper bound on the risk is related to two functions: the representation function and the prediction function. In contrast, conventional divergence-based theories operate directly on the input space $\mathcal{X}$ \emph{without} considering the effect of the representation function $\Phi(z|x)$. Proposition~\ref{gapMTNNproposition}  provides a principled result for understanding the role of the representation function in multitask learning.

The proposed $\afunc{gapMTNN}$ algorithm is described in Algorithm~\ref{alg:gapMTNN}. The high-level protocol is to alternately optimize the neural network parameters $\Phi, \{h_k\}_{k=1}^K$ and the task relation coefficients $\{\Gamma^k\}_{k=1}^K$. Concretely, within one training epoch over the mini-batches, we fix $\{\Gamma^k\}_{k=1}^K$ and  optimize the network parameters. Then, at each training epoch, we
re-estimate $\{\Gamma^k\}_{k=1}^K$ via standard convex optimization\footnote{In our implementation, we use the  CVXPY package for re-estimating the $\Gamma$'s, as  it is a standard convex optimization problem.}. Note that the regularization term $||\Gamma^j||_2^2$ in $\mathcal{E}_j$ leads to a quadratic programming problem, which encourages a uniform task weighting scheme and prevents the relation coefficients from focusing  only on $\gamma_j^j$. The regularization parameter $\lambda_3$ balances the trade-off between learning the single task and leveraging information from the other tasks.
 
\begin{algorithm}[t]\small
\caption{{ $\pmb{\afunc{gapMTNN}}$} (one epoch)}
\label{alg:gapMTNN}
  \textbf{Input:} \mbox{$\{S_k\}_{k=1}^K$, initial task weights $\{\Gamma^k\}_{k=1}^K$, and a learning rate $\alpha$.}
  \begin{algorithmic}[1]
  \STATE $\triangleright$ DNN Parameter Training Stage (given $\{\Gamma^k\}_{k=1}^K$ ) $\triangleleft$
        \FOR{min-batch $\{(x^k_i,y^k_i)\}$ from task $\{S_k\}_{k=1}^K$}
        \STATE Compute the  feature centroids for each task on the current batch; use moving average to update the class centroids $\{C_k^c\}_{c,k=1}^{|\mathcal{Y}|,K}$.
        \STATE Update the network parameters $\Phi, \{h_k\}_{k=1}^K$ by
        \begin{align*}
            \Phi \leftarrow \Phi -\alpha \frac{\partial \sum_{k=1}^K \mathcal{L}_S^k + \lambda_1 \nabla_k + \lambda_2 r_k}{\partial \Phi} \text{ and } h_k \leftarrow  h_k - \alpha \frac{\partial \mathcal{L}_S^k}{\partial h_k}
        \end{align*}
        \ENDFOR
		\STATE  $\triangleright$ Update $\{\Gamma^k\}_{k=1}^K$ by optimizing over Eq.~(\ref{gapmobj1}) (given $\Phi$ and $\{h_k\}_{k=1}^K$) $\triangleleft$
  \end{algorithmic}
  \textbf{Output:}  neural network parameters $\Phi, \{h_k\}_{k=1}^K$ and task relation coefficients $\{\Gamma^k\}_{k=1}^K$
\end{algorithm}

\section{Experiments} \label{experiments}

In this section, we evaluate the proposed algorithms on both transfer and multitask learning scenarios. 

\begin{table*}[t]
\caption{Comparison of different methods on the 20 Newsgroups (top) and Office-Caltech (bottom) data sets in term of error rate (\%). 
$\afunc{gapBoost}$ outperforms all baselines in the majority of transfer problems, and is competitive with the top performance in the remaining ones.  Notably, the performance of $\afunc{gapBoost}$ is considerably better in all cases than the no-transfer baseline, $\text{AdaBoost}_{\mathcal{T}}$.
Standard error is reported after the $\pm$.}
\vspace{-6pt}
\label{tab:results}
\begin{center}
\begin{small}
\begin{tabular}{p{60pt}p{60pt}p{60pt}p{60pt}p{60pt}p{60pt}}
\toprule
                      & $\,\,\,\text{AdaBoost}_\mathcal{T}$    & $\text{AdaBoost}_{\mathcal{T}\&\mathcal{S}}$    & \,\,\,\,TrAdaBoost     &  \,\,\,TransferBoost & \,\,\,\,\,\,\,\,$\afunc{gapBoost}$\\
\midrule
comp vs sci             & \,\, $12.45_{\pm 0.47}$   & \,\, $13.45_{\pm 0.48}$      & \,\, $\,\,12.03_{\pm 0.41}$    & \,\, $\,\,\,\,\,8.83_{\pm 0.37}$ & \,\,  $\,\,\,\,\bf 7.68_{\pm 0.25}$\\
rec vs sci              & \,\, $10.99_{\pm 0.37}$   & \,\, $11.79_{\pm 0.35}$      & \,\,  $\,\,10.03_{\pm 0.36}$    & \,\, $\,\,\,\,\,7.93_{\pm 0.30}$ & \,\, $\,\,\,\,\bf 7.39_{\pm 0.21}$\\
comp vs talk            & \,\, $11.83_{\pm 0.42}$   & \,\, $14.57_{\pm 0.47}$      & \,\,  $\,\,10.67_{\pm 0.37}$    & \,\, $\,\,\,\,\bf 6.45_{\pm 0.25}$ & \,\, $\,\,\,\,\,7.10_{\pm 0.27}$\\
comp vs rec             & \,\, $15.80_{\pm 0.53}$   & \,\, $17.50_{\pm 0.64}$      & \,\,  $\,\,14.86_{\pm 0.67}$    & \,\, $\,\,12.11_{\pm 0.43}$ & \,\, $\,\,\,\,\bf 9.81_{\pm 0.29}$\\
rec vs talk             & \,\, $12.08_{\pm 0.36}$   & \,\, $\,\,\,9.40_{\pm 0.31}$       & \,\,  $\,\,12.21_{\pm 0.40}$    & \,\, $\,\,\,\,\,6.26_{\pm 0.30}$ & \,\, $\,\,\,\,\bf 5.66_{\pm 0.21}$\\
sci vs talk             & \,\, $11.74_{\pm 0.49}$   & \,\, $10.52_{\pm 0.37}$      & \,\,  $\,\,10.13_{\pm 0.46}$    & \,\, $\,\,\,\,\,6.45_{\pm 0.26}$ & \,\, $\,\,\,\,\bf 5.92_{\pm 0.24}$\\
avg. & \, \hspace{9pt} $12.48$ & \hspace{14pt} $12.87$ & \hspace{18pt} $11.66$ &  \hspace{22pt} $8.00$ &  \hspace{22pt}  $\bf 7.26$\\ 
\hline
A $\rightarrow$ C       & \,\, $43.87 _{\pm 0.52}$   & \,\, $27.76 _{\pm 0.88}$   & \,\,  $\,\,37.57 _{\pm 0.68}$   & \,\, $\,\,27.86 _{\pm 0.82}$      & \,\, $\bf 27.06 _{\pm 0.87}$ \\
A $\rightarrow$ D       & \,\, $32.65 _{\pm 1.35}$   & \,\, $28.33 _{\pm 1.33}$   & \,\,  $\,\,34.93 _{\pm 1.43}$   & \,\, $\,\,28.96 _{\pm 1.38}$      & \,\, $\bf25.08 _{\pm 1.37}$ \\
A $\rightarrow$ W       & \,\, $37.23 _{\pm 0.98}$   & \,\, $26.94 _{\pm 1.17}$   & \,\,  $\,\,31.03 _{\pm 0.95}$   & \,\, $\,\,26.95 _{\pm 1.15}$      & \,\, $\bf 24.34 _{\pm 1.10}$ \\
C $\rightarrow$ A       & \,\, $39.92 _{\pm 0.74}$   & \,\, $20.32 _{\pm 0.80}$   & \,\,  $\,\,29.13 _{\pm 0.80}$   & \,\, $\,\,19.68 _{\pm 0.80}$      & \,\, $\bf 19.13 _{\pm 0.83}$ \\
C $\rightarrow$ D       & \,\, $27.88 _{\pm 1.14}$   & \,\, $25.69 _{\pm 1.19}$   & \,\,  $\bf 19.84 _{\pm 1.09}$   & \,\, $\,\,23.44 _{\pm 1.33}$      & \,\, $\,\,21.03 _{\pm 1.20}$ \\
C $\rightarrow$ W       & \,\, $30.25 _{\pm 1.05}$   & \,\, $24.50 _{\pm 1.30}$   & \,\,  $\,\,22.86 _{\pm 0.95}$   & \,\, $\,\,23.41 _{\pm 1.30}$      & \,\, $\bf 21.55 _{\pm 1.20}$ \\
D $\rightarrow$ A       & \,\, $44.30 _{\pm 0.45}$   & \,\, $40.86 _{\pm 0.39}$   & \,\,  $\,\,45.33 _{\pm 0.48}$   & \,\, $\bf 40.50 _{\pm 0.44}$      & \,\, $\,\,40.66 _{\pm 0.39}$ \\
D $\rightarrow$ C       & \,\, $44.00 _{\pm 0.56}$   & \,\, $40.09 _{\pm 0.46}$   & \,\,  $\,\,43.72 _{\pm 0.62}$   & \,\, $\,\,40.35 _{\pm 0.46}$      & \,\, $\bf 40.00 _{\pm 0.46}$ \\
D $\rightarrow$ W       & \,\, $50.63 _{\pm 0.58}$   & \,\, $49.64 _{\pm 0.66}$   & \,\,  $\,\,49.95 _{\pm 0.65}$   & \,\, $\bf 49.63 _{\pm 0.65}$      & \,\, $\,\,50.24 _{\pm 0.62}$ \\
W $\rightarrow$ A       & \,\, $42.91 _{\pm 0.46}$   & \,\, $37.22 _{\pm 0.56}$   & \,\,  $\,\,44.24 _{\pm 0.52}$   & \,\, $\bf 37.02 _{\pm 0.53}$      & \,\, $\,\,37.04 _{\pm 0.52}$ \\
W $\rightarrow$ C       & \,\, $44.12 _{\pm 0.50}$   & \,\, $37.93 _{\pm 0.58}$   & \,\,  $\,\,44.78 _{\pm 0.65}$   & \,\, $\,\,37.79 _{\pm 0.56}$      & \,\, $\bf 37.48 _{\pm 0.50}$ \\
W $\rightarrow$ D       & \,\, $40.63 _{\pm 1.45}$   & \,\, $45.52 _{\pm 1.58}$   & \,\,  $\bf 40.00 _{\pm 1.51}$   & \,\, $\,\,44.88 _{\pm 1.58}$      & \,\, $\,\,41.74 _{\pm 1.40}$ \\
avg. & \, \hspace{14pt} $39.86$ &   \hspace{14pt}  $33.73$ &   \hspace{16pt} $36.95$ &  \hspace{16pt}  $33.37$ &   \hspace{16pt} $\bf 32.11$ \\
\bottomrule
\end{tabular}
\end{small}
\end{center}
\vskip -0.1in
\end{table*}

\subsection{Transfer Learning}

We evaluate $\afunc{gapBoost}$ on two benchmark data sets for classification. The first data set we consider is \textbf{20 Newsgroups}\footnote{Available at: \url{http://qwone.com/~jason/20Newsgroups/}.}, which contains approximately 20,000 documents, grouped by seven top categories and 20 subcategories. Each transfer learning task involved a top-level classification problem, while the source  and target domains were chosen from different subcategories. The source and target data sets were in the same way as in~\citep{dai2007boosting}, yielding 6 transfer learning problems. 
The second data set we use is \textbf{Office-Caltech}~\citep{gong2012geodesic}, which contains approximately 2,500 images from four distinct domains: Amazon (A), DSLR  (D), Webcam (W), and Caltech (C), which enabled us to construct 12 transfer problems by alternately selecting each possible source-target pair. All four domains share the same 10 classes, so we constructed 5 binary classification tasks for each transfer problem and the averaged results are reported. 

For $\afunc{gapBoostR}$,  we  evaluate it on five benchmark data sets:  \textbf{Concrete},  \textbf{Housing},  \textbf{AutoMPG}, \textbf{Diabetes}, and \textbf{Friedman}. The first three data sets are from the UCI Machine Learning Repository\footnote{Available at: \url{http://www.ics.uci.edu/~mlearn/MLRepository.html}.},
and the Diabetes data set is from~\citep{efron2004least}. Following~\citep{pardoe2010boosting}, for each data set,
we identify a continuous feature that has a moderate
degree of correlation (around 0.4) with the label. Then, we sort the instances by this feature, divide the set in thirds (low, medium, and high), and remove this feature from the resulting sets. We use one set as the target and the other two as
sources, for a total of three experiments for each data set, and the averaged performances are reported. In the Friedman data set~\citep{friedman1991multivariate}, each instance consists of 10 features, with each component $x_i$ drawn independently from the uniform distribution $[0,1]$. The label for each instance is dependent on only the first five features:
\begin{align*}
    y &= a_1 10 \sin(\pi(b_1x_1+c_1) (b_2x_2+c_2))+a_2 20 (b_3x_3+c_3-0.5)^2 \\
    &+ a_310(b_4 x_4+c_4)+a_4 5 (b_5 x_5 + c_5) +\mathcal{N}(0,1) \enspace,
\end{align*}
where $a_i$, $b_i$, and $c_i$ are fixed parameters. To construct regression transfer problems, we follow~\citep{pardoe2010boosting} and set $a_i=b_i=1=c_i=1$ for the target domain, and draw each $a_i$ and $b_i$ from $\mathcal{N}(1,0.1)$ and each $c_i$ from $\mathcal{N}(1,0.05)$ for each source domain.

\begin{table*}[t]
\caption{Comparison of different methods on five regression benchmark data sets in term of RMS error. 
$\afunc{gapBoostR}$ consistently outperforms all the other baselines.  
Standard error is reported after the $\pm$.}
\vspace{-6pt}
\label{tab:resultsR}
\begin{center}
\begin{small}
\begin{tabular}{p{40pt}cccc}
\toprule
                      & $\text{AdaBoost.R2}_\mathcal{T}$    & $\text{AdaBoost.R2}_{\mathcal{T}\&\mathcal{S}}$    & TrAdaBoost.R2     & $\afunc{gapBoostR}$\\
\midrule
Concrete        & $12.06 _{\pm 0.89}$   & $ \, \, \,  9.91_{\pm 0.50}$   & $10.25 _{\pm 0.77}$      & $\bf \, \, \,   7.77 _{\pm 0.38}$ \\
Housing       & $ \, \, \,  6.63 _{\pm 0.82}$   & $ \, \, \,  5.52 _{\pm 0.35}$   & $\,\,\, 6.19 _{\pm 0.71}$      & $\bf \, \, \,  4.15 _{\pm 0.27}$ \\
AutoMPG       & $ \, \, \,  5.24 _{\pm 0.35}$   & $\, \, \,  3.69 _{\pm 0.17}$   & $\, \, \, 3.69 _{\pm 0.14}$      & $ \bf \, \, \, 3.25 _{\pm 0.16}$ \\
Diabetes       & $76.32 _{\pm 4.38}$   & $74.08 _{\pm 2.77}$   & $67.89 _{\pm 3.34}$      & $\bf 59.01 _{\pm 1.32}$ \\
Friedman        & $\, \, \, 3.87 _{\pm 0.29}$   & $\, \, \,  4.92 _{\pm 0.16}$   & $\, \, \,  4.57 _{\pm 0.11}$      & $\bf \, \, \,   2.91 _{\pm 0.17}$ \\
avg. & \hspace{3pt} $20.82$ & \hspace{3pt}  $19.62$ &  \hspace{3pt} $18.52$ & \,\,  \hspace{3pt} $\bf 15.42$ \\
\bottomrule
\end{tabular}
\end{small}
\end{center}
\vskip -0.1in
\end{table*}

\subsubsection{Performance Comparison}\label{sec:performance}
We evaluated $\afunc{gapBoost}$ against four baseline algorithms: $\text{AdaBoost}_\mathcal{T}$ trained only on target data, $\text{AdaBoost}_{\mathcal{T}\&\mathcal{S}}$ trained on both source and target data, TrAdaBoost, and TransferBoost. Logistic regression is used as the base learner for all methods, and the number of boosting iterations is set to 20. The hyper-parameters of $\afunc{gapBoost}$ were set as $\gamma_\text{max} = \frac{1}{\sqrt{N_\mathcal{T}}}$ as per Remark~\ref{remark7}, $\rho_\mathcal{T} = 0$, which corresponds to no punishment for the target data, and $\rho_\mathcal{S} = \log \frac{1}{2}$. 

In both data sets we pre-processed the data using principal component analysis (PCA) to reduce the feature dimension to 100. For each data set, we used all source data and a small amount of target data (10\% on 20 Newsgroups and 10 points on Office-Caltech) as training data, and used the rest of the target data for testing. We repeated all experiments over 20 different random train/test splits and the average results are presented in Table~\ref{tab:results}, showing that our method is capable of outperforming all the baselines in the majority of cases. In particular, $\afunc{gapBoost}$ consistently outperforms $\text{AdaBoost}_\mathcal{T}$, empirically indicating that it avoids \emph{negative transfer}.

For regression tasks, we evaluated $\afunc{gapBoostR}$ against $\text{AdaBoost.R2}_\mathcal{T}$, $\text{AdaBoost.R2}_{\mathcal{T}\&\mathcal{S}}$, and TrAdaBoost.R2,  an extension of TrAdaBoost to the regression scenario~\citep{pardoe2010boosting}, as described in Appendix~\ref{boosts}. The results, reported in Table~\ref{tab:resultsR}, demonstrate that $\afunc{gapBoostR}$ consistently outperforms all the other baselines for all the regression tasks.

\subsubsection{Learning with different sizes of target sample}
To further investigate the effectiveness of the gap minimization principle for transfer learning, we varied  the fraction of target instances of the 20 Newsgroups data set used for training, from 0.01 to 0.8. Figure~\ref{fig:ratio} shows full learning curves on three example tasks, as well as the average performance over all six tasks. The results reveal that $\afunc{gapBoost}$'s improvement over the baselines increases as the number of target instances grows, indicating that it is able to leverage target data more effectively than previous methods.

\begin{figure*}[t]\small
\centering 
\hfill\includegraphics[width=.3\textwidth]{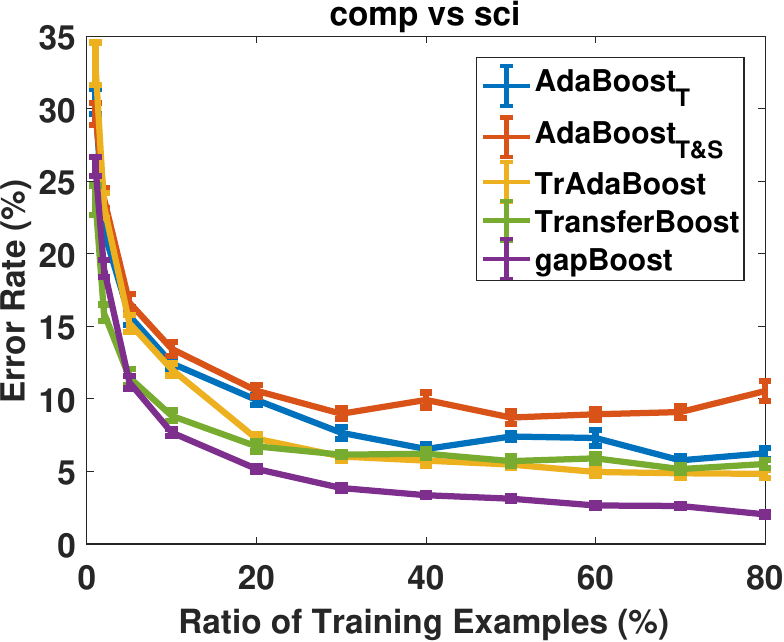}
\hspace{10pt}\includegraphics[width=.3\textwidth]{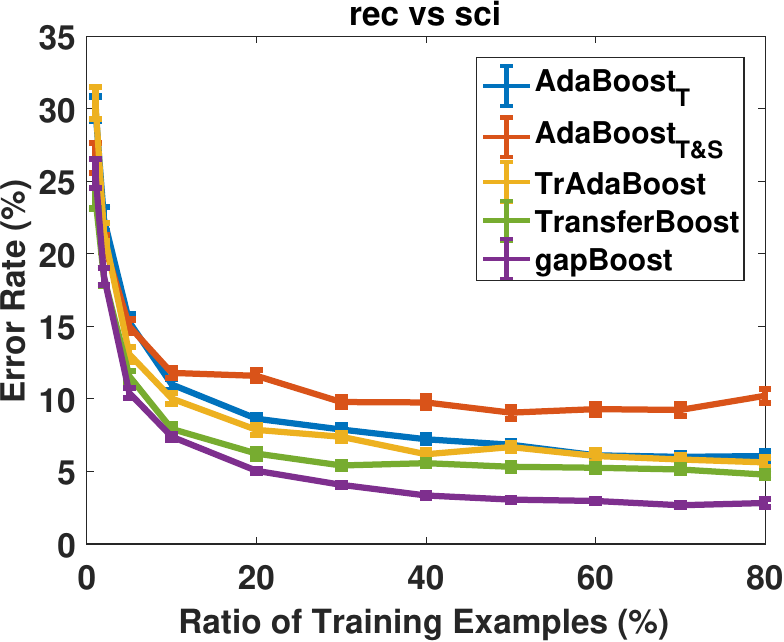}
\hfill~\\[1em]
\hfill\includegraphics[width=.3\textwidth]{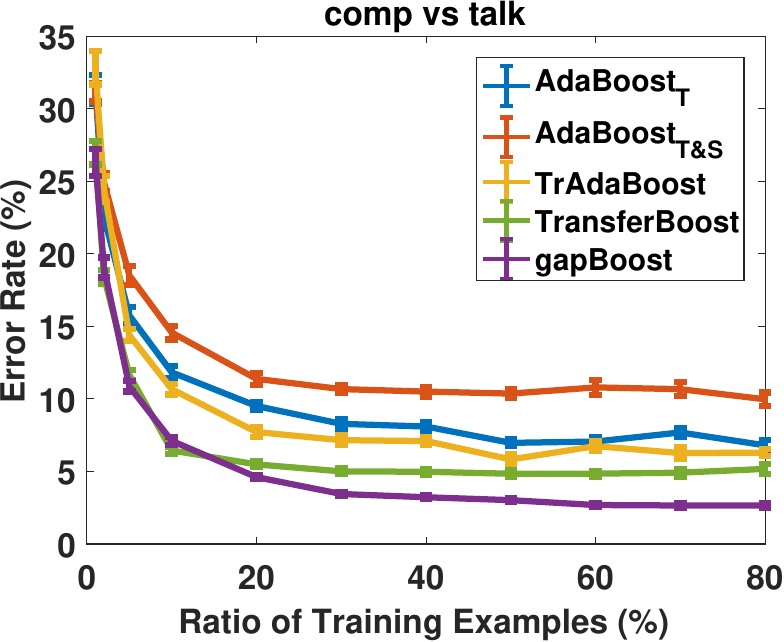}
\hspace{10pt}\includegraphics[width=.3\textwidth]{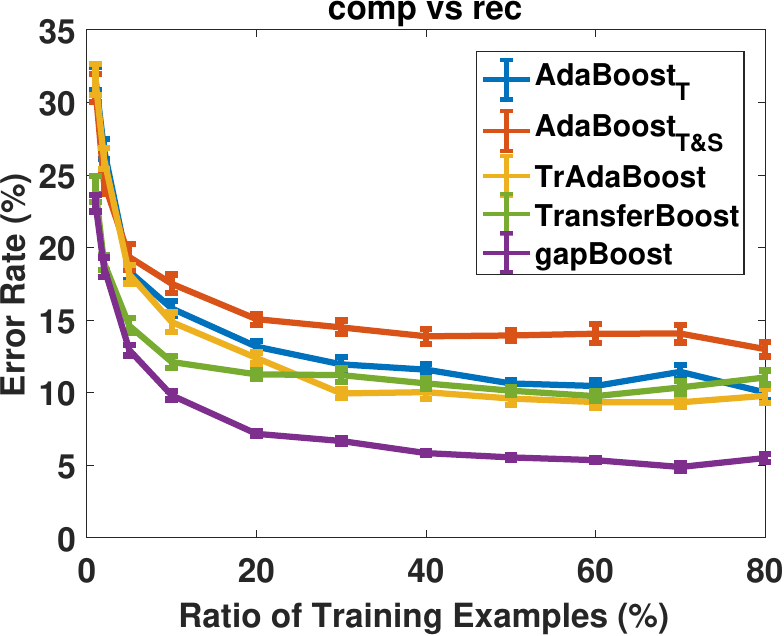}
\hspace{10pt}\includegraphics[width=.3\textwidth]{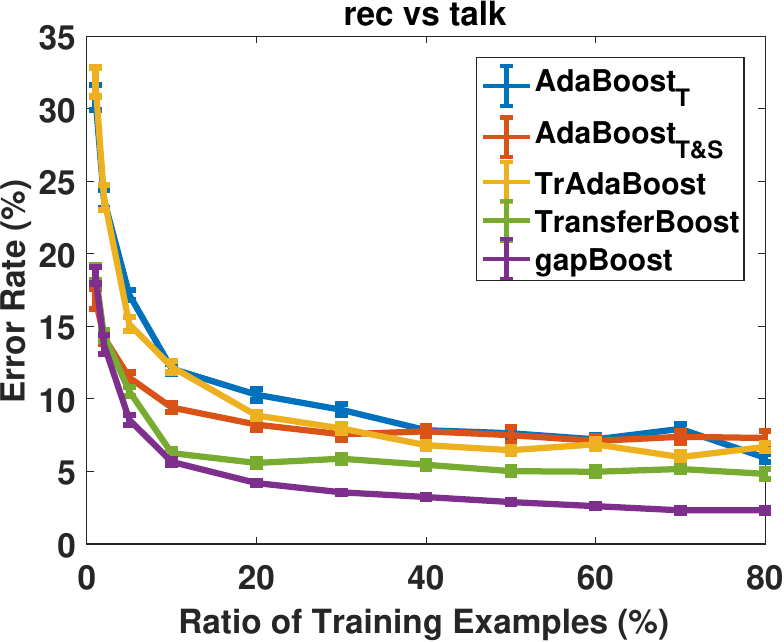}
\hfill~\\[1em]
\hfill\includegraphics[width=.3\textwidth]{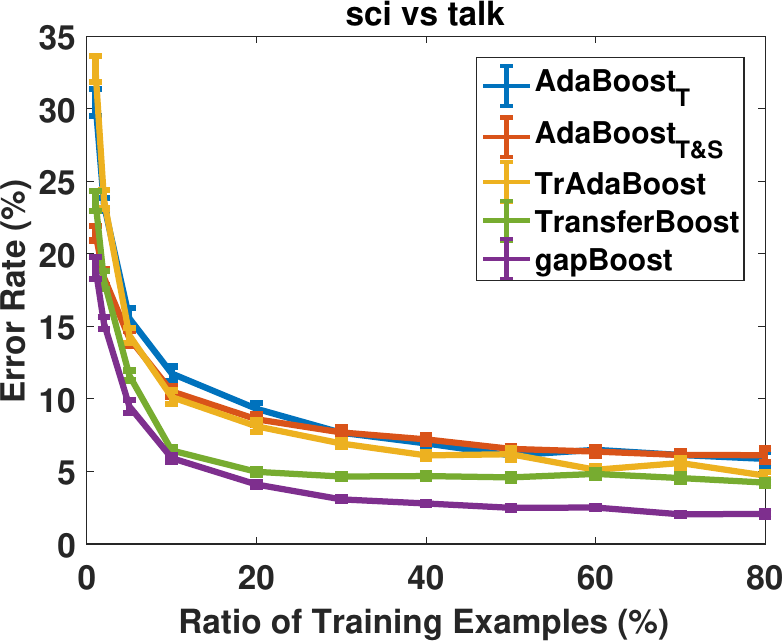}
\hspace{10pt}\includegraphics[width=.3\textwidth]{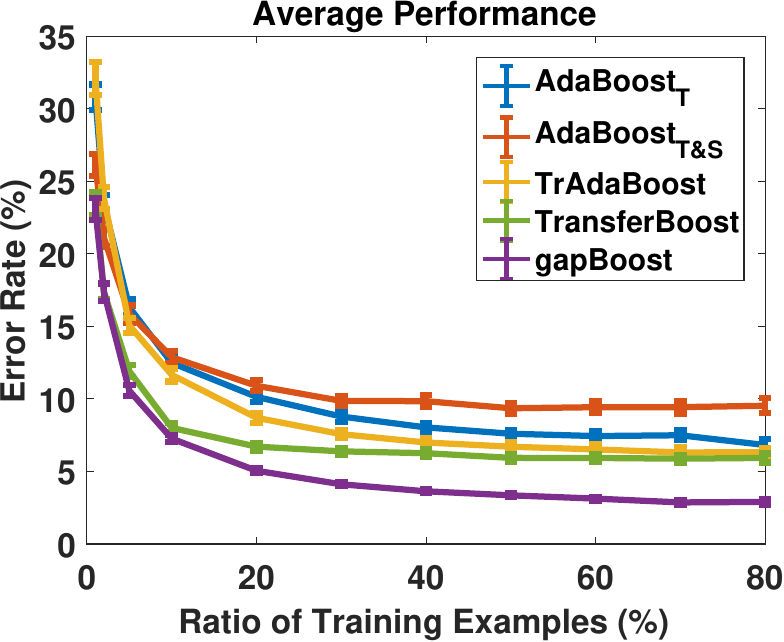}
\hfill~\\
\caption{Test error rates (\%) with different sizes of target sample on different tasks and on average across all tasks. $\afunc{gapBoost}$ consistently outperforms the baselines on all regimes of target sample size. Since $\afunc{gapBoost}$ more effectively leverages the target instances, its improvement over the baselines is more noticeable as the target sample size increases. Error bars represent standard error.}
\label{fig:ratio}
\end{figure*}

\subsubsection{Parameter sensitivity}
Next, we empirically evaluated our algorithms' sensitivity to the choice of hyper-parameters. We first fixed $\rho_\mathcal{T}=0$ and varied $\exp(\rho_\mathcal{S})$ in the range of $[0.1,\dots, 0.9]$. Figure~\ref{fig:rho} shows the results averaged over all transfer problems on the 20 Newsgroups data set, showing that as the size of the target sample increases, the influence of the hyper-parameter on performance decreases. In particular, we see that we are able to obtain a range of hyper-parameters for which our method outperforms all baselines in all sample size regimes.

\subsubsection{Increase the weight of a target instance when auxiliary learners coincide} 
To further minimize the gap, we can modify the weight update rule for target data: $\beta^\mathcal{T} = \rho_\mathcal{T} \mathbbm{1}_{ h_k^\mathcal{S}(x^\mathcal{T}) = h_k^\mathcal{T}(x^\mathcal{T})} + \alpha_k \mathbbm{1}_{ h_k(x^\mathcal{T}) \neq y^\mathcal{T}}$ with $\rho_\mathcal{T} \ge 0$. We vary $\rho_\mathcal{S}$ and $\rho_\mathcal{T}$ together, and the results are shown in Figure~\ref{fig:rhoTrhoS}. It can be observed that $\afunc{gapBoost}$ can achieve even better performance by focusing more on performance gap minimization (i.e., choosing large $\rho_\mathcal{S}$ and $\rho_\mathcal{T}$). As the target data increase, the results are less sensitive to the hyper-parameters.

\begin{figure*}[t]\small
\centering
\includegraphics[width=.4\textwidth]{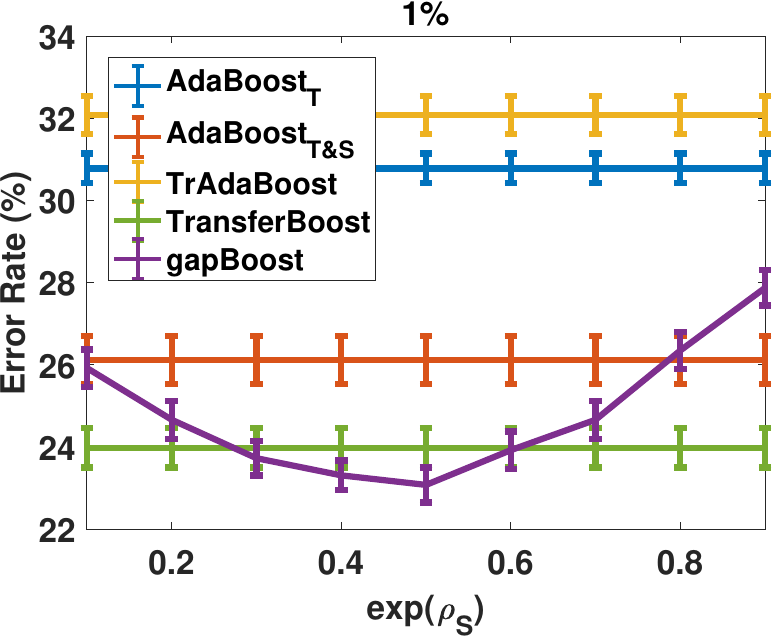}\hspace{10pt}
\includegraphics[width=.4\textwidth]{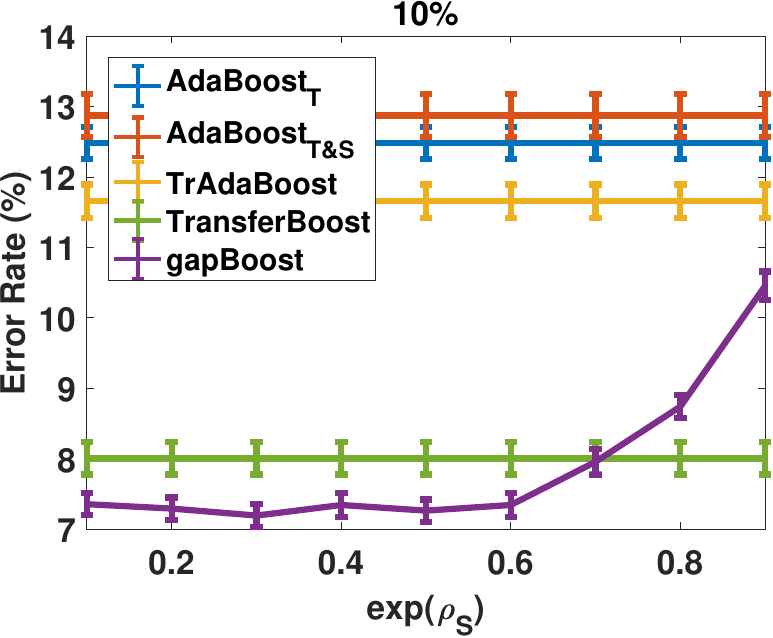}~\\[1em]
\includegraphics[width=.4\textwidth]{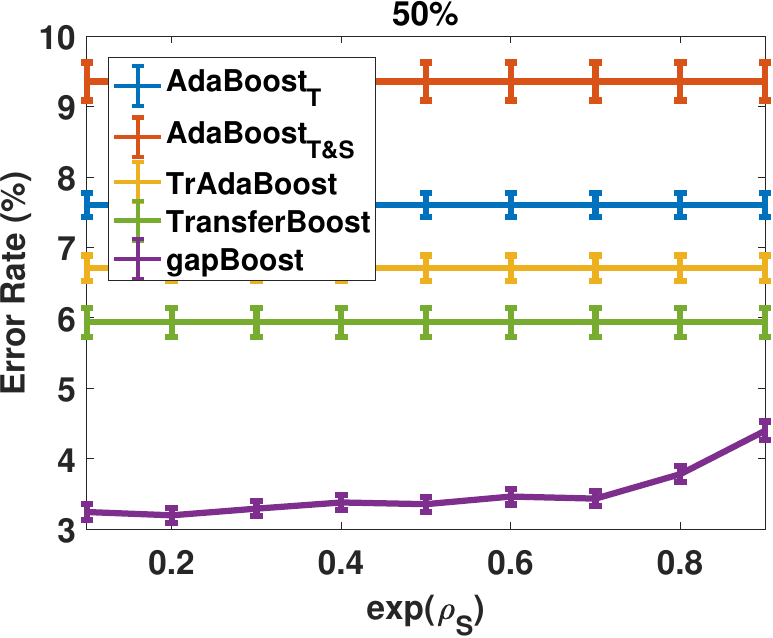}\hspace{10pt}
\includegraphics[width=.4\textwidth]{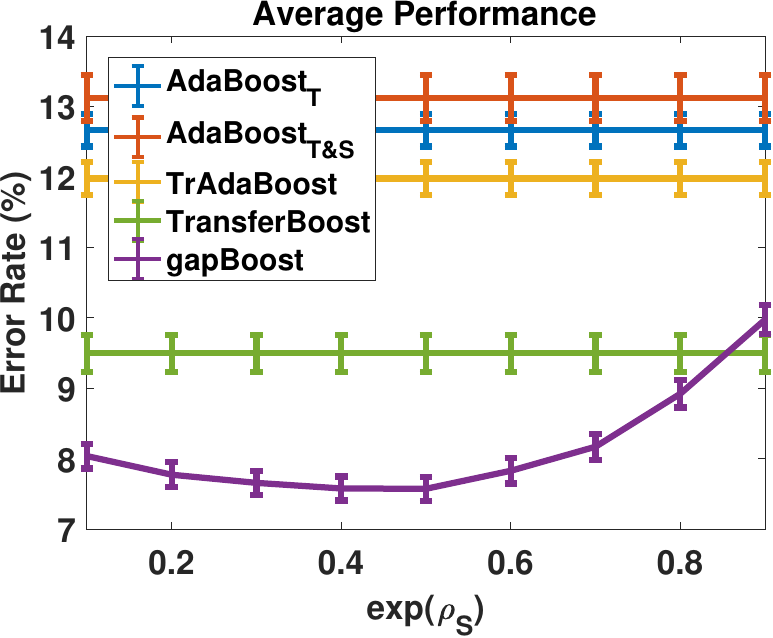}
\caption{Test error rates (\%) averaged across all tasks with respect to the values of the hyper-parameter $\rho_\mathcal{S}$ for varying sample sizes. The rightmost graphic shows results averaged over all sample sizes. $\afunc{gapBoost}$ becomes less sensitive to the choice of $\rho_\mathcal{S}$ as the target sample grows larger. In all cases, there is a range of $\rho_\mathcal{S}$ that outperforms all baselines. Error bars represent standard error.}
\label{fig:rho}
\end{figure*}

\begin{figure*}[t]\small
\centering
\includegraphics[width=.44\textwidth]{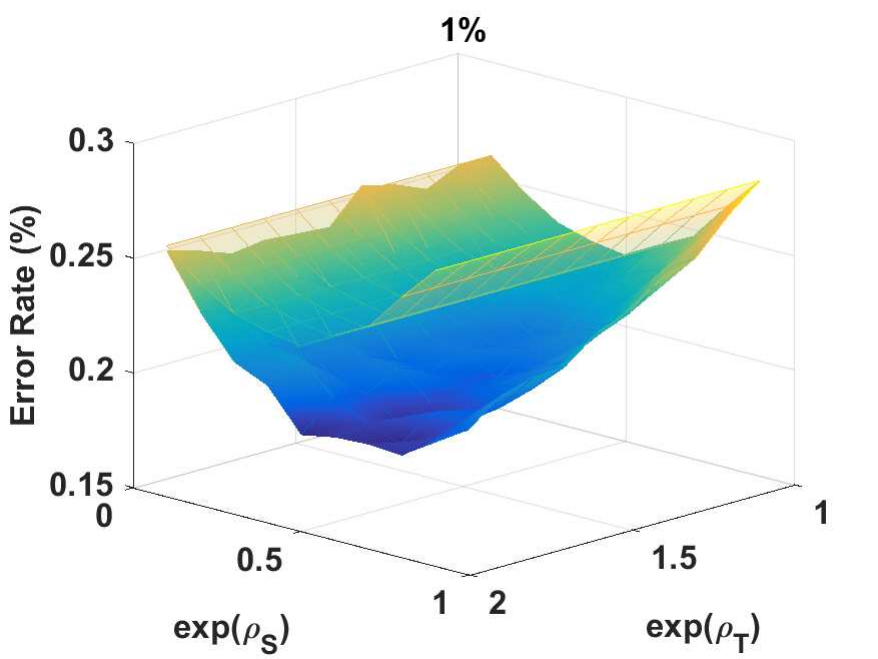}\hspace{10pt}
\includegraphics[width=.44\textwidth]{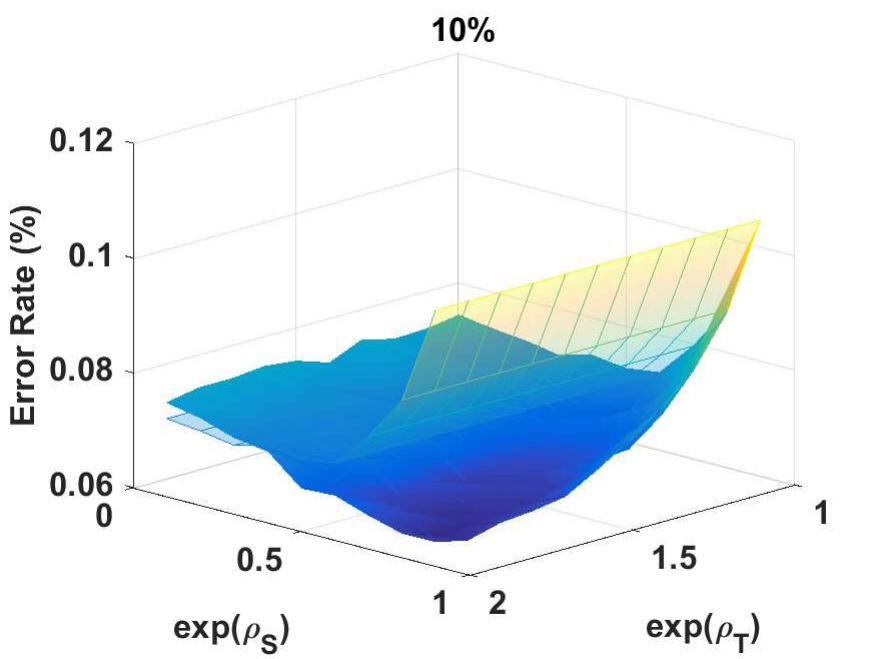}\\[1em]
\includegraphics[width=.44\textwidth]{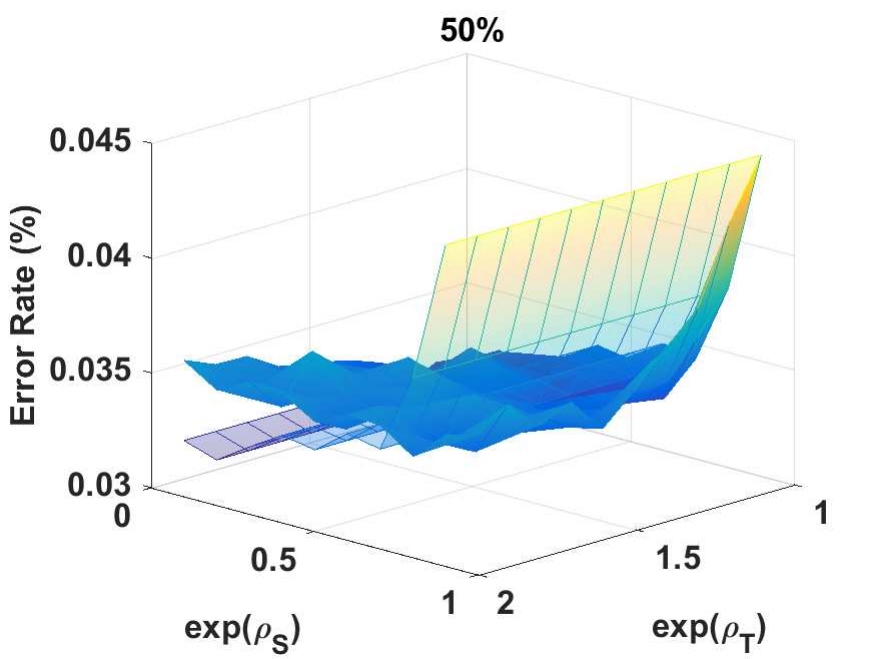}\hspace{10pt}
\includegraphics[width=.44\textwidth]{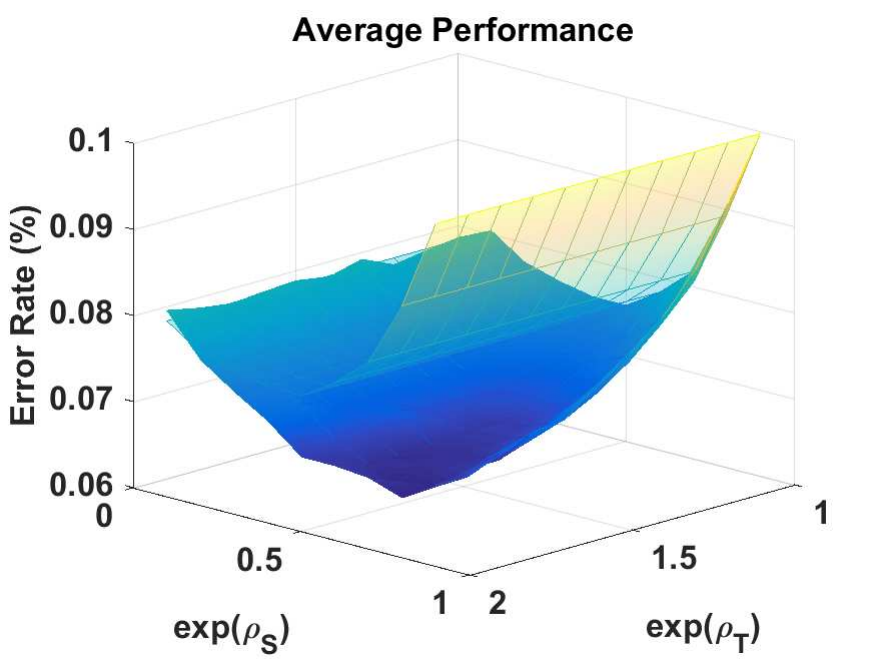}
\caption{Test error rates ($\%$) with varying $\rho_\mathcal{S}$ and $\rho_\mathcal{T}$. The valley curves correspond to $\rho_\mathcal{T} = 0$ (i.e., the purple curves in Figure~\ref{fig:rho}). Hence, regions below the curve indicate better hyper-parameters. } 
\label{fig:rhoTrhoS}
\end{figure*}

\subsection{Multitask Learning}\label{sec:multiexp}

Next, we examined $\afunc{gapMTNN}$  on four benchmark data sets: \textbf{Digits}~\citep{shui2019principled}, \textbf{PACS}~\citep{li2017deeper}, \textbf{Office-31}~\citep{saenko2010adapting}, and \textbf{Office-Home}~\citep{venkateswara2017Deep}. 
The Digits data set consists of tasks: MNIST, MNIST\_M (M-M), and SVHN, with 10 digit classes in each task. The PACS data set consists of images from four tasks: \emph{Photo} (P), \emph{Art painting} (A), \emph{Cartoon} (C), Sketch (S), with 7 different categories in each task. The Office-31 data set is a vision benchmark consisting of three different tasks: \emph{Amazon}, \emph{Dslr} and \emph{Webcam}, with 31 classes in each task. 
 Office-Home is a more challenging benchmark with four different tasks: \emph{Art} (Ar.), \emph{Clipart} (Cl.), \emph{Product} (Pr.) and \emph{Real World} (Rw.), each of which has 65 categories.

\begin{table*}[t]
\caption{Comparison of different methods on the Digits data set in terms of error rate (\%).  $\afunc{gapMTNN}$ outperforms all baselines in terms of average performance.}
\label{tab:results0}
\vspace{-6pt}
\begin{center}
\begin{footnotesize}
\begin{tabular}{p{7pt}p{36pt}p{41pt}p{41pt}p{41pt}p{41pt}p{41pt}p{41pt}p{41pt}}
\toprule
                      & & $\,\, \text{Uniform}$    &   $\text{Weighted}$    & \,\, Adv.H     &  \,\, Adv.W &  \mbox{Multi-Obj}  &  AMTNN  &   $\afunc{gapMTNN}$\\
\midrule
\multirow{5}{*}{$3K$} 
& MNIST    & $6.07_{\pm3.24}$            & $10.69_{\pm3.31}$            & $9.90_{\pm1.21}$
            & $3.21_{\pm0.62}$            & $\bf 2.48_{\pm0.27}$   & $3.13_{\pm0.22}$
            & $3.23_{\pm0.28}$  \\
& SVHN      & $42.73_{\pm0.42}$            & $29.78_{\pm1.78}$        & $29.18_{\pm0.51}$   
            & $30.50_{\pm1.06}$            & $45.22_{\pm0.32}$            & $22.86_{\pm0.92}$
            & $\bf {19.62}_{\pm0.42}$ \\
& MNIST\_M    & $22.88_{\pm2.58}$            & $23.62_{\pm3.05}$            & $18.80_{\pm1.33}$   
            & $\bf 18.67_{\pm0.71}$            & $23.09_{\pm0.53}$            & $19.18_{\pm1.54}$          
            & $18.98_{\pm0.51}$   \\
& avg.      & \hspace{9pt} $23.89$     & \hspace{9pt} $21.36$     & \hspace{9pt} $19.29$  
            & \hspace{9pt} $17.46$     & \hspace{9pt} $23.60$     & \hspace{9pt} $15.07$  
            & \hspace{9pt} $\bf 13.94$ \\
\hline
\multirow{5}{*}{$5K$} 
& MNIST    & $3.75_{\pm1.19}$            & $8.20_{\pm2.71}$            & $8.13_{\pm2.57}$   
            & $2.52_{\pm0.15}$            & $\bf 1.78_{\pm0.21}$            & $2.28_{\pm0.14}$          
            & $2.54_{\pm0.40}$            \\
& SVHN      & $31.97_{\pm2.93}$            & $26.40_{\pm3.11}$             & $26.40_{\pm1.61}$   
            & $27.37_{\pm1.22}$            & $38.84_{\pm0.80}$            & $21.60_{\pm0.78}$ 
            & $\bf {17.47}_{\pm0.47}$       \\
& MNIST\_M    & $20.86_{\pm3.12}$            & $25.76_{\pm0.88}$            & $\bf 16.27_{\pm1.37}$   
            & $16.58_{\pm0.41}$            & $19.78_{\pm0.74}$            & $16.43_{\pm1.13}$          
            & $17.54_{\pm1.91}$     \\
& avg.      & \hspace{9pt} $18.86$     & \hspace{9pt} $20.12$     & \hspace{9pt} $16.93$ 
            & \hspace{9pt} $15.49$     & \hspace{9pt} $20.13$     & \hspace{9pt} $13.44$ 
            & \hspace{9pt} $\bf 12.52$        \\
\hline
\multirow{5}{*}{$8K$} 
& MNIST    & $2.26_{\pm0.50}$            & $7.66_{\pm2.59}$            & $5.07_{\pm1.58}$   
            & $1.89_{\pm0.32}$            & $\bf 1.50_{\pm0.31}$            & $1.90_{\pm0.16}$          
            & $3.04_{\pm0.05}$           \\
& SVHN      & $28.63_{\pm0.93}$            & $25.88_{\pm1.55}$            & $20.86_{\pm0.31}$   
            & $24.56_{\pm1.06}$            & $30.06_{\pm0.88}$            & $19.79_{\pm1.34}$          
            & $\bf 15.90_{\pm0.28}$           \\
& MNIST\_M   & $16.32_{\pm2.19}$            & $23.12_{\pm3.07}$            & $\bf 14.84_{\pm0.34}$   
            & $15.70_{\pm0.43}$            & $17.21_{\pm0.49}$            & $16.92_{\pm2.12}$          
            & $15.46_{\pm0.35}$           \\
& avg.      & \hspace{9pt} $15.74$     & \hspace{9pt} $18.89$     & \hspace{9pt} $13.59$ 
            & \hspace{9pt} $14.05$     & \hspace{9pt} $16.27$     & \hspace{9pt} $12.87$ 
            & \hspace{9pt} $\bf 11.47$   \\
\bottomrule
\end{tabular}
\end{footnotesize}
\end{center}
\vskip -0.1in
\end{table*}

\begin{table*}[t]
\caption{Comparison of different methods on the PACS data set in terms of error rate (\%).  $\afunc{gapMTNN}$ outperforms all baselines in most tasks. }
\label{tab:results1}
\vspace{-6pt}
\begin{center}
\begin{footnotesize}
\begin{tabular}{p{11pt}p{27pt}p{41pt}p{41pt}p{41pt}p{41pt}p{41pt}p{41pt}p{41pt}}
\toprule
                      & & $\,\, \text{Uniform}$    &   $\text{Weighted}$    & \,\, Adv.H     &  \,\, Adv.W &  \mbox{Multi-Obj}  &  AMTNN  &   $\afunc{gapMTNN}$\\
\midrule
\multirow{5}{*}{$10\%$} 
& A     & $21.24_{\pm1.14}$  & $17.32_{\pm1.08}$    & $23.23_{\pm1.64}$    & $21.22_{\pm1.09}$  & $20.58_{\pm1.81}$   & $\mathbf{17.24_{\pm0.38}}$   & $19.37_{\pm 0.42}$  \\
& C     & $18.17_{\pm1.42}$  & $13.78_{\pm0.57}$    & $15.68_{\pm0.42}$    & $16.11_{\pm1.22}$  & $16.62_{\pm1.31}$   & $13.33_{\pm 0.38}$           & $\mathbf{12.07_{\pm0.41}}$ \\
& P     & $12.91_{\pm0.37}$  & $10.32_{\pm1.12}$    & $11.73_{\pm0.62}$    & $12.43_{\pm1.38}$  & $12.97_{\pm0.48}$   & $8.74_{\pm0.91}$             & $\mathbf{5.57_{\pm 0.51}}$        \\
& S     & $16.62_{\pm0.62}$  & $15.10_{\pm0.71}$    & $16.00_{\pm0.61}$    & $15.98_{\pm0.88}$  & $17.08_{\pm1.05}$   & $18.78_{\pm0.81}$            & $\mathbf{8.11_{\pm 0.52}}$          \\
& avg.  & \hspace{9pt} $17.24$  & \hspace{9pt} $14.13$  & \hspace{9pt} $16.66$ & \hspace{9pt} $16.44$ & \hspace{9pt} $16.81$  & \hspace{9pt} $14.52$ & \hspace{9pt} $\bf 11.28$  \\
\hline
\multirow{5}{*}{$15\%$} 
& A     & $17.21_{\pm0.82}$            & $\mathbf{14.90_{\pm0.89}}$   & $17.36_{\pm0.83}$
        & $16.42_{\pm1.42}$            & $17.25_{\pm0.52}$            & $14.90_{\pm0.32}$
        & $16.55_{\pm0.36}$             \\
& C     & $13.33_{\pm0.87}$            & $12.14_{\pm0.72}$            & $12.20_{\pm0.67}$                        
        & $15.21_{\pm0.52}$            & $12.45_{\pm0.44}$            & $11.23_{\pm0.33}$
        & $\mathbf{10.18_{\pm0.56}}$          \\
& P     & $10.73_{\pm0.72}$            & $8.82_{\pm1.04}$             & $10.08_{\pm0.67}$                        
        & $15.47_{\pm0.72}$            & $10.91_{\pm0.64}$            & $7.13_{\pm0.44}$
        & $\mathbf{5.45_{\pm0.43}}$       \\
& S     & $15.09_{\pm0.67}$            & $12.73_{\pm0.72}$             & $13.91_{\pm0.43}$                       
        & $16.03_{\pm0.56}$             & $13.50_{\pm0.52}$            & $14.22_{\pm0.51}$                        
        & $\mathbf{6.87_{\pm 0.42}}$        \\
& avg.  & \hspace{9pt} $14.09$     & \hspace{9pt}  $12.15$    & \hspace{9pt}  $13.39$    
        & \hspace{9pt} $15.78$     & \hspace{9pt}  $13.53$    & \hspace{9pt}  $11.87$& 
        \hspace{9pt}     $\bf 9.76$   \\
\hline
\multirow{5}{*}{$20\%$} 
& A     & $15.91_{\pm0.89}$            & $13.88_{\pm0.56}$            & $15.64_{\pm0.70}$
        & $16.45_{\pm0.47}$            & $15.73_{\pm0.88}$            & $\mathbf{12.64_{\pm0.22}}$
        & $13.67_{\pm0.30}$           \\
& C     & $12.11_{\pm0.82}$            & $10.27_{\pm0.42}$            & $12.36_{\pm0.33}$
        & $10.45_{\pm0.73}$            & $11.27_{\pm0.40}$            & $10.09_{\pm0.56}$
        & $\mathbf{8.40_{\pm0.43}}$          \\
& P     & $9.52_{\pm0.61}$             & $7.91_{\pm1.23}$             & $8.54_{\pm0.42}$
        & $8.62_{\pm0.45}$             & $9.02_{\pm0.33}$             & $6.27_{\pm0.39}$         
        & $\mathbf{4.47_{\pm0.44}}$         \\
& S     & $14.21_{\pm1.21}$            & $11.64_{\pm0.89}$            & $11.68_{\pm0.41}$                        
        & $12.73_{\pm0.56}$            & $11.2_{\pm0.73}$             & $12.27_{\pm0.12}$                        
        & $\mathbf{6.16_{\pm 0.44}}$  \\
& avg.  & \hspace{9pt} $12.94$     & \hspace{9pt} $10.93$     & \hspace{9pt} $12.06$ 
        & \hspace{9pt} $12.06$     & \hspace{9pt} $11.81$     & \hspace{9pt} $10.32$ 
        & \hspace{9pt}   $\bf 8.18$       \\
\bottomrule
\end{tabular}
\end{footnotesize}
\end{center}
\vskip -0.1in
\end{table*}

\begin{table*}[t]
\caption{Comparison of different methods on the Office-31 data set in term of error rate (\%).  $\afunc{gapMTNN}$ outperforms all baselines in the majority of tasks. }
\label{tab:results2}
\vspace{-6pt}
\begin{center}
\begin{footnotesize}
\begin{tabular}{p{11pt}p{27pt}p{41pt}p{41pt}p{41pt}p{41pt}p{41pt}p{41pt}p{41pt}}
\toprule
                      & & $\,\, \text{Uniform}$    &   $\text{Weighted}$    & \,\, Adv.H     &  \,\, Adv.W &  \mbox{Multi-Obj}  &  AMTNN  &   $\afunc{gapMTNN}$\\
\midrule
\multirow{5}{*}{$5\%$} 
& Amazon    & $38.73_{\pm1.32}$            & $36.69_{\pm0.23}$            & $34.21_{\pm1.12}$
            & $33.50_{\pm1.89}$            & $\mathbf{31.13_{\pm1.18}}$   & $36.67_{\pm0.60}$
            & $31.77_{\pm1.67}$  \\
& Dslr      & $28.18_{\pm2.12}$            & $\bf 12.57_{\pm2.33}$        & $26.45_{\pm0.78}$   
            & $28.18_{\pm1.09}$            & $27.51_{\pm1.44}$            & $19.88_{\pm1.56}$
            & ${13.25}_{\pm1.40}$ \\
& Webcam    & $27.88_{\pm1.12}$            & $15.09_{\pm0.56}$            & $28.64_{\pm0.72}$   
            & $30.08_{\pm0.92}$            & $27.73_{\pm0.41}$            & $14.64_{\pm0.30}$          
            & $\mathbf{14.18_{\pm0.89}}$   \\
& avg.      & \hspace{9pt} $31.60$     & \hspace{9pt} $21.45$     & \hspace{9pt} $29.77$  
            & \hspace{9pt} $30.59$     & \hspace{9pt} $28.79$     & \hspace{9pt} $23.73$  
            & \hspace{9pt} $\bf 19.73$ \\
\hline
\multirow{5}{*}{$10\%$} 
& Amazon    & $26.82_{\pm0.51}$            & $29.36_{\pm1.21}$            & $28.98_{\pm0.89}$   
            & $25.31_{\pm1.13}$            & $25.36_{\pm0.89}$            & $28.73_{\pm1.20}$          
            & $\mathbf{23.56_{\pm0.63}}$            \\
& Dslr      & $19.36_{\pm1.44}$            & $7.91_{\pm0.91}$             & $15.91_{\pm0.91}$   
            & $14.09_{\pm0.80}$            & $13.18_{\pm1.12}$            & $\mathbf{7.21_{\pm0.89}}$ 
            & ${7.87}_{\pm1.81}$       \\
& Webcam    & $17.89_{\pm0.91}$            & $11.60_{\pm1.33}$            & $10.56_{\pm0.14}$   
            & $14.31_{\pm0.78}$            & $13.08_{\pm0.78}$            & $10.41_{\pm1.21}$          
            & $\mathbf{8.16_{\pm1.54}}$     \\
& avg.      & \hspace{9pt} $21.36$     & \hspace{9pt} $16.29$     & \hspace{9pt} $18.48$ 
            & \hspace{9pt} $17.90$     & \hspace{9pt} $17.21$     & \hspace{9pt} $15.45$ 
            & \hspace{9pt} $\bf 13.20$        \\
\hline
\multirow{5}{*}{$20\%$} 
& Amazon    & $20.56_{\pm0.84}$            & $23.21_{\pm0.93}$            & $20.31_{\pm0.53}$   
            & $20.73_{\pm0.56}$            & $20.77_{\pm0.78}$            & $19.80_{\pm0.91}$          
            & $\mathbf{17.38_{\pm0.89}}$           \\
& Dslr      & $8.82_{\pm1.01}$             & $\mathbf{3.40}_{\pm0.71}$    & $6.31_{\pm0.67}$   
            & $6.17_{\pm0.44}$             & $7.91_{\pm0.63}$             & $5.82_{\pm1.22}$          
            & $3.65_{\pm0.81}$    \\
& Webcam    & $6.91_{\pm0.82}$             & $4.40_{\pm0.52}$             & $6.32_{\pm0.64}$   
            & $7.77_{\pm0.90}$             & $5.26_{\pm0.56}$             & $5.58_{\pm0.89}$          
            & $\bf {4.08_{\pm0.72}}$    \\
& avg.      & \hspace{9pt} $12.10$     & \hspace{9pt} $10.34$     & \hspace{9pt} $10.98$ 
            & \hspace{9pt} $11.56$     & \hspace{9pt} $11.31$     & \hspace{9pt} $10.40$ 
            & \hspace{9pt} $\bf 8.37$   \\
\bottomrule
\end{tabular}
\end{footnotesize}
\end{center}
\vskip -0.1in
\end{table*}

\begin{table*}[t]
\caption{Comparison of different methods on the Office-Home data set in terms of error rate (\%).  $\afunc{gapMTNN}$ outperforms all baselines in most tasks. }
\label{tab:results3}
\vspace{-6pt}
\begin{center}
\begin{footnotesize}
\begin{tabular}{p{11pt}p{27pt}p{41pt}p{41pt}p{41pt}p{41pt}p{41pt}p{41pt}p{41pt}}
\toprule
                      & & $\,\, \text{Uniform}$    &   $\text{Weighted}$    & \,\, Adv.H     &  \,\, Adv.W &  \mbox{Multi-Obj}  &  AMTNN  &   $\afunc{gapMTNN}$\\
\midrule
\multirow{5}{*}{5\%} 
& Ar.       & $73.82_{\pm0.31}$            & $73.18_{\pm1.56}$            & $72.27_{\pm1.44}$ 
            & $73.21_{\pm0.82}$            & $74.36_{\pm1.50}$            & $67.50_{\pm1.33}$
            & $\bf 61.91_{\pm0.35}$  \\
& Cl.       & $69.91_{\pm0.22}$            & $68.78_{\pm1.83}$            & $67.94_{\pm1.45}$ 
            & $67.31_{\pm0.45}$            & $68.29_{\pm1.71}$            & $65.45_{\pm0.90}$          
            & $\bf 60.01_{\pm 1.02}$ \\
& Pr.       & $42.36_{\pm0.13}$            & $40.81_{\pm0.43}$            & $40.44_{\pm0.73}$ 
            & $41.70_{\pm0.89}$            & $41.32_{\pm1.32}$            & $43.72_{\pm0.78}$    
            & $\bf 37.85 _{\pm 0.82}$  \\
& Rw.       & $52.60_{\pm1.11}$            & $49.45_{\pm1.22}$            & $48.91_{\pm0.93}$ 
            & $52.92_{\pm0.37}$            & $48.47_{\pm0.89}$            & $50.14_{\pm1.75}$
            & $\bf 43.96_{\pm0.27}$  \\
& avg.      & \hspace{9pt} $59.67$     & \hspace{9pt}   $58.06$   & \hspace{9pt}  $57.39$ 
            & \hspace{9pt} $58.79$     & \hspace{9pt}  $58.11$    & \hspace{9pt}  $56.70$ 
            & \hspace{9pt} $\bf 50.93$  \\
\hline
\multirow{5}{*}{10\%} 
& Ar.       & $64.18_{\pm0.67}$            & $61.81_{\pm1.03}$            & $61.03_{\pm0.89}$                
            & $61.50_{\pm0.82}$            & $65.38_{\pm0.92}$            & $58.89_{\pm1.02}$ 
            & $\bf 55.67_{\pm0.70}$           \\
& Cl.       & $56.73_{\pm0.63}$            & $54.68_{\pm1.61}$            & $54.21_{\pm1.78}$ 
            & $55.64_{\pm0.67}$            & $56.66_{\pm1.43}$            & $52.45_{\pm0.83}$
            & $\bf 48.50_{\pm0.87}$        \\
& Pr.       & $32.86_{\pm0.42}$            & $30.88_{\pm0.23}$            & $30.61_{\pm0.41}$ 
            & $32.44_{\pm0.67}$            & $33.85_{\pm1.50}$            & $31.61_{\pm0.72}$
            & $\bf 28.76_{\pm0.45}$      \\
& Rw.       & $43.18_{\pm1.31}$            & $41.67_{\pm0.78}$            & $41.20_{\pm0.56}$ 
            & $40.46_{\pm0.89}$            & $43.19_{\pm0.73}$            & $41.10_{\pm0.90}$          
            & $\bf 37.82_{\pm0.27}$        \\
& avg.      & \hspace{9pt} $49.24$     & \hspace{9pt} $47.26$     & \hspace{9pt} $46.76$  
            & \hspace{9pt} $47.51$     & \hspace{9pt} $49.77$     & \hspace{9pt} $46.01$ 
            & \hspace{9pt} $\bf 42.69$   \\
\hline
\multirow{5}{*}{20\%} 
& Ar.       & $54.54_{\pm0.82}$            & $52.07_{\pm0.12}$            & $53.34_{\pm0.51}$ 
            & $52.12_{\pm0.49}$            & $53.80_{\pm0.82}$            & $\bf 51.11_{\pm0.50}$          
            & $51.20_{\pm1.03}$           \\
& Cl.       & $43.88_{\pm0.61}$            & $43.29_{\pm0.89}$            & $43.45_{\pm1.13}$ 
            & $43.33_{\pm0.56}$            & $43.44_{\pm0.50}$            & ${39.34}_{\pm0.41}$ 
            & $\bf 39.15_{\pm0.93}$            \\
& Pr.       & $25.62_{\pm0.67}$            & $24.36_{\pm0.56}$            & $24.38_{\pm0.43}$ 
            & $24.59_{\pm1.12}$            & $25.74_{\pm0.67}$            & $24.64_{\pm0.40}$          
            & $\bf 23.45_{\pm0.78}$         \\
& Rw.       & $37.41_{\pm0.62}$            & $35.20_{\pm0.89}$            & $34.88_{\pm0.74}$ 
            & $34.30_{\pm0.78}$            & $37.24_{\pm0.56}$            & $35.31_{\pm0.44}$          
            & $\bf 33.38_{\pm0.73}$          \\
& avg.      & \hspace{9pt} $40.36$     & \hspace{9pt} $38.73$     & \hspace{9pt} $39.01$ 
            & \hspace{9pt} $38.59$     & \hspace{9pt} $40.06$     & \hspace{9pt} $37.60$ 
            & \hspace{9pt} $\bf 36.80$    \\
\bottomrule
\end{tabular}
\end{footnotesize}
\end{center}
\vskip -0.1in
\end{table*}

Following the evaluation protocol in prior work~\citep{long2015learning,shui2019principled,zhou2021multi}, we evaluated $\afunc{gapMTNN}$ against the baselines when only part of the data is available. For the Digits data set, we randomly select $3K$, $5K$ and $8K$ instances for training. For the SVHN data set, we resize the images to $28\times 28 \times 1$; we do not apply any data augmentation on the Digits data set. A LeNet-5~\citep{lecun1998gradient} model is implemented as the feature extractor and three 3-layer MLPs are deployed as task-specific classifiers, with the features from the classifier being of size $128$. For the experiments on PACS data set, we randomly selected $10\%$, $15\%$ and $20\%$ of the total training data for training. We adopted the pre-trained AlexNet model (\emph{PyTorch} implementation;~\citep{paszke2019pytorch}) as a feature extractor, removing the last fully-connected layers, yielding a feature dimension of size 4096. For the  Office-31 and Office-Home data sets, we adopted the ResNet-18 model as a feature extractor by also removing the last fully-connected layers of the original implementations, yielding a feature dimension of size 512. We followed the evaluation protocol of~\citep{zhou2021multi} by choosing $5\%$, $10\%$ and $20\%$ of the data for training.
We trained the networks using the Adam optimizer, with an initial learning rate of $2e-4$, decaying by $5\%$ every $5$ epochs, for a total of $120$ epochs. Additional details on the experimental implementation can be found in Appendix~\ref{implementation}. 

\subsubsection{Performance Comparison}

We adopted the following algorithms as baselines:
\begin{itemize} 
    \item \textit{Uniform}: Treating all the tasks equally without any task alignments for training the deep networks.
    \item \textit{Weighted}: Adapted from~\citep{murugesan2016adaptive}, apply a weighted risk over all the tasks, where the weights are determined based on a probabilistic interpretation.
    \item \textit{Adv.H}: Implementing the method of \citep{liu2017adversarial} with the same loss function while training with $\mathcal{H}$-divergence as the adversarial objective.
    \item \textit{Adv.W}: Implementing the method of \citep{liu2017adversarial} with the same loss function while training with the Wasserstein distance based adversarial training method.
    \item \textit{Multi-Obj}: Treating the multitask learning problem as a multi-objective problem~\citep{sener2018multi}.
    \item \textit{AMTNN}: A Wasserstein adversarial training method for estimating the task relations~\citep{shui2019principled}.
\end{itemize}

The results on Digits, PACS, Office-31, and Office-Home data sets are, respectively, reported in Tables~\ref{tab:results0}~--~\ref{tab:results3}. It can be observed that  $\afunc{gapMTNN}$ outperforms the baselines in most cases. In particular, compared with AMTNN, which only aligns marginal distributions, $\afunc{gapMTNN}$ has a large margin of improvement especially when there are few labeled instances (e.g., 5\% of the total instances), which confirms the effectiveness of our methods when dealing with limited data.

\subsubsection{Visualizations of Task Relation Coefficients}
In order to further verify the task weighting principles of $\afunc{gapMTNN}$, we visualize the task relation coefficients $\Gamma$ learned from the Digits data set, as shown in Figure~\ref{fig:taskcoef}. From the figure, we observe the following: (1) the task coefficients are non-uniform and asymmetric, which highlights the importance of properly choosing the coefficients for combining the tasks rather than simply treating them equally. (2) The coefficients between MNIST and MNIST\_M are higher than the coefficients between SVHN and MNIST or MNIST\_M, which is reasonable as SVHN is less related to the other tasks~\citep{shui2019principled}. (3) MNIST is assigned a higher weight to itself than the other tasks (see the diagonal entries of the matrices). This can be due to the fact that MNIST is a relatively easy task, and therefore can be learned well without leveraging the knowledge from other tasks too much. (4) From left to right, the values of diagonal entries of the matrices indicate that each task relies on the other tasks less as  the number of training instances increases. This is also reasonable since once the sample size of each individual task becomes larger, the benefit of knowledge transfer and sharing is less significant for multitask learning.  

\begin{figure*}[t]\small
\centering
\includegraphics[width=.95\textwidth]{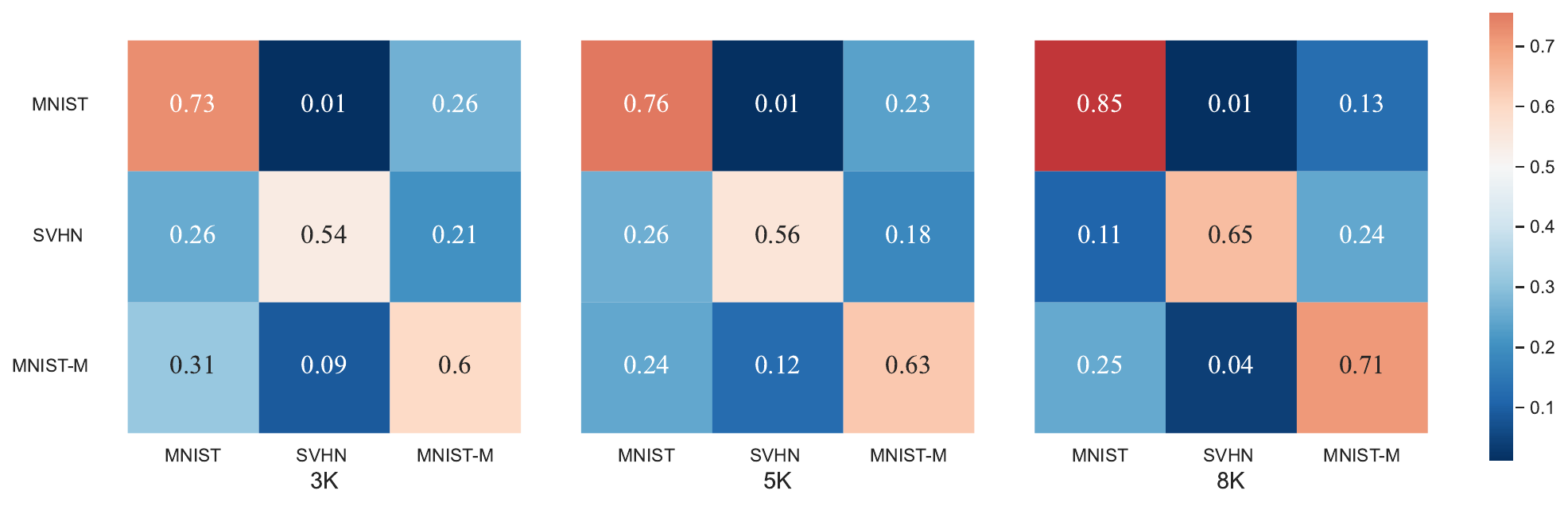}
\caption{Task relation coefficient matrices learned by $\afunc{gapMTNN}$ on the Digits data set with training set of $3K$, $5K$, and $8K$ instances, respectively}
\label{fig:taskcoef}
\end{figure*}

\subsubsection{Running Time Comparison}
To show the efficiency of the $\afunc{gapMTNN}$ method, we compared its training time against the baselines. Specifically, we evaluated multitask learning algorithms on the Digits ($8K$), PACS ($20\%$), Office-31 ($20\%$) and Office-Home ($20\%$) data sets, and report the relative time comparison of one training epoch in a relative percentage bar chart in Fig.~\ref{fig:time}. It can be observed that as $\afunc{gapMTNN}$ adopts the centroid alignment strategy, it achieves comparable efficiency with Multi-Obj, and is more efficient than the adversarial training based methods (e.g., Adv.H, Adv.W and AMTNN), especially on the Office-Home data set. Taking the results reported in Tables~\ref{tab:results0}~--~\ref{tab:results3} into consideration, $\afunc{gapMTNN}$ can improve the performances on the benchmark data sets while reducing the time needed for training. 
\begin{figure}[t]
\centering
\includegraphics[width=.8\textwidth]{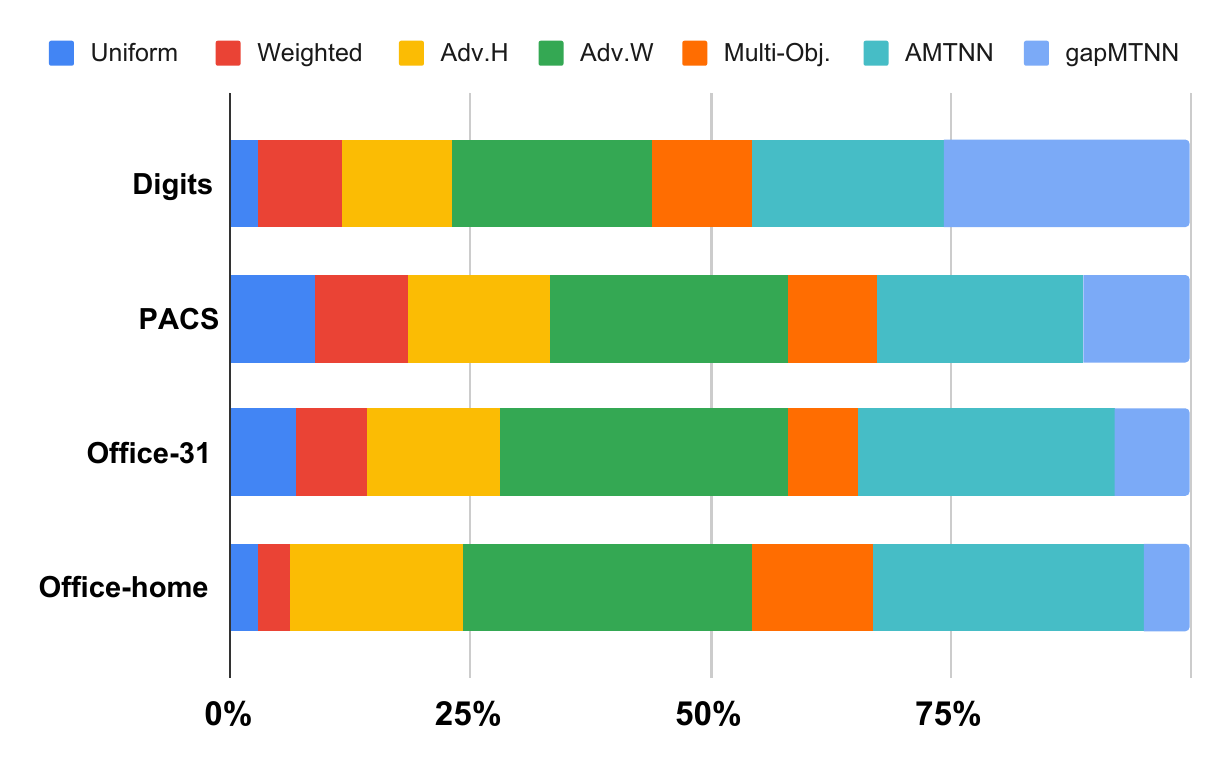}
\caption{Time comparison on different data sets.} 
\label{fig:time}
\end{figure}

\subsubsection{Ablation Studies}

We conducted ablation studies on the Office-31 data set with 20\% of the data to verify each component of $\afunc{gapMTNN}$. We compared the full version of $\afunc{gapMTNN}$ with the following ablated versions: (1) \emph{cls.~only}: train the model uniformly with only the classification objective. (2) \emph{w/o marginal alignment}: we omit the marginal alignment objective and train the model with semantic matching and task relation optimization. \mbox{(3)~\emph{w/o semantic matching}:} we omit the semantic conditional distribution matching objective and train the model with marginal alignment objective with task relation optimization. (4) \emph{w/o cvx opt.}: we omit task relation estimations and train the model  with marginal alignment and semantic matching objectives. (5) \emph{w/o marginal $\&$ semantic}: we remove both adversarial and semantic learning objectives. The results of the ablation studies are presented in Table~\ref{tab:ablation}, showing that distribution matching is crucial for the algorithm. Both marginal and conditional distribution matching improve learning performance.

\begin{table}[]
\centering
\caption{Ablation studies on Office-31 data set.}
\label{tab:ablation}
{\begin{tabular}{@{}l|lll|l@{}}
\toprule
\textbf{Method}         & \textbf{Amazon}               & \textbf{Dslr}          & \textbf{WebCam}         & \textbf{Avg.}        \\ \midrule
Cls. only  & $20.6_{\pm0.8}$           & $8.8_{\pm1.0}$         & $6.9_{\pm0.8}$           & $12.1$ \\
w/o marginal  &   $18.9_{\pm0.2}$          & $3.5_{\pm0.1}$         & $3.9_{\pm0.2}$           & $8.8$  \\
w/o sem. matching & $18.8_{\pm1.1}$ & $3.4_{\pm0.7}$ & $4.8_{\pm0.4}$ & $9.0$    \\
w/o cvx opt. & $18.2_{\pm0.3}$ & $4.7_{\pm0.9}$ & $3.8_{\pm 0.5}$ & $8.9$              \\
\begin{tabular}{@{}l}w/o marginal $\&$\\ sem. matching \end{tabular}  & $20.2_{\pm0.5}$          & $3.9_{\pm0.3}$          & $4.6_{\pm0.3}$          & $9.6$          \\
Full method    & $17.4_{\pm0.9}$          & $3.7_{\pm0.8}$         & $4.1\pm0.7$           & $8.4$                  \\ \bottomrule
\end{tabular}}
\end{table}

\section{Conclusion}\label{conclusion}

In this paper, we propose the notion of performance gap to measure the discrepancy between the tasks with labeled instances. We relate this notion with the model complexity and show that it can be viewed as a data- and algorithm-dependent regularizer, which eventually leads to gap minimization, a general principle that is applicable to both transfer learning and multitask learning. We propose $\afunc{gapBoost}$, $\afunc{gapBoostR}$, and $\afunc{gapMTNN}$ as three algorithmic instantiations that exploit the gap minimization principle. The empirical evaluation justifies the effectiveness of our algorithms.

The principle of performance gap minimization opens up several avenues for knowledge sharing and transfer. For example, it could be used to analyze strategies for other knowledge sharing and transfer scenarios such as domain generalization, multi-label learning, lifelong learning, or even knowledge transfer across different learning paradigms (e.g., between classification and regression). It could also be adopted for fair learning~\citep{shui2022fair,shuilearning}.
On the theoretical side, future directions could include extending the notion of performance gap to unlabeled data for domain adaptation, and to non-stationary environments. We plan to explore these questions in future work.

\section*{Acknowledgments}

We appreciate constructive feedback from the action editor, John Shawe-Taylor, and the anonymous reviewers. B.~Wang and G.~Xu are supported by the Natural Sciences and Engineering Research Council of Canada (NSERC), Discovery Grants program. C.~Shui and C.~Gagn\'e acknowledge the support from NSERC-Canada and the Canada Institute for Advanced Research (CIFAR) Artificial Intelligence Chairs program. J.A.~Mendez and E.~Eaton are partially supported by the DARPA Lifelong Learning Machines program under grant FA8750-18-2-0117, the DARPA SAIL-ON program under contract HR001120C0040, the DARPA ShELL program under agreement HR00112190133, and the Army Research Office under MURI grant W911NF20-1-0080.









\newpage

\appendix
\section{Proof of Theoretical Results}
\label{app:theorem}

\subsection{Instance Weighting}

In order to prove Theorem~\ref{maintheorem1}, we need some auxiliary results.
\begin{lemma}\label{lemma0}
Let $\mathcal{L}^{\Gamma}_\mathcal{D} = \mathcal{L}_{\mathcal{D}_1}^{\Gamma^1} + \mathcal{L}_{\mathcal{D}_2}^{\Gamma^2}$.
Then, for any $h \in \mathcal{H}$, we have
\begin{align*}
         \mathcal{L}_{\mathcal{D}_1}(h) \le  \mathcal{L}_{\mathcal{D}}^\Gamma + \|\Gamma^2\|_1 \dist_\mathcal{Y}(\mathcal{D}_1, \mathcal{D}_2) \enspace.
\end{align*}
\end{lemma}
\begin{proof}
\begin{align*}
         \mathcal{L}_{\mathcal{D}_1}(h)-\mathcal{L}_{\mathcal{D}}^{\Gamma} (h)
         & \le \left|\mathcal{L}_{\mathcal{D}_1} (h)- \mathcal{L}_{\mathcal{D}_1}^{\Gamma^1}(h)- \mathcal{L}_{\mathcal{D}_2}^{\Gamma^2}(h) \right| \\
         & = \left|\mathcal{L}_{\mathcal{D}_1}(h)-  (1-\gamma_2) \mathcal{L}_{\mathcal{D}_1}(h) - \gamma_2 \mathcal{L}_{\mathcal{D}_2}(h) \right| \hspace{-40pt} &\text{linearity of expectation}\\
         & =\|\Gamma^2\|_1 \left| \mathcal{L}_{\mathcal{D}_1}(h) - \mathcal{L}_{\mathcal{D}_2}(h) \right| \\
          & \le\|\Gamma^2\|_1  \dist_\mathcal{Y}(\mathcal{D}_1, \mathcal{D}_2)  & \text{definition of $\mathcal{Y}$-discrepancy distance} 
\end{align*}
\end{proof}

\begin{definition}[{\bf Weight dependent uniform stability}]\label{definition1d}
    Let $h_{{S}}\in \mathcal{H}$ be the hypothesis returned by a learning algorithm $\mathcal{A}$ when trained on sample $S$ weighted by $\Gamma$. An algorithm $\mathcal{A}$ has weight dependent uniform stability, with $\beta_i \ge 0$, if the following holds:
   \begin{align*}
        \sup_{z \sim \mathcal{D}} \left| \ell(h_S(x),y)- \ell(h_{S^i}(x),y)\right| \le \beta_i \enspace, \qquad \forall {S}, {S}^i
    \end{align*}
    where ${S}^i$ is the training sample ${S}$ with the $i$-th example $z_i$ replaced by an i.i.d. example $z_i'$.
\end{definition}

We bound the generalization error for weight dependent stable algorithms.
\begin{lemma}\label{lemmaas}
 Assume that the loss function is upper bounded by $B\ge0$. Let $S$ be a training sample of $N$ i.i.d. points drawn from some distribution $\mathcal{D}$, weighted by $\Gamma$, and let $h_S$ be the hypothesis returned by a weight dependent stable learning algorithm $\mathcal{A}$. Then, for any $\delta \in (0,1)$, with probability at least $1-\delta$, the following holds:
    \begin{align*}
         \mathcal{L}_{\mathcal{D}}^\Gamma (h_S) \le \mathcal{L}_S^\Gamma (h_S) + \beta + (\Delta + \beta + \|\Gamma\|_\infty B)\sqrt{\frac{N\log \frac{1}{\delta}}{2}} \enspace,
     \end{align*}
where $\beta = \max\{\beta_i\}_{i=1}^N$, $\Delta = \sum_{i=1}^N \gamma_i \beta_i $ and $\|\Gamma\|_\infty = \max \{\gamma_i\}_{i=1}^N$.
\end{lemma}

\begin{proof}
Let $\Phi^\Gamma(S) = \mathcal{L}_\mathcal{D}^\Gamma(h_S) - \mathcal{L}_S^\Gamma(h_S)$. Then, by the definition of $\Phi^\Gamma$, we have
\begin{align*}
        |\Phi(S)-\Phi(S^i)| \le |\mathcal{L}_{\mathcal{D}}^\Gamma (h_S)-\mathcal{L}_{\mathcal{D}}^\Gamma (h_{S^i})| + |\mathcal{L}_S^\Gamma(h_S)+ \mathcal{L}_{S^i}^\Gamma(h_{S^i})| \enspace.
\end{align*}

By the stability of the algorithm, we have\footnote{We write $\ell\left(h(x), y\right)$ as $\ell_z(h)$ for simplicity.}
\begin{align*}
        |\mathcal{L}_{\mathcal{D}}^\Gamma (h_S)-\mathcal{L}_{\mathcal{D}}^\Gamma (h_{S^i})| = |\mathbb{E}_{z\sim \mathcal{D}}[\ell_z( h_S)] - \mathbb{E}_{z\sim \mathcal{D}}[\ell_z( h_{S^i} )] | \le \beta \enspace,
\end{align*}
where $\beta = \max \{\beta_i\}_{i=1}^N$.
In addition, we also have
\begin{align*}
         |\mathcal{L}_S^\Gamma(h_S) - \mathcal{L}_{S^i}^\Gamma(h_{S^i}) | &= \left|\sum_{j \neq i} \gamma_j (\ell_{z_j}(h_S) - \ell_{z_j}(h_{S^i})) + \gamma_i(\ell_{z_i}(h_S) - \ell_{z_i'}(h_{S^i})) \right| \\
         & \le \left|\sum_{j \neq i} \gamma_j \left|\ell_{z_j}(h_S) - \ell_{z_j}(h_{S^i})\right| + \gamma_i \left|\ell_{z_i}(h_S) - \ell_{z_i'}(h_{S^i})\right| \right|  \\
         & \le \sum_{j \neq i} \gamma_j \beta_j + \gamma_i B \\
         & \le \Delta + \|\Gamma\|_\infty B
    \end{align*}
Consequently, $\Phi^\Gamma$ satisfies $|\Phi^\Gamma(S)-\Phi^\Gamma(S^i)| \le \sum_{i=1}^N \gamma_i \beta_i + \beta + \|\Gamma\|_\infty B$. By applying McDiarmid's inequality, we have
\begin{align}\label{eqmc}
    \text{Pr} \left[ \Phi (S) \ge \epsilon + \mathbb{E}[\Phi(S)]\right] \le \exp \left( \frac{-2\epsilon^2}{N\left(\Delta + \beta + \|\Gamma\|_\infty B \right)^2} \right) \enspace.
\end{align}
By setting $\delta = \exp \left( \frac{-2\epsilon^2}{N\left(\Delta + \beta + \|\Gamma\|_\infty B \right)^2} \right)$, we obtain $\epsilon = \left(\Delta + \beta + \|\Gamma\|_\infty B \right)\sqrt{\frac{N\log\frac{1}{\delta}}{2}}$. Plugging $\epsilon$ back to (\ref{eqmc}) and rearranging terms, with probability $1-\delta$, we have
\begin{align}\label{eqas1}
        \Phi(S) \le \mathbb{E}[\Phi(S)] + \left(\Delta + \beta + \|\Gamma\|_\infty B \right)\sqrt{\frac{N\log\frac{1}{\delta}}{2}} \enspace.
\end{align}
By the linearity of expectation, we have $\mathbb{E}[\Phi(S)] = \mathbb{E}_{S\sim \mathcal{D}^N}[\mathcal{L}^\Gamma_\mathcal{D}(h_S)] - \mathbb{E}_{S\sim \mathcal{D}^N}[\mathcal{L}_S^\Gamma(h_S)]$. By the definition of the generalization error, we have
\begin{align*}
        \mathbb{E}_{S\sim \mathcal{D}^N}[\mathcal{L}^\Gamma_\mathcal{D}(h_S)] = \mathbb{E}_{S,z\sim \mathcal{D}^{N+1}}[\ell_z(h_S)] \enspace.
\end{align*}
On the other hand, by the linearity of expectation, we have
\begin{align*}
        \mathbb{E}_{S\sim \mathcal{D}^N}[\mathcal{L}_S^\Gamma(h_S)] =  \mathbb{E}_{S\sim \mathcal{D}^N} \left[ \sum_{i=1}^N \gamma_i \ell_{z_i} (h_S) \right] = \mathbb{E}_{S,z\sim \mathcal{D}^{N+1}} \left[ \ell_{z} (h_{S'}) \right] \enspace,
\end{align*}
where $S'$ is a sample of $N$ data points containing $z$ drawn from the data set $\{S,z\}$. Therefore, we have
\begin{align*}
        \mathbb{E}[\Phi(S)]  & \le \left| \mathbb{E}_{S,z\sim \mathcal{D}^{N+1}}[\ell_z(h_S)] -  \mathbb{E}_{S,z\sim \mathcal{D}^{N+1}} \left[ \ell_{z} (h_{S'}) \right] \right| \\
        & \le \mathbb{E}_{S,z\sim \mathcal{D}^{N+1}} \left[\left| \ell_z(h_S) - \ell_{z} (h_{S'}) \right|\right]  & \text{Jensen's inequality} \\
        & \le \beta
\end{align*}
Replacing $\mathbb{E}[\Phi(S)]$ by $\beta$ in Eq.~(\ref{eqas1}) completes the proof.
\end{proof}

Lemma~\ref{lemma1} shows that the instance weighting algorithm has weight dependent stability.
\begin{lemma}\label{lemma1}
The instance weighting transfer learning algorithm with a $\rho$-Lipschitz continuous loss function and the regularizer $\mathcal{R}(h)=  \|h\|_2^2$ has weight dependent uniform stability, with
    \begin{align*}
         \beta_i \le    \frac{\gamma_i \rho^2 R^2}{\lambda} \enspace.
     \end{align*}
\end{lemma}

\begin{proof}
Let $\mathcal{V}_{S}(h) = \mathcal{L}^\Gamma_{S}(h)+\lambda\mathcal{R}(h)$. By the definition of Bregman divergence, we have
    \begin{align*}
        d_{\mathcal{V}_{{S}^i}}\big({h}_{{S}},{h}_{{S}^i}\big) + d_{\mathcal{V}_{{S}}}\big({h}_{{S}^i},{h}_{{S}}\big) &= \mathcal{L}^\Gamma_{{S}^i} \big( {h}_{{S}} \big) - \mathcal{L}^\Gamma_{{S}^i} \big( {h}_{{S}^i} \big) +  \mathcal{L}^\Gamma_{{S}} \big( {h}_{{S}^i} \big) -  \mathcal{L}^\Gamma_{{S}} \big( {h}_{{S}} \big) \\
        & = \gamma_i \left( \ell (\langle {h}_{{S}}, x_i' \rangle, y_i') - \ell (\langle {h}_{{S}^i}, x_i' \rangle, y_i')  + \ell (\langle {h}_{{S}^i}, x_i \rangle, y_i) - \ell (\langle {h}_{{S}}, {x_i} \rangle, {y_i})\right) \\
        & \le \gamma_i \left( \rho \left| \langle {h}_{{S}}   -  {h}_{{S}^i}, x_i' \rangle \right|  +  \rho \left|  \langle {h}_{{S}}   -  {h}_{{S}^i}, x_i \rangle  \right| \right) \\
        & \le 2\gamma_i\rho R \left|\left| {h}_{{S}}   -  {h}_{{S}^i} \right|\right|_2
    \end{align*}
    where $h_{{S}}$ and $h_{{S}^i}$ are, respectively, the optimal solutions of $\mathcal{V}_{S}$ and $\mathcal{V}_{{S}^i}$.
    The first equality holds because of the first-order optimality condition~\citep{boyd2004convex} of $\mathcal{V}_{S}$ and $\mathcal{V}_{{S}^i}$, and the last two inequalities are, respectively, due to the Lipschitz continuity of loss function $\ell$ and the Cauchy-Schwarz inequality.

    Since $d_{\lambda\mathcal{R}}({h}_{{S}},{h}_{{S}^i}) = d_{\lambda\mathcal{R}}({h}_{{S}^i},{h}_{{S}})=\lambda\|{h}_{{S}}-{h}_{{S}^i}\|_2^2$, by the non-negative and additive properties of Bregman divergence, we have
    \begin{align*}
        \lambda \|{h}_{{S}}-{h}_{{S}^i}\|_2^2 \le \gamma_i \rho R  \left|\left| {h}_{{S}}   -  {h}_{{S}^i} \right|\right|_2 \enspace,
     \end{align*}
     which gives
     \begin{align*}
         \|{h}_{{S}}-{h}_{{S}^i}\|_2 \le \frac{\gamma_i \rho R}{\lambda } \enspace.
     \end{align*}
     Consequently, by the Lipschitz continuity of $\ell$ and the Cauchy-Schwarz inequality, we have
     \begin{align*}
         \beta_i \le    \frac{\gamma_i\rho^2 R^2}{\lambda} \enspace.
     \end{align*}
\end{proof}

\begin{proof}[Proof of Theorem~\ref{maintheorem1}]
Combining Lemmas~\ref{lemma0}, \ref{lemmaas} and Lemma~\ref{lemma1}, we immediately obtain  Theorem~\ref{maintheorem1}
\end{proof}

\begin{proof}[Proof of Lemma~\ref{lemmait}]
By the definition of  $h_{\mathcal{T}}$, $h_{\mathcal{S}}$, and $h^*$, we have
\begin{align*}
         \mathcal{V}_{\mathcal{S}}(h_{\mathcal{S}}) \le \mathcal{V}_{\mathcal{S}}(h^*), \quad \text{and} \quad \mathcal{V}_{\mathcal{T}}(h_{\mathcal{T}}) \le \mathcal{V}_{\mathcal{T}}(h^*) \enspace,
\end{align*}
which gives
\begin{align} \label{tfineq1}
         \mathcal{V}_{\mathcal{S}}(h_{\mathcal{S}}) + \mathcal{V}_{\mathcal{T}}(h_{\mathcal{T}}) \le \mathcal{V}(h^*)+(2\eta -1) \lambda\mathcal{R}(h^*) \enspace.
\end{align}
On the other hand, we also have
\begin{align} \label{tfineq2}
          \mathcal{V}(h^*) \le \mathcal{V}(h_{\mathcal{S}}) \enspace,
\end{align}
and
\begin{align} \label{tfineq3}
           \mathcal{V}(h^*) \le \mathcal{V}(h_{\mathcal{T}}) \enspace.
\end{align}
From (\ref{tfineq1}) and (\ref{tfineq2}), we have
\begin{align} \label{tfineq4}
          \lambda(1-2\eta) \mathcal{R}(h^*)  & \le \mathcal{V}(h^*) - \mathcal{V}_{\mathcal{S}}(h_{\mathcal{S}}) - \mathcal{V}_{\mathcal{T}}(h_{\mathcal{T}}) \\
         & \le \mathcal{V}(h_{\mathcal{S}}) - \mathcal{V}_{\mathcal{S}}(h_{\mathcal{S}}) - \mathcal{V}_{\mathcal{T}}(h_{\mathcal{T}}) \nonumber \\
         & =  \nabla_\mathcal{T} + \lambda (1-\eta)\mathcal{R}(h_{\mathcal{S}})   - \eta \lambda \mathcal{R}(h_{\mathcal{T}}) \nonumber
\end{align}
Similarly, from (\ref{tfineq1}) and (\ref{tfineq3}), we also have
\begin{align} \label{tfineq5}
         &\lambda(1-2\eta) \mathcal{R}(h^*) \le  \nabla_\mathcal{S} + \lambda (1-\eta)\mathcal{R}(h_{\mathcal{T}})   - \eta \lambda \mathcal{R}(h_{\mathcal{S}})
\end{align}
Combining (\ref{tfineq4}) and (\ref{tfineq5}) gives
\begin{align*}
         \mathcal{R}(h^*) \le \frac{\nabla}{2\lambda(1-2\eta)}   + \frac{1}{2}\left( \mathcal{R}(h_{\mathcal{S}}) + \mathcal{R}(h_{\mathcal{T}})\right) \enspace.
\end{align*}
Substituting $\mathcal{R}(h) = \|h\|_2^2$ concludes the proof.
\end{proof}

\subsection{Feature Representation}

\begin{proof}[Proof of Theorem~\ref{maintheorem2}]
Following a similar proof for Theorem~\ref{maintheorem1}, we can first show that
\begin{align}\label{freq1}
\mathcal{L}_{\mathcal{D}_{\mathcal{T}}}(h\circ\Phi) - \mathcal{L}_{\mathcal{D}}(h\circ\Phi)  \le \frac{N_\mathcal{S}}{N} |\mathcal{L}_{\mathcal{D}_{\mathcal{T}}}(h\circ\Phi)- \mathcal{L}_{\mathcal{D}_{\mathcal{S}}}(h\circ\Phi)|  \leq \frac{N_\mathcal{S}}{N}\text{dist}_{\mathcal{Y}}(\mathcal{D}_{{\mathcal{T}}}^{\Phi}, \mathcal{D}_{{\mathcal{S}}}^{\Phi}) \enspace.
\end{align}
Similarly, we also have
\begin{align}\label{freq2}
 \mathcal{L}_{\mathcal{D}}(h_S \circ \Phi) \leq \mathcal{L}_{S}(h_S \circ \Phi) + \beta +  (2N\beta + B)\sqrt{\frac{\log(\frac{1}{\delta})}{2N}} \enspace,
\end{align}
and
\begin{align}\label{freq3}
 \beta \le \frac{\rho^2R^2}{\lambda N} \enspace.
\end{align}
Combining (\ref{freq1}), (\ref{freq2}) and (\ref{freq3}), we immediately obtain  Theorem~\ref{maintheorem2}.
\end{proof}

\subsection{Hypothesis Transfer}
\begin{proof}[Proof of Theorem~\ref{maintheorem3}]
The proof of Theorem~\ref{maintheorem3} follows  readily from Theorem 14.2 and Proposition 14.4 in~\citep{mohri2018foundations}, and therefore is omitted here.
\end{proof}

\begin{proof}[Proof of Lemma~\ref{lemmahtf}]
Let $\mathcal{V}(h) = \mathcal{L}_{S_\mathcal{T}}(h) + \lambda ||h-\langle H, \Xi \rangle||_2^2$. By the definition of $h^*$, we have
\begin{align*}
    \mathcal{V}(h^*) \le \mathcal{V}({\langle H, \Xi \rangle})  &\Rightarrow  \mathcal{L}_{S_\mathcal{T}}(h^*) + \lambda ||h^*-\langle H, \Xi \rangle||_2^2 \le   \mathcal{L}_{S_\mathcal{T}}(\langle H, \Xi \rangle) \\
    & \Rightarrow ||h^*-\langle H, \Xi \rangle||_2 \le \sqrt{\frac{\mathcal{L}_{S_\mathcal{T}}(\langle H, \Xi \rangle) - \mathcal{L}_{S_\mathcal{T}}(h^*)}{\lambda}} \\
     & \Rightarrow ||h^*||_2 \le \sqrt{\frac{\mathcal{L}_{S_\mathcal{T}}(\langle H, \Xi \rangle) - \mathcal{L}_{S_\mathcal{T}}(h^*)}{\lambda}} + ||\langle H, \Xi \rangle||_2 \\
     & \Rightarrow ||h^*||_2 \le \sqrt{\frac{\mathcal{L}_{S_\mathcal{T}}(\langle H, \Xi \rangle) - \mathcal{L}_{S_\mathcal{T}}(h_\mathcal{T})}{\lambda}} + ||\langle H, \Xi \rangle||_2
 \end{align*}
\end{proof}

\subsection{Task Weighting}

\begin{proof}[Proof of Theorem~\ref{twmtheorem}]
The  proof of Theorem~\ref{twmtheorem} follows a similar line as in Theorem~\ref{maintheorem1}.

Let $\mathcal{L}_\mathcal{D}^{\Gamma^j} (h)= \sum_{k=1}^K \gamma_k^j \mathcal{L}_{\mathcal{D}_k} (h)$. Then, we have
\begin{align}\label{twmeq1}
\mathcal{L}_{\mathcal{D}_j}(h) -  \mathcal{L}_\mathcal{D}^{\Gamma^j} (h) & = \sum_{k=1}^K \gamma_k^j \mathcal{L}_{\mathcal{D}_j}(h) - \sum_{k=1}^K \gamma_k^j \mathcal{L}_{\mathcal{D}_k} (h) \\
& \le \sum_{k\neq j} \gamma_k^j\left| \mathcal{L}_{\mathcal{D}_j} (h) - \mathcal{L}_{\mathcal{D}_k} (h) \right| \nonumber \\
&\le \sum_{k\neq j} \gamma_k^j \dist_\mathcal{Y}(\mathcal{D}_j, \mathcal{D}_k) \nonumber
\end{align}

On the other hand, let $h_S$ be the hypothesis returned by the task weighting multitask learning approach. Then, by Lemma~\ref{lemmaas}, we also have
    \begin{align}\label{twmeq2}
         \mathcal{L}_{\mathcal{D}}^{\Gamma^j} (h_S) \le \mathcal{L}_S^{\Gamma^j} (h_S) + {\beta}^j + ({\Delta} + {\beta}^j + \frac{\|{\Gamma}^j\|_\infty}{N} B)\sqrt{\frac{KN\log \frac{1}{\delta}}{2}} \enspace,
     \end{align}
     where  ${\beta^j} = \max \{{\beta}_k^j\}_{k=1}^{K}$, and  ${\Delta} = \sum_{k=1}^K \gamma^j_k \beta_k^j$.

     Similarly, by Lemma~\ref{lemma1}, we can prove that
      \begin{align}\label{twmeq3}
         \beta_k^j \le    \frac{\gamma_k^j \rho^2 R^2}{\lambda N} \enspace.
     \end{align}
     Combining (\ref{twmeq1}), (\ref{twmeq2}), and (\ref{twmeq3}), we immediately obtain Theorem~\ref{twmtheorem}.
\end{proof}

\begin{proof}[Proof of Lemma~\ref{lemmatwm}]
    Let $\mathcal{V}(h)= \sum_{k=1}^K \gamma_{k}^j \mathcal{L}_{k}(h) + \lambda\mathcal{R}(h)$, $\mathcal{V}_k(h)=\mathcal{L}_{k}(h) + \eta \lambda \mathcal{R}(h), \forall k=1,\dots, K$. Then, we have
  \begin{align*}
     \mathcal{V}_k(\bar{h}_k) \le \mathcal{V}_k(h_j^*)
     &\Rightarrow \sum_{k=1}^K \gamma_k^j \mathcal{V}_k(\bar{h}_k)  \le \sum_{k=1}^K \gamma_k^j \mathcal{V}_k(h_j^*) \\
     &\Rightarrow \sum_{k=1}^K \gamma_k^j [\mathcal{L}_k(\bar{h}_k) + \eta \lambda ||\bar{h}_k||_2^2] \le  \sum_{k=1}^K \gamma_k^j \mathcal{L}_k(h_j^*) + \eta \lambda ||h_j^*||_2^2
 \end{align*}
 On the other hand, for any $\bar{h}_j$, we also have
   \begin{align*}
     \mathcal{V}(h_j^*) \le \mathcal{V}(\bar{h}_j) \Rightarrow \sum_{k=1}^K \gamma_k^j \mathcal{L}_k(h_j^*) +  \lambda ||h_j^*||_2^2 \le  \sum_{k=1}^K \gamma_k^j \mathcal{L}_k(\bar{h}_j) +  \lambda ||\bar{h}_j||_2^2 \enspace,
 \end{align*}
 which gives
 \begin{align*}
   & \sum_{k=1}^K \gamma_k^j [\mathcal{L}_k(\bar{h}_k) + \eta \lambda ||\bar{h}_k||_2^2]  \le  \sum_{k=1}^K \gamma_k^j \mathcal{L}_k(h_j^*) + \eta \lambda ||h_j^*||_2^2 \le (\eta-1) \lambda  ||h_j^*||_2^2 +  \sum_{k=1}^K \gamma_k^j \mathcal{L}_k(\bar{h}_j) +  \lambda ||\bar{h}_j||_2^2 \\
    & \Rightarrow (1-\eta) \lambda_j  ||h_j^*||_2^2  \le  \sum_{k=1}^K \gamma_k^j \mathcal{L}_k(\bar{h}_j) +  \lambda ||\bar{h}_j||_2^2 - [\sum_{k=1}^K \gamma_k^j [\mathcal{L}_k(\bar{h}_k) + \eta \lambda ||\bar{h}_k||_2^2] ] \\
    & \Rightarrow (1-\eta) \lambda  ||h_j^*||_2^2  \le (1-\eta)\lambda||\bar{h}_j||_2^2 +  \sum_{k=1}^K \gamma_k^j [\mathcal{L}_k(\bar{h}_j) +  \eta \lambda ||\bar{h}_j||_2^2] - [\sum_{k=1}^K \gamma_k^j [\mathcal{L}_k(\bar{h}_k) + \eta \lambda ||\bar{h}_k||_2^2] ]\\
     & \Rightarrow (1-\eta) \lambda  ||h_j^*||_2^2  \le (1-\eta)\lambda||\bar{h}_j||_2^2 + \sum_{k\neq j} \gamma_k^j [\mathcal{V}_k(\bar{h}_j)-\mathcal{V}_k(\bar{h}_k)] \\
     & \Rightarrow ||h_j^*||_2 \le \sqrt{||\bar{h}_j||_2^2+ \frac{\nabla}{(1-\eta)\lambda}}
 \end{align*}
\end{proof}

\subsection{Parameter Sharing}

\begin{proof}[Proof of Theorem~\ref{psmtheorem}]
The proof of the upper bound of $\mathcal{L}_{\mathcal{D}_j}(h_j^*)$ follows quite readily from Theorem 14.2 in~\citep{mohri2018foundations}. We only prove the upper bound of $\beta$.

Let $W=[w_0,w_1,\dots,w_K]$. We define the convex function $\mathcal{V}_{{S}}(W)$ as
\begin{align*}
  \mathcal{V}_{S}(W) = \mathcal{L}_{S}(W) + \mathcal{N}(W),
\end{align*}
where $\mathcal{L}_{S}(W) = \frac{1}{K}\sum_{k=1}^K\mathcal{L}_{S_k}(w_0+w_k)$ is the empirical loss of $W$ over $S$, and $\mathcal{N}(W)=\lambda_0||w_0||_2^2+ \frac{\lambda}{K} \sum_{k=1}^K ||w_k||_2^2$. Note that $\mathcal{N}(W)$ is a strictly convex function with respect to $W$.  By the  definition of Bregman divergence and the first-order optimality condition of $\mathcal{V}$, we have
\begin{align}\label{mteq1}
        & d_{\mathcal{V}_{S^j_i}}\big({W}_{S},W_{S^i_j}\big) + d_{\mathcal{V}_{S}}\big({W}_{S^i_j},{W}_{S}\big) \nonumber \\
         &= \mathcal{L}_{S^i_j} \big( {W}_{S} \big) - \mathcal{L}_{S^i_j} \big( W_{S^i_j} \big) +  \mathcal{L}_{S} \big( {W}_{S^i_j} \big) -  \mathcal{L}_{S} \big( W_{S} \big) \nonumber \\
        & = \frac{1}{KN} \Big( \ell (\langle {w_0}_{S} + {w_j}_{S}, {x_i^j}' \rangle, {y_i^j}') - \ell (\langle {w_0}_{S^i_j} + {w_j}_{S^i_j}, {x_i^j}' \rangle, {y_i^j}') \nonumber \\
        & \hspace{143pt}+ \ell (\langle {w_0}_{S^i_j} + {w_j}_{S^i_j}, {x_i^j} \rangle, {y_i^j}) - \ell (\langle {w_0}_{S} + {w_j}_{S}, {x_i^j} \rangle, {y_i^j})\Big) \nonumber \\
        & \le \frac{1}{KN} \Big( \rho\left| \langle {w_0}_{S}   -  {w_0}_{S^i_j} + {w_j}_{S}   -  {w_j}_{S^i_j}, {x_i^j}' \rangle \right|  + \rho \left| \langle {w_0}_{S}   -  {w_0}_{S^i_j} + {w_j}_{S}   -  {w_j}_{S^i_j}, x_i^j \rangle \right|   \Big) \nonumber \\
        & \le \frac{2\rho R}{KN} \left|\left| {w_0}_{S}   -  {w_0}_{S^i_j} + {w_j}_{S}   -  {w_j}_{S^i_j} \right|\right|_2 \enspace,
    \end{align}
where $S_j^i$ is the training set $S$ with the $i$-th training example of the $j$-th task, $z_i^j$, replaced by an i.i.d. point ${z_i^j}'$. The last two inequalities are, respectively, due to the Lipschitzness of the loss function $\ell$ and the Cauchy-Schwarz inequality.
    On the other hand, by the definition of $\mathcal{N}(W)$, we also have
     \begin{align}\label{mteq2}
        &d_{\mathcal{N}}\big(W_{S^i_j},W_{S}\big) = d_{\mathcal{N}}\big(W_{S},W_{S^i_j}\big)  \nonumber \\
        &=\lambda_0||{w_0}_{S} ||_2^2+\frac{\lambda}{K}\sum_{k=1}^K  ||{w_k}_{S} ||_2^2 - \left( \lambda_0||{w_0}_{S^i_j} ||_2^2+ \frac{\lambda}{K} \sum_{k=1}^K  ||{w_k}_{S^i_j} ||_2^2 \right) \nonumber \\
        & \hspace{160pt}- \left\langle W_{S}  - W_{S^i_j} ,  \left[ 2\lambda_0 {w_0}_{S^i_j}, \frac{2\lambda}{K} {w_1}_{S^i_j},\dots, \frac{2\lambda}{K} {w_K}_{S^i_j}\right] \right\rangle \nonumber \\
         &=\lambda_0||{w_0}_{S} ||_2^2+ \frac{\lambda}{K}\sum_{k=1}^K  ||{w_k}_{S} ||_2^2 + \left( \lambda_0||{w_0}_{S^i_j} ||_2^2+\frac{\lambda}{K}\sum_{k=1}^K  ||{w_k}_{S^i_j} ||_2^2 \right) \nonumber \\
         & \hspace{160pt} - \left\langle W_{S} ,  \left[ 2\lambda_0 {w_0}_{S^i_j}, \frac{2\lambda}{K} {w_1}_{S^i_j},\dots,  \frac{2\lambda}{K} {w_K}_{S^i_j}\right] \right\rangle \nonumber \\
         & = \lambda_0 \left|\left|{w_0}_{S} - {w_0}_{S^i_j}\right|\right|_2^2 + \frac{\lambda}{K}\sum_{k=1}^K  \left|\left|{w_k}_{S} - {w_k}_{S^i_j} \right|\right|_2^2 \enspace.
    \end{align}
     Combining (\ref{mteq1}) with  (\ref{mteq2}), and applying the non-negative and additive properties of Bregman divergence,  we have, for any task $j$, the following holds:
    \begin{align} \label{mteq3}
    \lambda_0 \left|\left|{w_0}_{S} - {w_0}_{S^i_j}\right|\right|_2^2 + \frac{\lambda}{K}  \left|\left|{w_j}_{S} - {w_j}_{S^i_j} \right|\right|_2^2  & \le \lambda_0 \left|\left|{w_0}_{S} - {w_0}_{S^i_j}\right|\right|_2^2 + \frac{\lambda}{K}\sum_{k=1}^K \left|\left|{w_k}_{S} - {w_k}_{S^i_j} \right|\right|_2^2 \nonumber \\
    & \le \frac{\rho R}{KN} \left|\left| {w_0}_{S}   -  {w_0}_{S^i_j} + {w_j}_{S}   -  {w_j}_{S^i_j} \right|\right|_2 \enspace.
    \end{align}
    Applying triangle inequality to (\ref{mteq3}) yields
    \begin{align*}
          \left|\left| {w_0}_{S}   -  {w_0}_{S^i_j} + {w_j}_{S}   -  {w_j}_{S^i_j} \right|\right|_2 \le \frac{\rho R}{\lambda_0 KN} + \frac{\rho R}{\lambda N} \enspace.
    \end{align*}
    Then, by  the Lipschitzness of the loss function and the Cauchy-Schwarz inequality, we have
    \begin{align*}
         \beta \le    \frac{\rho^2 R^2}{\lambda_0 K N} + \frac{\rho^2 R^2}{\lambda N} \enspace,
     \end{align*}
     which concludes the proof.
\end{proof}

\begin{proof}[Proof of Lemma~\ref{psmlemma}]
For any task $j$, by the definition of $h^*_j$ and $\bar{h}_j$, we have $\mathcal{V}_j(\bar{h}_j) \le \mathcal{V}_j (h_j^*)$, which gives
\begin{align}\label{mttheorem3eq1}
          \mathcal{V}_j(\bar{h}_j)  \le \mathcal{L}_{{S}_j}(w_0^* + w_j^*) + \bar{\lambda} ||w_0^* + w_j^*||_2^2  \le \mathcal{L}_{{S}_j}(w_0^* + w_j^*) + \lambda_0 ||w_0^*||_2^2 + \lambda ||w_j^*||_2^2 \enspace.
\end{align}
On the other hand, by the definition of $h_0$, we also have
\begin{align*}
  \frac{1}{K}\sum_{k=1}^K \left[ \mathcal{L}_{{S}_k} (w_0^*+w_k^*) + \lambda_0 ||w_0^*||_2^2 + \lambda ||w_k^*||_2^2 \right] \le \frac{1}{K}\sum_{k=1}^K \left[ \mathcal{L}_{{S}_k} (\bar{h}_0) + \lambda_0 ||\bar{h}_0||_2^2  \right] \enspace,
\end{align*}
which gives
\begin{align}\label{mttheorem3eq2}
    \mathcal{L}_{{S}_j} (w_0^*+w_j^*) + \lambda_0 ||w_0^*||_2^2 + \lambda ||w_j^*||_2^2  \le K \mathcal{V}_0 - \sum_{k\neq j} \left[ \mathcal{L}_{{S}_k} (w_0^*+w_k^*) + \lambda_0 ||w_0^*||_2^2 + \lambda ||w_k^*||_2^2 \right] \enspace.
\end{align}
Combining (\ref{mttheorem3eq1}) and (\ref{mttheorem3eq2}) yields
\begin{align*}
\bar{\lambda} ||w_0^* + w_j^*||_2^2  & \le \lambda_0 ||w_0^*||_2^2 + \lambda ||w_j^*||_2^2 \\
& \le K \mathcal{V}_0 - \sum_{k\neq j} \left[ \mathcal{L}_{{S}_k} (w_0^*+w_k^*) + \lambda_0 ||w_0^*||_2^2 + \lambda ||w_k^*||_2^2 \right] - \mathcal{L}_{S_j} (w_0^*+w_j^*) \\
& \le K \mathcal{V}_0 - \sum_{k\neq j} \mathcal{V}_k(\bar{h}_k) - \mathcal{L}_{S_j} (w_0^*+w_j^*) \\
& \le \nabla + \mathcal{V}_j(\bar{h}_j) - \mathcal{L}_{S_j} (w_0^*+w_j^*) \\
& \le \nabla + \bar{\lambda}||\bar{h}_j||_2^2 \enspace,
\end{align*}
which gives
\begin{align*}
    ||h_j^*||_2 \le \sqrt{ \frac{\nabla}{\bar{\lambda}} + ||\bar{h}_j||_2^2} \enspace.
 \end{align*}
\end{proof}

\subsection{Task Covariance}

\begin{proof}[Proof of Theorem~\ref{tcmtheorem}]
The proof of the upper bound of $\mathcal{L}_{\mathcal{D}_j}(h_j^*)$ follows quite readily from Theorem 14.2 in~\citep{mohri2018foundations} and Lemma~\ref{lemma0}. We only prove the upper bound of $\beta$.

Let $\mathcal{V}_S(H)= \mathcal{L}_S(H) + \mathcal{N}(H)$, where $\mathcal{L}_S(H)=\frac{1}{K}\sum_{k=1}^K \mathcal{L}_{S_k}(h_k)$ is the empirical loss over $K$ tasks, and $\mathcal{N}(H) = \text{tr}(H\Omega^{-1}H^\top)$. By the definition of Bregman divergence and the first-order optimality condition of $\mathcal{V}$, for any task $j$, we have
    \begin{align}\label{cmmeq1}
        & d_{\mathcal{V}_{{S}_j^i}}\big({H}_S,{H}_{S_j^i}\big) + d_{\mathcal{V}_{{S}}}\big(H_{S_j^i},H_S\big) \nonumber \\
        & \hspace{30pt}  = \mathcal{L}_{{S}^i_j} \big( H_S \big) - \mathcal{L}_{{S}^i_j} \big( H_{S_j^i} \big) +  \mathcal{L}_{{S}} \big( H_{S_j^i} \big) -  \mathcal{L}_{{S}} \big( H_S \big) \nonumber \\
        &\hspace{30pt} \le   \frac{1}{NK}  \left( \ell \big(  \langle h_S, {x_i^j}'\rangle,  {y_i^j}' \big) -  \ell \big( \langle h_{S_j^i},  {x_i^j}' \rangle, {y_i^j}'  \big) + \ell \big( \langle h_{{S}^j_i}, x_i^j\rangle, y_i^j \big) - \ell \big( \langle h_S, x_i^j\rangle,  y_i^j \big) \right) \nonumber \\
        &\hspace{30pt}  \le \frac{2 \rho R}{NK} ||h_S - h_{S_j^i}||_2 \enspace, 
    \end{align}
    where $S_j^i$ is the traing set $S$ with the $i$-th training example of the $j$-th task, $(x_i^j, y_i^j)$, replaced by an i.i.d. point $({x_i^j}', {y_i^j}')$.     On the other hand, by the definition of $\mathcal{N}(H)$, we also have
        \begin{align}\label{cmmeq2}
        d_{\mathcal{N}}\left(H_{S_j^i},H_S\right) & = d_{\mathcal{N}}\left(H_S,H_{S_j^i}\right) \nonumber \\
        & = \text{tr}\left(H_S\Omega^{-1}H_S^\top\right) - \text{tr}\left(H_{S_j^i}\Omega^{-1}H_{S_j^i}^\top\right) - 2\text{tr}\left(H_{S_j^i}\Omega^{-1}(H_S-H_{S_j^i})^\top\right) \nonumber\\
        &=\text{tr}\left((H_S-H_{S_j^i})\Omega^{-1}(H_S-H_{S_j^i})^\top\right) \enspace.
    \end{align}
    In addition, we also have
    \begin{align}\label{cmmeq3}
         \frac{||h_S-h_{S_j^i}||_2^2}{\sigma_\text{max}} &\le\frac{||H_S-H_{S_j^i}||_\mathsf{F}^2}{\sigma_\text{max}} \nonumber \\ & =\frac{1}{\sigma_\text{max}}\text{tr}\left((H_S-H_{S_j^i})(H_S-H_{S_j^i})^\top\right)  \nonumber \\
         & \le \text{tr}\left((H_S-H_{S_j^i})\Omega^{-1}(H_S-H_{S_j^i})^\top\right) \enspace.
    \end{align}

    Combining (\ref{cmmeq1}), (\ref{cmmeq2}), (\ref{cmmeq3}), and applying the non-negative and additive properties of Bregman divergence,  for any task $j$, we have
    \begin{align*}
        \frac{||h_S-h_{S_j^i}||_2^2}{\sigma_\text{max}} \le \text{tr}\left((H_S-H_{S_j^i})\Omega^{-1}(H_S-H_{S_j^i})^\top\right) \le \frac{ \rho R}{NK} ||h_S - h_{S_j^i}||_2 \Rightarrow ||h_S-h_{S_j^i}||_2 \le \frac{\sigma_{\text{max}}\rho R}{NK},
    \end{align*}
    which gives
    \begin{align*}
        \beta \le \frac{\sigma_{\text{max}}\rho^2 R^2}{NK} \enspace.
    \end{align*}
\end{proof}

\begin{proof}[Proof of Lemma~\ref{tcmlemma}]
For any task $j$, by the definition of $h_j^*$ and $\bar{h}_j$, we have $\mathcal{L}_{S_j}(\bar{h}_j) + {\frac{K}{\sigma_\text{max}}}||\bar{h}_j||_2^2\le \mathcal{L}_{S_j} (h_j^*) + {\frac{K}{\sigma_\text{max}}}||h_j^*||_2^2 $. On the other hand, we also have
     \begin{align*}
         \frac{1}{K}\sum_{k=1}^K \mathcal{L}_{S_k}(h_k^*) + \frac{1}{\sigma_\text{max}}\text{tr}({H^*}^\top H^*) &  \le \frac{1}{K}\sum_{k=1}^K \mathcal{L}_{S_k}(h_k^*) + \text{tr}(H^*\Omega^{-1} {H^*}^\top)  \\
         &\le \frac{1}{K}\sum_{k=1}^K \mathcal{L}_{S_k}(\bar{h}_0) + \text{tr}(\bar{H}_0\Omega^{-1} \bar{H}_0^\top) \enspace,
    \end{align*}
    where $\bar{H}_0=[\bar{h}_0,\dots,\bar{h}_0]$. Since $\text{tr}(H_0\Omega^{-1} H_0^\top) = \omega ||h_0||_2^2$, where $\omega$ is the sum of the elements of $\Omega^{-1}$,  for any task $j$, we have
     \begin{align*}
            \frac{K}{\sigma_\text{max}(\Omega)}||h_j^*||_2^2 &  \le \sum_{k=1}^K \mathcal{L}_{S_k}(\bar{h}_0) + K\omega ||\bar{h}_0||_2^2 - \sum_{k \neq j} \left[ \mathcal{L}_{S_k}(h_k^*) + \frac{K}{\sigma_\text{max}(\Omega)}||h_k^*||_2^2  \right]-\mathcal{L}_{S_j}(h_j^*)\\
          & \le \sum_{k=1}^K \mathcal{L}_{S_k}(\bar{h}_0) + K\omega ||\bar{h}_0||_2^2 - \sum_{k \neq j} \left[ \mathcal{L}_{S_k}(\bar{h}_k) + \frac{K}{\sigma_\text{max}(\Omega)}||\bar{h}_k||_2^2 \right] - \mathcal{L}_{S_j}(h_j^*)\\
          & \le \nabla + \mathcal{V}_j(\bar{h}_j) - \mathcal{L}_{S_j}(h_j^*) \\
          & \le  \nabla  + \frac{K}{\sigma_\text{max}(\Omega)}||\bar{h}_j||_2^2 \enspace,
    \end{align*}
    which gives
    \begin{align*}
           ||h^*_j||_2 \le \sqrt{\frac{\sigma_\text{max} \nabla}{K}+||\bar{h}_j||_2^2}\enspace.
    \end{align*}
\end{proof}

\subsection{$\afunc{gapBoost}$}

In order to prove Proposition \ref{boostpropositionA}, we need the following auxiliary result. 
\begin{lemma}\label{boostlemma}
Let $\mathcal{H}$ be a hypothesis class of real-valued functions returned by the transfer learning algorithm~(\ref{obj1})
\begin{align}\label{obj1}
    \min_{h\in \mathcal{H}}  \mathcal{L}_{S}^{\Gamma}(h)+  \lambda \mathcal{R}(h)\enspace,
\end{align}
with a $\rho$-Lipschitz continuous loss function. The convex hull of $\mathcal{H}$ is defined as
\begin{align*}
 \mathcal{F}  = \left\{ \sum_{k=1}^K \mu_k h_k(x):  \sum_{k=1}^K \mu_k = 1, \mu_k \geq 0, h_k \in \mathcal{H}, \forall k = \{1,\dots, K\} \right\} \enspace.
\end{align*}
Define $\Gamma = [\Gamma_1,\dots, \Gamma_K] \in \mathbb{R}^{N \times K}$, where for any $k\in 1,\dots,K$, $\Gamma_k = [\Gamma_k^\mathcal{T}; \Gamma^\mathcal{S}_k] =  [\gamma_{1,k}^{\mathcal{T}}, \dots,\gamma_{N_\mathcal{S},k}^{\mathcal{S}}]^\top \in \mathbb{R}^N$ are the weights for the $k$-th base learner.  Then, for any $\delta \in (0,1)$, we probability at least $1-\delta$, we have
\begin{align*}
 \mathfrak{R}_{\mathcal{D}_\mathcal{T}}(\mathcal{F})  \le \frac{\gamma_\infty^\mathcal{T}\rho R^2}{\lambda}\sqrt{2N\log\frac{2}{\delta}} \enspace,
\end{align*}
where $\gamma_\infty^\mathcal{T} = \max_k \{\|\Gamma_k^\mathcal{T}\|_\infty\}_{k=1}^K$ is the largest weight of the target sample over all the boosting iterations.
\end{lemma}

\begin{proof}
We derive the generalization bound from the unweighted target training sample, treating the source domain sample as a regularizer~\citep{liu2017understanding,liu2017algorithm}. Then, following the similar proof schema as in Lemma 7.4 of~\citep{mohri2018foundations}, we have
\begin{align*}
\mathfrak{R}_{\mathcal{D}_\mathcal{T}}(\mathcal{F}) &= \frac{1}{N_\mathcal{T}}\mathbb{E} \sup_{\substack{h_1\in \mathcal{H},\dots,h_K\in \mathcal{H} \\ \mu_k \ge 0, \sum_{k=1}^K \mu_k \le 1}}  \sum_{i=1}^{N_\mathcal{T}} \sigma_i \sum_{k=1}^K \mu_k h_k(x_i) \\
&= \frac{1}{N_\mathcal{T}}\mathbb{E} \sup_{\substack{h_1\in \mathcal{H},\dots,h_K\in \mathcal{H} \\ \mu_k \ge 0, \sum_{k=1}^K \mu_k \le 1}} \sum_{k=1}^K \mu_k \left( \sum_{i=1}^{N_\mathcal{T}} \sigma_i    h_k(x_i) \right) \\
& = \frac{1}{N_\mathcal{T}}\mathbb{E} \sup_{h_1 \in \mathcal{H},\dots,h_K\in \mathcal{H}} \max_{k\in\{1,\dots,K\}}  \sum_{i=1}^{N_\mathcal{T}} \sigma_i    \langle h_k - \mathbb{E}h_{k,S}, x_i \rangle \\
& \le \frac{1}{N_\mathcal{T}}\mathbb{E} \max_{k\in\{1,\dots,K\}} \sup_{h_1 \in \mathcal{H},\dots,h_K\in \mathcal{H}}  \|h_k - \mathbb{E}h_{k,S}\| \left|\left| \sum_{i=1}^{N_\mathcal{T}} \sigma_i  x_i\right|\right|_2  \\ 
& \le \frac{1}{N_\mathcal{T}} \max_{k\in\{1,\dots,K\}} \left\{\frac{\|\Gamma_k^\mathcal{T}\|_\infty\rho R}{\lambda}\sqrt{2N_\mathcal{T}\log\frac{2}{\delta}} \right\} \mathbb{E}  \left|\left| \sum_{i=1}^{N_\mathcal{T}} \sigma_i \gamma_i^k x_i\right|\right|_2 \\
& \le  \frac{\gamma_\infty^\mathcal{T}\rho R}{N_\mathcal{T}\lambda}\sqrt{2N_\mathcal{T}\log\frac{2}{\delta}} R\sqrt{N_\mathcal{T}} =  \frac{\gamma_\infty^\mathcal{T}\rho R^2}{\lambda}\sqrt{2\log\frac{2}{\delta}}
\end{align*}
Note that compared with Lemma 7.4 of~\citep{mohri2018foundations}, the main difference in our proof is that for each base learner, its hypothesis class defined by learning algorithm~(\ref{obj1}) is different from others.
\end{proof}

\begin{proof}[Proof of Proposition~\ref{boostpropositionA}]
Given Lemma~\ref{boostlemma}, by following the standard proof schema as in~\citep{mohri2018foundations}, we immediately obtain Proposition~\ref{boostpropositionA}.
\end{proof}

\subsection{$\afunc{gapMTNN}$}

\begin{proof}[Proof of Proposition~\ref{gapMTNNproposition}]
The prediction loss for each task $\mathcal{D}_t$ can be expressed as:
\begin{align*}
    \mathcal{L}_{\mathcal{D}_t}(h\circ g) & = \frac{1}{|\mathcal{Y}|}\sum_{y} \mathcal{L}_{\mathcal{D}_t(z|Y=y)}(h) \enspace.\\
\end{align*}
According to \citep{shui2021aggregating}, as for another task $k$, we have the following inequality:
\begin{align*}
   \mathcal{L}_{\mathcal{D}_j}(h_j^*)= \frac{1}{|\mathcal{Y}|}\sum_{y} \mathcal{L}_{\mathcal{D}_j(z|Y=y)}(h_j^*) \leq \frac{1}{|\mathcal{Y}|}\sum_{y} \mathcal{L}_{\mathcal{D}_k(z|Y=y)}(h_j^*) + \frac{\nu\rho}{|\mathcal{Y}|} \sum_{y} W_1(\mathcal{D}_j(z|y)\|\mathcal{D}_k(z|y)) \enspace.
\end{align*}
For any task $k\neq j$, we assign $\gamma^j_k$ for each task, then we have:
\begin{align*}
     \mathcal{L}_{\mathcal{D}_j}(h_j^*) = \sum_{k}\Gamma_t[k] \mathcal{L}_{\mathcal{D}_j}(h_j^*) & \leq \sum_{k=1}^K\gamma^j_k \mathcal{L}_{\mathcal{D}_k}(h^*_j) +  \frac{\nu\rho}{|\mathcal{Y}|} \sum_{k} \gamma^j_k \sum_{y} W_1(\mathcal{D}_t(z|y)\|\mathcal{D}_k(z|y))\\
    & = \mathcal{L}^{\Gamma^j}_{\mathcal{D}}(h_j^*)+  \frac{\nu\rho}{|\mathcal{Y}|} \sum_{k\neq j} \gamma^j_k \sum_{y} W_1(\mathcal{D}_t(z|y)\|\mathcal{D}_k(z|y)) \enspace.
\end{align*}
\end{proof}

\section{TrAdaBoost, TransferBoost, and TrAdaBoost.R2}\label{boosts}
For the sake of completeness, we include the pseudocode of TrAdaBoost, TransferBoost, and TrAdaBoost.R2, in Algorithms~\ref{alg:tradaboost}, \ref{alg:transferboost}, and \ref{alg:tradaboostR}, respectively.

\begin{algorithm}[t]\small
\caption{TrAdaBoost}
\label{alg:tradaboost}
  \textbf{Input:} $S_\mathcal{S}, S_{\mathcal{T}}, K$, a learning algorithm $\mathcal{A}$
  \begin{algorithmic}[1]
  \STATE Initialize $D^\mathcal{S}_1(i) = D^\mathcal{T}_1(i)=\frac {1}{N_\mathcal{S}+N_\mathcal{T}}$,  $\beta = 1/(1+\sqrt{2\log N_\mathcal{S}/K})$
    \FOR{\(k=1,\dots,K\)}
        \STATE Call $\mathcal{A}$ to train a base learner $h_k$ using $S_\mathcal{S}\cup S_\mathcal{T}$ with distribution $D^\mathcal{S}_k \cup D^\mathcal{T}_k$
        \STATE ${\epsilon_{k} = \sum_{i=1}^{N_\mathcal{S}} D_k^\mathcal{S}(i) \mathbbm{1}_{h_k(x_i^\mathcal{S})\neq y_i^\mathcal{S}}+\sum_{i=1}^{N_\mathcal{T}} D_k^\mathcal{T}(i)\mathbbm{1}_{h_k(x_i^\mathcal{T}) \neq y_i^\mathcal{T} }}$
        \STATE $\alpha_k=\log\frac{1-\epsilon_{k}}{\epsilon_{k}}$
        \FOR{$i = 1,\dots, N_\mathcal{S}$}
            \STATE{$\beta_i^\mathcal{S} = \beta \mathbbm{1}_{ h_k(x_i^\mathcal{S}) \neq y_i^\mathcal{S}}$}
            \STATE {$D^\mathcal{S}_{k+1}(i)=D^\mathcal{S}_{k}(i) \beta^\mathcal{S}_i$}
        \ENDFOR
        \FOR{$i = 1,\dots, N_\mathcal{T}$}
            \STATE{$\beta_i^\mathcal{T} =  \alpha_k \mathbbm{1}_{ h_k(x_i^\mathcal{T}) \neq y_i^\mathcal{T}}$}
            \STATE {$D^\mathcal{T}_{k+1}(i)=D^\mathcal{T}_{k}(i)\exp\left( \beta^\mathcal{T}_i \right)$}
        \ENDFOR
        \STATE Normalize $D^\mathcal{S}_{k+1}$ and $D^\mathcal{T}_{k+1}$ such that $\sum_{i=1}^{N_\mathcal{S}} D_{k+1}^\mathcal{S}(i)+\sum_{i=1}^{N_\mathcal{T}} D_{k+1}^\mathcal{T}(i) =1$
    \ENDFOR
  \end{algorithmic}
  \textbf{Output:}  $f(x) = \text{sign} \left(\sum_{k=\lceil K/2 \rceil}^K \alpha_k h_k(x) \right)$
\end{algorithm}

\begin{algorithm}[t]\small
\caption{TransferBoost}
\label{alg:transferboost}
  \textbf{Input:} $S_\mathcal{S}, S_{\mathcal{T}}, K$, a learning algorithm $\mathcal{A}$
  \begin{algorithmic}[1]
  \STATE Initialize $D^\mathcal{S}_1(i) = D^\mathcal{T}_1(i)=\frac {1}{N_\mathcal{S}+N_\mathcal{T}}$
    \FOR{\(k=1,\dots,K\)}
        \STATE Call $\mathcal{A}$ to train a base learner $h_k$ using $S_\mathcal{S}\cup S_\mathcal{T}$ with distribution $D^\mathcal{S}_k \cup D^\mathcal{T}_k$
        \STATE Call $\mathcal{A}$ to train an auxiliary learner $h_k^\mathcal{T}$ over target domain using $S_\mathcal{T}$ with distribution $D^\mathcal{T}_k $
        \STATE $\epsilon_{k} = \sum_{i=1}^{N_\mathcal{S}} D_k^\mathcal{S}(i) \mathbbm{1}_{h_k(x_i^\mathcal{S})\neq y_i^\mathcal{S}}+\sum_{i=1}^{N_\mathcal{T}} D_k^\mathcal{T}(i)\mathbbm{1}_{h_k(x_i^\mathcal{T}) \neq y_i^\mathcal{T}}$
        \STATE $\alpha_k=\log\frac{1-\epsilon_{k}}{\epsilon_{k}}$
        \STATE $\dot{\epsilon}_k^\mathcal{T} = \frac{\sum_{i=1}^{N_\mathcal{T}} D_k^\mathcal{T}(i)\mathbbm{1}_{h_k(x_i^\mathcal{T}) \neq y_i^\mathcal{T}}} {\sum_{i=1}^{N_\mathcal{T}} D_k^\mathcal{T}(i)}$, $\ddot{\epsilon}_k^\mathcal{T} = \frac{\sum_{i=1}^{N_\mathcal{T}} D_k^\mathcal{T}(i)\mathbbm{1}_{h_k^\mathcal{T}(x_i^\mathcal{T}) \neq y_i^\mathcal{T}}} {\sum_{i=1}^{N_\mathcal{T}} D_k^\mathcal{T}(i)}$
        \STATE $\zeta_k = \ddot{\epsilon}_k^\mathcal{T} -\dot{\epsilon}_k^\mathcal{T}$
        \FOR{$i = 1,\dots, N_\mathcal{S}$}
            \STATE{$\beta_i^\mathcal{S} = \alpha_k \mathbbm{1}_{ h_k(x_i^\mathcal{S}) \neq y_i^\mathcal{S}}$}
            \STATE {$D^\mathcal{S}_{k+1}(i)=D^\mathcal{S}_{k}(i) \exp \left( \beta^\mathcal{S}_i + \zeta_k \right)$}
        \ENDFOR
        \FOR{$i = 1,\dots, N_\mathcal{T}$}
            \STATE{$\beta_i^\mathcal{T} =  \alpha_k \mathbbm{1}_{ h_k(x_i^\mathcal{T}) \neq y_i^\mathcal{T}}$}
            \STATE {$D^\mathcal{T}_{k+1}(i)=D^\mathcal{T}_{k}(i)\exp\left( \beta^\mathcal{T}_i \right)$}
        \ENDFOR
        \STATE Normalize $D^\mathcal{S}_{k+1}$ and $D^\mathcal{T}_{k+1}$ such that $\sum_{i=1}^{N_\mathcal{S}} D_{k+1}^\mathcal{S}(i)+\sum_{i=1}^{N_\mathcal{T}} D_{k+1}^\mathcal{T}(i) =1$
    \ENDFOR
  \end{algorithmic}
  \textbf{Output:}  $f(x) = \text{sign} \left(\sum_{k=1}^K \alpha_k h_k(x) \right)$
\end{algorithm}

\begin{algorithm}[t]\small
\caption{TrAdaBoost.R2}
\label{alg:tradaboostR}
  \textbf{Input:} $S_\mathcal{S}, S_{\mathcal{T}}, K$, a learning algorithm $\mathcal{A}$
  \begin{algorithmic}[1]
  \STATE Initialize $D^\mathcal{S}_1(i) = D^\mathcal{T}_1(i)=\frac {1}{N_\mathcal{S}+N_\mathcal{T}}$,  $\beta = 1/(1+\sqrt{2\log N_\mathcal{S}/K})$
    \FOR{\(k=1,\dots,K\)}
        \STATE Call $\mathcal{A}$ to train a base learner $h_k$ using $S_\mathcal{S}\cup S_\mathcal{T}$ with distribution $D^\mathcal{S}_k \cup D^\mathcal{T}_k$
        \STATE Let $D_k = \max \left\{|h_k(x_1^\mathcal{S})-y_1^\mathcal{S}|,\dots,|h_k(x_{N_\mathcal{S}}^\mathcal{S})-y_{N_\mathcal{S}}^\mathcal{S}|,|h_k(x_1^\mathcal{T})-y_1^\mathcal{T}|,\dots,|h_k(x_{N_\mathcal{T}}^\mathcal{T})-y_{N_\mathcal{T}}^\mathcal{T}|\right\}$
        \STATE ${\epsilon_{k} = \sum_{i=1}^{N_\mathcal{S}} D_k^\mathcal{S}(i) \frac{|h_k(x_i^\mathcal{S}) - y_i^\mathcal{S}|}{D_k}+\sum_{i=1}^{N_\mathcal{T}} D_k^\mathcal{T}(i)\frac{|h_k(x_i^\mathcal{T}) - y_i^\mathcal{T}|}{D_k}}$
        \FOR{$i = 1,\dots, N_\mathcal{S}$}
            \STATE {$D^\mathcal{S}_{k+1}(i)=D^\mathcal{S}_{k}(i) \beta^{\frac{|h_k(x_i^\mathcal{S}) - y_i^\mathcal{S}|}{D_k}}$}
        \ENDFOR
        \FOR{$i = 1,\dots, N_\mathcal{T}$}
            \STATE{$\beta_i^\mathcal{T} =  \frac{1-\epsilon_{k}}{\epsilon_{k}}$}
            \STATE {$D^\mathcal{T}_{k+1}(i)=D^\mathcal{T}_{k}(i){\beta_i^\mathcal{T}}^{\frac{|h_k(x_i^\mathcal{T}) - y_i^\mathcal{T}|}{D_k}}$}
        \ENDFOR
        \STATE Normalize $D^\mathcal{S}_{k+1}$ and $D^\mathcal{T}_{k+1}$ such that $\sum_{i=1}^{N_\mathcal{S}} D_{k+1}^\mathcal{S}(i)+\sum_{i=1}^{N_\mathcal{T}} D_{k+1}^\mathcal{T}(i) =1$
    \ENDFOR
  \end{algorithmic}
  \raggedright\textbf{Output:}  $f(x) = \text{sign} \left(\sum_{k=\lceil K/2 \rceil}^K \alpha_k h_k(x) \right)$
\end{algorithm}

\section{Implementation Details for the Multitask Learning Experiments}\label{implementation}

We evaluate the algorithm on Digits, PACS, Office-31 and Office-Home data sets. In this part we show the implementation details about these experiments.
\subsection{Network architecture}

 For the Digits data set, we adopt the LeNet-5 as feature extractor and adopt a three layer MLP as classifier. The model is
trained from scratch and the architecture of the model is provided as follows
\begin{itemize}
    \item Feature extractor: with 2 convolution layers.
    \begin{itemize}
        \item ’layer1’: ’conv’: [1, 32, 5, 1, 2], ’relu’, ’maxpool’: [3, 2, 0]
        \item ’layer2’: ’conv’: [32, 64, 5, 1, 2], ’relu’, ’maxpool’: [3, 2, 0]
    \end{itemize}
    \item Task prediction: with 2 fc layers.
    \begin{itemize}
        \item ’layer3’: ’fc’: [*, 128], ’act\_fn’: ’elu’
        \item ’layer4’: ’fc’: [128, 10], ’act\_fn’: ’softmax’
    \end{itemize}
\end{itemize}

For experiments on PACS, we adopt the pre-trained AlexNet provided in Pytorch while removing the last FC layer as a feature extractor, which leads to the output feature size of 4096. Then, we implement several three-MLPs for the task-specific classification (4096-256-RELU-\# classes), the architecture of the classification network is
\begin{itemize}
    \item  (0): Linear(in\_features=4096, out\_features=256, bias=True)
     \item (1): Linear(in\_features=256, out\_features= \# classes, bias=True)
     \item (2): Softmax(dim=-1)
\end{itemize}

For the Office-31 and Office-Home data sets, we adopt pre-trained ResNet-18 in Pytorch while removing the last FC layers as a feature extractor. The output size of the feature extractor is 512. Then, we also implement an MLP for the classification (512-256-RELU-\# classes))
\begin{itemize}
    \item(0): Linear(in\_features=512, out\_features=256, bias=True)
      \item(1): Linear(in\_features=256, out\_features=\# classes, bias=True)
      \item(2): Softmax(dim=-1)
\end{itemize}

\# classes of PACS data sets is 7, \# classes of Office-31 is 31 and \# classes of Office-Home is 65. When computing the semantic loss, we extract the features with size 256. 

\subsection{Data Set Processing and Hyper-parameter Setup}

For the Digits data set, we randomly select $3K$, $5K$ and $8K$ instances for training. For the SVHN data set, we resize the images to $28\times 28 \times 1$; we do not apply any data augmentation on the Digits data set. For PACS data set, we adopt the same original split of train/validation/test splits of~\citep{li2017deeper}. Then, for training the set, we randomly select $10\%, 15\%, 20\%$ of the total instances in the original training set and use the original validation set for developing the model and report the empirical results on the testing set. We train the model with a mini-batch size of $64$. We didn't specify a random seed for the random selection process. We follow the pre-processing protocol of~\citep{carlucci2019domain} to process the data set.
For the office-31 and Office-Home data sets, we adopt the same train/validation/test splits strategy of~\citep{long2017learning,zhou2021multi}. We randomly select $5\%, 10\%, 20\%$ of the instances from the training set to train the model. The validation set is used for developing the model while we test on the rest test data set.  For Office-31 we use a training batch size of $16$ and for the Office-Home data set we use a training batch size of $24$.

We adopt the Adam optimizer to train the model for the three data sets. The initial learning rate was set to $2\times10^{-4}$. During the training process, we decay the learning rate $5\%$ for every 5 epochs. To enforce a $L_2$ regularization, we also enable the weight decay option in Adam provided by Pytorch. The relation coefficients $\balpha$ are initialized by $1/\#\text{tasks}$ for each class, then it is optimized dynamically as the training goes on. We set a weight of $0.1$ to regularize the semantic loss and the marginal alignment objective.

All the experiments were repeated 10 times on Intel Gold 6148 Skylake 2.4 GHz CPU and 2 x NVidia V100SXM2 (16G memory) GPU. 
The main software packages used for the experiments are Pytorch version 1.0, Torchvision Version 0.2, Python version 3.6, and Python CVXPY package 1.0.

\subsection{Experimental Details on Measuring the Wasserstein Distance}

One main similarity measure adopted in our work is the Wasserstein distance. In this Appendix section, we  follow~\citep{zhou2021multi} to briefly introduce the W-distance and its corresponding adversarial training process.

The Wasserstein distance of order $p$ between $\D_i$ and $\D_j$ for any $p \geq 1$ is defined as:
\begin{equation}
\label{Eq.primal_Wass}
W_p^p(\D_i,\D_j) = \inf_{\gamma\in \Pi(\D_i,\D_j)} \int_{\mathcal{X}_1 \times \mathcal{X}_2} c(x,y)^p d \gamma(x,y),
\end{equation}
where $c:\mathcal{X}\times\mathcal{X} \to \R^{+}$ is the cost function of transportation of one unit of mass $x$ to $y$ and $\Pi(\D_i,\D_j)$ is the collection of all joint probability measures on $\mathcal{X}_1 \times \mathcal{X}_2$ with marginals $\D_i$ and $\D_j$. In this paper, we consider the case of Wasserstein-1 distance ($p=1$).

Suppose $|\mathcal{D}_i|=n, |\mathcal{D}_j|=m$, then, computing Eq.~(\ref{Eq.primal_Wass}) has complexity of $\mathcal{O}(n^3+m^3)$. Aiming to solve this issue, we can adopt the \emph{Kantorovich-Rubinstein duality}~\citep{shen2018wasserstein,wainwright2019high}, and for any function $f$, if it is 1- Lipschitz \emph{w.r.t.} the cost function, $\|f(x)-f(x^\prime)\| \leq c(x,x^\prime)$, we have 
\begin{equation}
\label{EqW1}
    W_1 (\mathcal{D}_i,\mathcal{D}_j)  \geq \mathbb{E}_{x\sim \mathcal{D}_i} d(x) - \mathbb{E}_{x^\prime \sim \mathcal{D}_j} d(x^\prime)
\end{equation}

Eq.~(\ref{EqW1}) arrives equality when $f$ reaches the maximum of the left side. Then, we can approximate $W_1$ by,
\begin{equation}
    \label{original_eq_wasserstein}
    W_1(\mathcal{D}_i,\mathcal{D}_j) = \sup_{\|f\|_L < 1} \mathbb{E}_{x\in\mathcal{D}_i}f(x) - \mathbb{E}_{x^\prime \in \mathcal{D}_j}f(x^\prime)
\end{equation}

In practice, we can implement a neural network model to approximate the function and determine its maximum value. Thus, we can have an approximation of Wasserstein distance by determining 
\begin{equation}\label{equation_argmax_wass}
    \arg \max \mathbb{E}_{x\in\mathcal{D}_i}f(x) - \mathbb{E}_{x^\prime \in \mathcal{D}_j}f(x^\prime)
\end{equation}

Then, to minimize the Wasserstein distance can be approximated by minimizing Eq.~(\ref{equation_argmax_wass}),
which leads to a `min-max' optimization and could be solved by adversarial training~\citep{shen2018wasserstein}.



\newpage
\bibliography{ref}

\end{document}